\documentclass[sigconf]{aamas} 
\renewcommand\footnotetextcopyrightpermission[1]{} % removes footnote with conference information in first column
\usepackage{balance} % for balancing columns on the final page
\usepackage{caption, subcaption}
\usepackage{multirow}

\usepackage[noend]{algpseudocode}
\usepackage{algorithm}
\usepackage{float}
\usepackage{amsfonts, amsmath, amsthm}
\usepackage{bm}
\usepackage{dsfont}  % for \mathds{1}
\usepackage{xcolor}         % colors
\usepackage{color}         % colors

\usepackage{multirow} % multirow in tables

\usepackage{thmtools}
\usepackage{thm-restate}

\newtheorem{theorem}{Theorem}[section]
\newtheorem{lemma}{Lemma}[section]

\newtheorem{restatedthminner}{Theorem}
\newenvironment{restatedthm}[1]{%
\renewcommand\therestatedthminner{#1}%
\restatedthminner
}{\endrestatedthminner}

\theoremstyle{definition}
\newtheorem{definition}{Definition}[section]

\theoremstyle{remark}
\newtheorem*{remark}{Remark}
\newtheorem*{assumption}{Assumption}

\usepackage{color}

\setcopyright{ifaamas}
\acmConference[AAMAS '26]{Proc.\@ of the 25th International Conference
on Autonomous Agents and Multiagent Systems (AAMAS 2026)}{May 25 -- 29, 2026}
{Paphos, Cyprus}{C.~Amato, L.~Dennis, V.~Mascardi, J.~Thangarajah (eds.)}
\copyrightyear{2026}
\acmYear{2026}
\acmDOI{}
\acmPrice{}
\acmISBN{}

\acmSubmissionID{\#281}

\title[Safe Reinforcement Learning via Recovery-based Shielding with Gaussian Process]{Safe Reinforcement Learning via Recovery-based Shielding with Gaussian Process Dynamics Models}

\author{Alexander W. Goodall}
\affiliation{
  \institution{Imperial College London, Department of Computing}
  \city{London}
  \country{United Kingdom}}
\email{a.goodall22@imperial.ac.uk}

\author{Francesco Belardinelli}
\affiliation{
  \institution{Imperial College London, Department of Computing}
  \city{London}
  \country{United Kingdom}}
\email{francesco.belardinelli@imperial.ac.uk}

\begin{abstract}
Reinforcement learning (RL) is a powerful framework for optimal decision-making and control but often lacks provable guarantees for safety-critical applications. In this paper, we introduce a novel recovery-based shielding framework that enables safe RL with a provable safety lower bound for unknown and non-linear continuous dynamical systems. The proposed approach integrates a backup policy (shield) with the RL agent, leveraging Gaussian process (GP) based uncertainty quantification to predict potential violations of safety constraints, dynamically recovering to safe trajectories only when necessary. Experience gathered by the `shielded' agent is used to construct the GP models, with policy optimization via internal model-based sampling -- enabling unrestricted exploration and sample efficient learning, without compromising safety. Empirically our approach demonstrates strong performance and strict safety-compliance on a suite of continuous control environments. %\footnote{Full paper including technical appendix at: \url{https://arxiv.org/abs/}} 
\end{abstract}

\keywords{Safe Control; Reinforcement Learning; Gaussian Process}

\newcommand{\BibTeX}{\rm B\kern-.05em{\sc i\kern-.025em b}\kern-.08em\TeX}

\begin{document}

%%% The following commands remove the headers in your paper. For final 
%%% papers, these will be inserted during the pagination process.

\pagestyle{fancy}
\fancyhead{}

%%% The next command prints the information defined in the preamble.

\maketitle 

%%%%%%%%%%%%%%%%%%%%%%%%%%%%%%%%%%%%%%%%%%%%%%%%%%%%%%%%%%%%%%%%%%%%%%%%

\section{Introduction}

Safe reinforcement learning (RL) \cite{garcia2015comprehensive} is an active research area focused on training policies that strictly respect safety constraints during both learning and deployment. In high-stakes domains like robotics, autonomous driving, and healthcare, adhering to safety constraints is crucial, as even a single policy failure can lead to catastrophic consequences \cite{amodei2016concrete}. Provably safe RL \cite{krasowski2023provably} is a promising paradigm for these safety-critical settings, aiming to provide formal guarantees on an agent’s performance. Such guarantees typically fall into two categories: \emph{probabilistic} guarantees or \emph{hard} (deterministic) guarantees. Probabilistic approaches leverage either known stochastic models of the environment \cite{konighofer2021online,wabersich2018linear} or learned models of the environment to verify safety with high confidence \cite{berkenkamp2017safe,wang2023enforcing,gronauer2024reinforcement}, whereas hard-guarantee approaches incorporate strong prior knowledge (e.g. exact system dynamics or formal abstractions) to ensure absolute constraint satisfaction \cite{bastanib2021safe, alshiekh2018safe}. Despite substantial progress, ensuring safety in \emph{unknown} or \emph{uncertain} environments remains highly challenging -- existing methods often rely on idealized assumptions or suffer from scalability issues.

A widely-used paradigm for enforcing safety at runtime is \emph{shielding} \cite{bloem2015shield, alshiekh2018safe}. In this paradigm, each action proposed by the RL agent is checked against safety constraints before execution; if an action is deemed unsafe, it is overridden by a backup (safe) action that keeps the system within a designated safe set. The backup policy is typically a simpler, trusted controller that ensures the state remains in an \emph{operationally safe} region. This scheme allows the learned policy to operate freely when far from danger, and invokes the backup policy only near the boundary of the safe region. An important benefit of shielding is that formal guarantees can be obtained by verifying the backup policy (which is designed for safety) instead of the learned policy, greatly simplifying the verification burden, if for example, the backup policy is simpler (e.g., linear controller).

In the context of continuous control, shielding has been used in many prior works \cite{bastani2021safe,banerjee2024dynamic,wachi2024safe, li2020robust, bastanib2021safe} and even extended to high-dimensional visual-input scenarios \cite{goodall2024leveraging}. 
\emph{Model Predictive Shielding (MPS)} \cite{bastani2021safe} is a classic example of this approach: it combines a stabilizing linear controller with a recovery strategy to verify ``recoverable'' safe states on-the-fly. MPS demonstrated that high-performance RL is achievable with a safety net, but it also highlighted key limitations of prior shielding methods. Notably, MPS (and extensions thereof, e.g., \emph{Robust MPS} \cite{li2020robust}, \emph{Dynamic MPS} \cite{banerjee2024dynamic}) assumes the environment is known and deterministic; or relies on heuristics that break safety guarantees \cite{li2020robust}. For stochastic systems, \emph{Statistical MPS} \cite{bastanib2021safe}, uses samples from a known stochastic model of the environment, for high levels of safety, these sampling based approaches scale poorly and loose any meaningful guarantees. Formal methods based on Hamilton-Jacobi (HJ) reachability analysis \cite{akametalu2014reachability} compute invariant safe sets for continuous control tasks and use them to enforce safety via action replacement. 
%\new{
While HJ-based approaches \cite{fisac2018general, wabersich2023data} provide strong safety guarantees they scale exponentially in the state dimension,
%-- a well known results -- 
making them impractical for high-dimensional systems \cite{mitchell2005time}. Other approaches to safe RL may encode safety as a constraint in the learning objective, e.g., via Constrained Markov Decision Process (CMDP) \cite{altman1999constrained}, however these methods only guarantee constraint satisfaction in expectation and are unsuitable for highly safety-critical domains \cite{voloshin2022policy}.
%}

%The main idea is to integrate the learned RL policy $\widehat \pi$ with a reliable backup policy $\pi_{\text{backup}}$, which can be used to keep the system inside a set of ``operationally safe''' states. Ideally, $\pi_{\text{backup}}$ is to be used only near the boundary of ``operationally safe'' states allowing the learned policy $\widehat \pi$ to be used as frequently as possible for optimal performance. The benefit of this is that provable guarantees can be obtained without having to verify $\widehat \pi$, but rather $\pi_{\text{backup}}$ instead, which is typically easier, as $\pi_{\text{backup}}$ is only required for safety and is usually from a simpler policy class \cite{bastani2021safe}. \emph{Model predictive shielding} (MPS) \cite{bastani2021safe} is a classic example of this idea, that dynamically verifies the set of ``operationally safe'' states, by decomposing the backup policy into a stabilizing linear controller and recovery policy. There are many extensions of MPS that still assume the dynamics are known \cite{li2020robust,bastanib2021safe}, or known and deterministic \cite{banerjee2024dynamic}. Often they rely on heuristics that break safety guarantees \cite{li2020robust}, and even with a known model of the transition dynamics, verification often relies on sampling \cite{li2020robust,bastanib2021safe}, which can scale poorly for very strict lower bounds on the safety probability. 

\paragraph{Contributions} In this paper we deal with some of the limitations highlighted above; we introduce a novel \emph{recovery-based shielding} framework for safe RL with a provable lower bound on safety probability. We summarize our key contributions as follows:

    (1) We combine shielding with online model learning, where the \emph{unknown dynamics} of the system are and modelled directly via \emph{Gaussian process (GP) dynamics models} \cite{deisenroth2011pilco}, and a pre-computed safe backup policy is used to intervene on the agent's actions only when necessary. Our method does not require full knowledge of the system’s dynamics, assuming only an initial backup controller, invariant set and standard smoothness conditions on the dynamics.
    
    (2) We leverage analytic GP uncertainty estimates to predict potential constraint violations in advance, forgoing sampling entirely, and analytically verifying actions \emph{on-the-fly}. This design enables \emph{provable safety guarantees} up to an arbitrary level of safety without additional computational burden.
    
    (3) We provide a concrete proof that our shielded RL agent maintains safety with high probability throughout learning and after proper calibration of the GP dynamics model (which is often quick, both in theory and in practice). In particular, we establish a formal lower bound on the probability of constraint satisfaction at all times, subject to the usual regularity assumptions for GP-based learning.
    
    (4) We evaluate our approach on a suite of continuous control environments. The results show that our method can learn high-performing policies while strictly respecting safety constraints, outperforming a variety of baseline safe RL techniques in terms of both safety (zero constraint violations) and reward optimization.

%}

%}

%Beyond safety, GPs have been used in RL for reward modelling and value function approximation \cite{taylor2009kernelized} and explorations strategies such as upper confidence bound (UCB) \cite{krause2011contextual} and Thompson sampling \cite{russo2018tutorial}.

\section{Preliminaries}
\label{sec:preliminaries}
%TODO: only consider bounded disturbances to mitigate against partial observability
We consider discrete time \emph{non-linear continuous-state dynamical systems}, with both observation noise and bounded disturbances \cite{kochdumper2021aroc},
\begin{equation}
    x(t+1) = f(x(t), u(t), w(t)) \label{eq:differential}
\end{equation}
where $x(t) \in \mathbb{R}^n$ is a \emph{system state}, $u(t) \in \mathbb{R}^m$ is a \emph{control input} or \emph{action}, and $w(t) \in \mathbb{R}^k$ is a \emph{system disturbance}, and $f : \mathbb{R}^{n+m+k} \to \mathbb{R}^n$ is an \emph{unknown} Lipschitz continuous \emph{transition function}. 

We let $\mathcal{X} \subseteq \mathbb{R}^n$ denote the \emph{state constraints}, $\mathcal{U} \subseteq \mathbb{R}^m,$ denote the \emph{control constraints}, and $\mathcal{W} \subseteq \mathbb{R}^k$ denote the set of disturbances. The disturbances are can be drawn from an arbitrary distribution $\mathcal{P}_{\mathcal{W}}$ but the support is assumed to be bounded and known. Furthermore, we denote $\mathcal{X}_0 \subset \mathcal{X}$ as the set of initial system states. The observed system state $\widehat x(t) \in \mathbb{R}^n$ is subject to observation noise $v(t) \in \mathbb{R}^n$, such that, $\widehat x(t)  = x(t) + v(t)$, we assume that that $v(t)$ is drawn from the multivariate Gaussian distribution $\mathcal{P}_{\mathcal{V}} = \mathcal{N}(\mathbf{0}, \Sigma)$, where $\Sigma = \text{diag}([\sigma_1^2, \ldots, \sigma_n^2]) \in \mathbb{R}^{n \times n}$ and $\sigma_1^2, \ldots \sigma_n^2$ are known noise variances for each state dimension (e.g., small sensor noise).

\subsection{Problem setup}

%\new{
\paragraph{Safety semantics} We formalize the safety semantics of our framework using a flexible class of state constraints. Let $\mathcal{X} \subseteq \mathbb{R}^n$ be the state space and $\mathcal{X}_{\text{safe}} \subseteq \mathcal{X}$ the admissible (safe) set. We consider four basic constraint types:

$\bullet$ \textbf{Box constraints.} Each dimension $i$ of the state vector $x \in \mathbb{R}^n$ is bounded by an interval $[a_i, b_i]$:
    \begin{equation}
        \mathcal{X}_{\text{box}} = \{x \in \mathbb{R}^n : a_i \leq x_i \leq b_i, \; \forall i=1,\dots,n\} \label{eq:box}
    \end{equation}
    
$\bullet$ \textbf{Convex hull constraints.} More generally, the safe set may be defined as a convex polytope represented by a finite set of $\ell$ linear inequalities, where $A \in \mathbb{R}^{\ell \times n}$ and $b \in \mathbb{R}^\ell$: %\fb{$l$ or $\ell$?}
    \begin{equation}
        \mathcal{X}_{\text{hull}} = \{x \in \mathbb{R}^n : A x \leq b\} \label{eq:convexhull}
    \end{equation}
    
$\bullet$ \textbf{Inclusion (boundary) constraints.} For either box or convex hull regions, we may require that the state remains inside the region, i.e. $x(t) \in \mathcal{X}_{\text{incl}} \subseteq \mathbb{R}^n$.

$\bullet$ \textbf{Exclusion (obstacle) constraints.} To represent obstacles, we may require that the state avoids certain forbidden sets $\mathcal{X}_{\text{obs}}$, such as boxes or convex hulls. Formally,
    \begin{equation}
        \mathcal{X}_{\text{excl}} = \mathcal{X} \setminus \mathcal{X}_{\text{obs}}
    \end{equation}
The overall safe set $\mathcal{X}_{\text{safe}}$, is then defined as the intersection of inclusion constraints and the complement of exclusion constraints. This naturally yields non-convex feasible regions, e.g. by combining multiple disjoint polytopes with obstacle regions removed. 

%\fb{for clarity, we could in principle provide a BNF.} 
%}

%Our approach handles these general constraint compositions: during shielding, uncertainty sets $\mathcal{E}(t)$ (ellipsoids propagated under the GP dynamics model) are checked against all active inclusion and exclusion constraints. Thus, even highly non-convex safe regions composed of unions and differences of convex sets can be accommodated.

\paragraph{Safe RL objective.} We consider a \emph{deterministic reward function} $R : \mathcal{X} \times \mathcal{U}  \times \mathcal{X} \to \mathbb{R}$ and a \emph{discount factor} $\gamma \in [0, 1)$. Let $\Pi$ be the set of all policies, where $\pi \in \Pi$ is a function from system states to actions $\pi : \mathcal{X} \to \mathcal{U}$. The goal is to maximize an objective function $J(\pi)$, while ensuring that, from any initial state $x(0) \in \mathcal{X}_0 \subseteq \mathcal{X}_{\text{safe}}$ and for any sequence of disturbances $\vec w = (w(0), w(1), \ldots) \in \mathcal{W}^{\infty}$ and noise $\vec v = (v(0), v(1), \ldots ) \in \mathcal{V}^{\infty}$, we have  $x(t) \in \mathcal{X}_{\text{safe}}$, for all $t = 0, 1, \ldots$, where $x(t+1) = f(x(t), \pi(\widehat x(t)), w(t))$, and $w(t) \sim \mathcal{P}_{\mathcal{W}}$ and $v(t) \sim \mathcal{P}_{\mathcal{V}}$ are drawn i.i.d. In this paper we consider the usual objective of \emph{expected discounted return}: $J(\pi) = \mathbb{E}_\pi\left[\sum^{\infty}_{t=0}\gamma^t r(t)\right]$, if we denote $\Pi_{\text{safe}} \subseteq \Pi$ as the set of all safe policies, then our goal is to find a policy $\pi^* = \arg\max_{\pi \in \Pi_{\text{safe}}} J(\pi)$. 

Safety is enforced with high-probability due to the unavoidable observation noise, state disturbances and epistemic GP uncertainty. We adopt of the notion of $\epsilon$-safe policies \cite{bastanib2021safe}. Concretely, let $\xi (x(0), \pi, \vec w, \vec v ) = (x(0), x(1), \ldots) \in \mathcal{X}^{\infty}$ denote a sequence of system states under policy $\pi$ and from state $x(0) \in \mathcal{X}_0$, then,
\begin{definition}[$\epsilon$-safe policy] A policy $\pi \in \Pi$ is $\epsilon$-safe if,
\begin{equation*}
    \mathbb{P}_{\vec w \sim \mathcal{P}_{\mathcal{W}}, \vec v \sim \mathcal{P}_{\mathcal{V}}} ( \xi (x(0), \pi, \vec w, \vec v ) \subseteq \mathcal{X}_{\text{safe}} ) \geq 1 - \epsilon \quad \forall x(0) \in \mathcal{X}_0 \label{eq:epsilonsafe}
\end{equation*}
\end{definition}
The admissible safe policy set $\Pi_{\text{safe}}$ therefore consists of all $\epsilon$-safe policies from the initial states, this is the standard requirement introduced in prior works \cite{bastanib2021safe}. 

\subsection{Dynamics modelling}
\label{sec:dynamicsmodelling}

For modelling the unknown dynamics of the system $f$ we use the same methodology as PILCO \cite{deisenroth2011pilco}. The dynamics model is implemented as a Gaussian process (GP) model with input $\tilde x(t) = (\widehat x(t), u(t)) \in \mathcal{X} \times \mathcal{U}$ and output targets $\Delta(t) = \widehat x(t) - \widehat x(t-1) \in \mathbb{R}^n$ which model the next step deltas. The GP model provides the following one-step predictions,
\begin{align*}
    p(x(t) \mid x(t-1), u(t-1)) & = \mathcal{N}(x(t) \mid \mu(t), \Sigma(t)) \\
    \mu(t) & = x(t-1) + \mathbb{E}_f[\Delta(t)] \\
    \Sigma(t) &= \text{var}_f(\Delta(t))
\end{align*}
In this paper we only consider the zero-mean prior function and radial basis function (RBF) kernel, whose covariance is given by,
\begin{equation*}
    k(\tilde x(\cdot), \tilde x '(\cdot)) = \alpha^2 \exp\left(- \frac{1}{2} (\tilde x(\cdot) - \tilde x '(\cdot))^T \Lambda (\tilde x(\cdot) - \tilde x'(\cdot))\right)
\end{equation*}
where $\alpha^2$ models the variance of the transition function $f$, and $\Lambda = \text{diag}([l_1^2, \ldots l_n^2])$ corresponds to the length-scales $l_i$ for automatic relevance determination. The posterior hyperparameters, which include the signal variance $\alpha^2$, length scales $l_i$ and noise variances $\Sigma_{\varepsilon}$, are trained by evidence maximization \cite{rasmussen2003gaussian}. For training and prediction we use GPJAX \cite{Pinder2022} a flexible library for GP implemented in JAX \cite{jax2018github}, which benefits from just-in-time (JIT) compilation and GPU compatibility. 

Let $\tilde X = [\tilde x_1(\cdot), \ldots, \tilde x_D(\cdot)]$ denote the $D$ training inputs and $y = [\Delta_1(\cdot), \ldots \Delta_D(\cdot)]$ denote the $D$ training outputs. The posterior predictive distribution $p(\Delta(\cdot) \mid \tilde x(\cdot))$ for a known test input $x(\cdot)$, is Gaussian with mean and variance given by,
\begin{align}
    m_f(\tilde x(\cdot)) &= \mathbb{E}_f[\Delta(\cdot)] = k_*^T(K  + \sigma_{\varepsilon}^2 \mathbf{I})^{-1}y  = k_*^T \beta \label{eq:meanpred} \\
    \sigma_f^2 &=  \text{var}_f(\Delta(\cdot)) = k_{**} - k^T_{*} (K  + \sigma_{\varepsilon}^2 \mathbf{I})k_*
\end{align}
where $k_* := k(\tilde X, \tilde x(\cdot))$, $k_{**} := k(\tilde x(\cdot), \tilde x(\cdot))$, $\beta = (K  + \sigma_{\varepsilon}^2 \mathbf{I})^{-1}y$ and $K$ is the gram matrix with entries $K_{ij} = k (\tilde x_i(\cdot), \tilde x_j(\cdot))$. In the case where the output targets are multivariate (i.e., the state dimension $n>1$), we train $n$ conditionally independent GP models. Conditional independence holds for a known (deterministic) test input $x(\cdot)$, however, the joint predicative covariance becomes non-diagonal under uncertain inputs due to shared input uncertainty and cross-covariance terms, these details are covered in Sec.~\ref{sec:gaussianprocess}.

\section{Recovery-based shielding}
\label{sec:recoveryshielding}

\paragraph{Overview} 
For a candidate policy $\widehat \pi \in \Pi$, our goal is to construct a \emph{shielded policy} $\pi_{\text{shield}}$ that is provably $\epsilon$-safe. In this paper, shielding is implemented by carefully switching between  the learned policy $\widehat\pi$ and backup controller $\pi_{\text{backup}}$. The ``switching criterion'' depends only on the system state and analytic safety verification conditions derived from the GP uncertainty sets. The key advantage is that safety is certified without making assumptions about the structure of $\widehat\pi$; the shield can enforce constraints even if $\widehat\pi$ is an arbitrary neural network policy. Formally, $\pi_{\text{shield}}$ coincides with $\widehat\pi$ whenever the state is judged $\epsilon_t$-recoverable; otherwise, it defers to $\pi_{\text{backup}}$ to drive the system back into a verified invariant set. 
%\new{
This separation makes the safety proof policy-agnostic, relying only on the backup controller and uncertainty analysis, as alluded to earlier this is often easier if the backup policy comes from a restricted policy class (e.g., linear controller).%, we make use of this property here.
%}

\paragraph{Control invariant sets} For proving infinite horizon safety for systems defined in terms of \eqref{eq:differential}, a standard approach is to consider stable \emph{equilibrium points} and \emph{control invariant sets}. Formally, a state $x_{eq} \in \mathcal{X}$ is a (stable) {\em  equilibrium point} iff there exists $u_{eq} \in \mathcal{U}$ such that $f(x_{eq}, u_{eq}, \mathbf{0}) = x_{eq}$. A (nominal) {\em control invariant set} $\mathcal{X}_{\text{inv}} \subseteq \mathcal{X}$ is a set of states such that if $x(\cdot) \in \mathcal{X}_{\text{inv}}$ then there exists $u_x \in \mathcal{U}$ such that $f(x(\cdot), u_x,\mathbf{0}) \in \mathcal{X}_{\text{inv}}$. In words, the control invariant set is a subset of system states for which there exists corresponding control inputs that can keep the system within this set indefinitely. For simply providing safety guarantees it is enough to determine an invariant set $\mathcal{X}_{\text{inv}}$ for the control policy and then establish that $\mathcal{X}_{\text{inv}} \subseteq \mathcal{X}_{\text{safe}}$. The computation of control invariant sets of general non-linear dynamics is usually based on Hamilton-Jacobi analysis \cite{mitchell2005time,fisac2018general}, which scale exponentially in the dimension of the state space $n$. In this paper we consider the class of linear (or RBF) backup controllers, for which more scalable methods exist \cite{gruber2020computing,schafer2023scalable}. 

\paragraph{Backup policy.} The goal of the backup policy $\pi_{\text{backup}}$ is to drive the current system state $x(t) \in \mathcal{X}$ back into the control invariant set $\mathcal{X}_{\text{inv}}$. Crucially the backup policy $\pi_{\text{backup}}$ should satisfy,
\begin{equation}
\begin{split}
    x(\cdot) \in \mathcal{X}_{\text{inv}}\Rightarrow \xi(x(\cdot), \pi_{\text{backup}}(x(\cdot)), \vec w, \vec v) \subseteq \mathcal{X}_{\text{inv}}\\
    (\forall \vec w, \vec v \in \mathcal{W}^{\infty} \times \mathcal{V}^{\infty})
\end{split}\label{eq:invariant}
\end{equation}
In this paper, we consider a Zonotope-based approach \cite{schafer2023scalable} provided by AROC \cite{kochdumper2021aroc} for computation of control invariant sets with maximum volume around a given stable equilibrium point $x_{eq}$. The corresponding backup policy $\pi_{\text{backup}}$ is constructed with Linear Quadratic Regulator (LQR) \cite{anderson2007optimal}, and is defined by the feedback matrix $K \in \mathbb{R}^{n \times n}$, formally,
\begin{equation}
    \pi_{\text{backup}} (\widehat x(\cdot)) = u_{eq} - K( \widehat x(\cdot) - x_{eq}) \label{eq:controllaw}
\end{equation}
so that $\pi_{\text{backup}}$ satisfies, $f(x_{eq}, \pi_{\text{backup}}(x_{eq}), \mathbf{0}) = x_{eq}$. The control invariant set $\mathcal{X}_{\text{inv}}$ is then calculated by iteratively applying the control law in \eqref{eq:controllaw} and over-approximating the reachable sets by successive convexification, we refer the reader to \cite{kochdumper2021aroc,schafer2023scalable} for full details. Mathematically the control invariant set $\mathcal{X}_{\text{inv}}$ is represented as a convex polytope (or convex hull), c.f., \eqref{eq:convexhull} from earlier. We state our main assumptions here.

\begin{assumption} We assume privileged access to a linear backup policy $\pi_{\text{backup}}$; and a corresponding control invariant set $\mathcal{X}_{\text{inv}}$ for which $\pi_{\text{backup}}$ is invariant, i.e., \eqref{eq:invariant} holds. 
\end{assumption}

\paragraph{Recoverable states} Provided that $x_{eq} \in \mathcal{X}_{\text{safe}}$ and $\mathcal{X}_{\text{inv}} \subseteq \mathcal{X}_{\text{safe}}$ and ensuring that the system never leaves $\mathcal{X}_{\text{inv}}$ is sufficient for maintaining safety. However, since the computation of $\mathcal{X}_{\text{inv}}$ relies on convex over-approximations and linear controllers, the actual set of \emph{admissible states} might be much larger. Rather we allow for states that are \emph{recoverable} within some fixed time horizon $N \in \mathbb{N}$. %We write this formally.

\begin{definition}[Recoverable state] For a given horizon $N$, disturbances $\vec w \in \mathcal{W}^{N+1}$ and observation noise $\vec v \in \mathcal{V}^{N+1}$ a state $x(\cdot) \in \mathcal{X}$ is \emph{recoverable} if for the sequence $x(0), x(1), \ldots, x(N) \in \mathcal{X}^{N+1}$, given by,
\begin{align*}
    x(0) &= x(\cdot) \\
%    \widehat x(0) & = x(0) + v(0)\\
    \widehat x(t) & = x(t) + v(t) \quad  \forall t = 0, \ldots, N \\
    x(1) &= f(x(0), \widehat\pi(\widehat x(0)), w(t))\\
%    \widehat x(t) & = x(t) + v(t) \quad  \forall t = 1, \ldots, N \\ 
    x(t+1) & = f(x(t), \pi_{\text{backup}}(\widehat x(t)), w(t)) \quad  \forall t = 1, \ldots, N
\end{align*}
we have $x(t) \in \mathcal{X}_{\text{safe}}$ for all $t = 0, \ldots, N$, and there exists $t' = 0, \ldots, N$, such that $x(t') \in \mathcal{X}_{\text{inv}}$.
\end{definition}

Intuitively, a state $x(\cdot) \in \mathcal{X}$ is recoverable if for the control $u(\cdot) = \widehat\pi(\widehat x (\cdot))$ determined by the learned policy $\widehat \pi$, we can recover the system state within $N$ timesteps (using the backup policy $\pi_{\text{backup}}$), back to the control invariant set $\mathcal{X}_{\text{inv}}$ (from which we have already established safety). 
%\new{
Fig.~\ref{fig:uncertaintysets} illustrates this phenomenon; noting both the  \emph{recoverable} (blue) and \emph{irrecoverable}  (red) start outside the control invariant set $\mathcal{X}_{\text{inv}}$ and inside the box constraint $\mathcal{X}_{\text{safe}}$; the backup policy drives the recoverable trajectory back into the control invariant set $\mathcal{X}_{\text{safe}}$ -- establishing a high-probability safety certification, whereas the irrecoverable trajectory leaves the safe set $\mathcal{X}_{\text{safe}}$, intersecting with the explicit box constraint.
%}

We denote $\mathcal{X}^{\vec w, \vec v}_{\text{rec}}(N) \subseteq \mathcal{X}_{\text{safe}}$ as the set of all recoverable states given $\vec w, \vec v \in \mathcal{W}^{N+1} \times \mathcal{V}^{N+1}$ and within time horizon $N$.  Given the randomness associated with the disturbances $\vec w$ and the observation noise $\vec v$, it becomes challenging to check whether the current system state is recoverable given all possible realizations of $\vec w$ and $\vec v$. Rather we allow for states that are recoverable with high probability.

\begin{definition}[$\epsilon$-Recoverable state] For a given horizon $N \in \mathbb{N}$ and tolerance $\epsilon \in [0, 1]$ a state $x(\cdot) \in \mathcal{X}$ is \emph{$\epsilon$-recoverable} if,
\begin{equation*}
    \mathbb{P}_{\vec w \sim \mathcal{P}_{\mathcal{W}}, \vec v \sim \mathcal{P}_{\mathcal{V}}} ( x(\cdot) \in \mathcal{X}^{\vec w, \vec v}_{\text{rec}}(N)) \geq 1 - \epsilon
\end{equation*}
\end{definition}

\begin{figure}[t]
    \centering
    \includegraphics[width=0.98\linewidth]{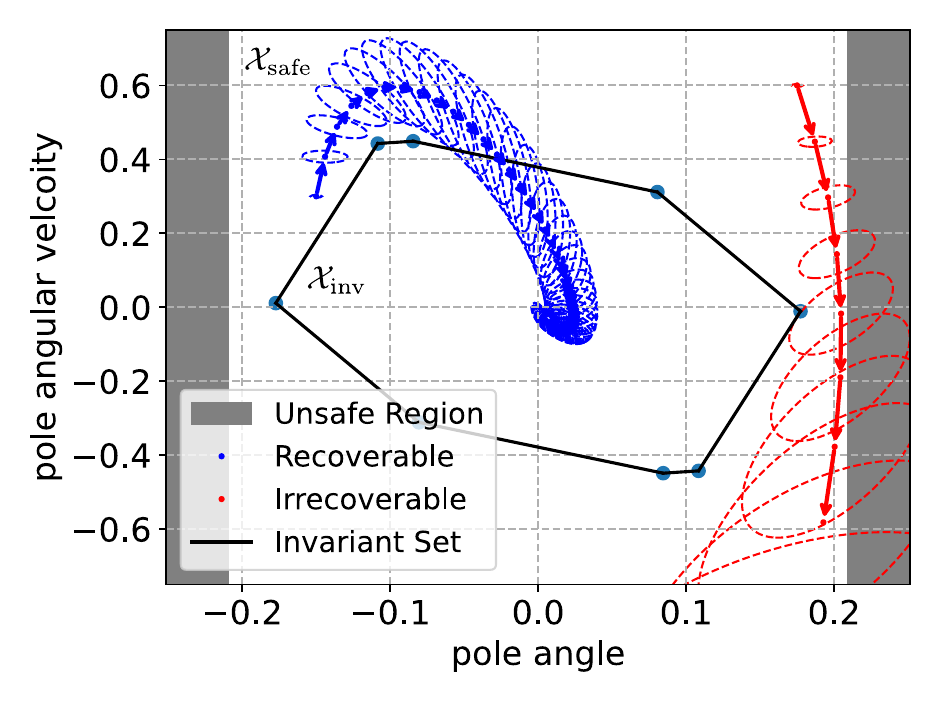}
    \caption{$99.99\%$-CI uncertainty sets for recoverable (blue) and irrecoverable (red) trajectories under the backup policy for \texttt{cartpole (i)/(ii)} environment.}
    \label{fig:uncertaintysets}
\end{figure}

We denote $\mathcal{X}^{\epsilon}_{\text{rec}}(N) \subseteq \mathcal{X}_{\text{safe}}$ as the set of $\epsilon$-recoverable states for some horizon $N$. The shielded policy $\pi_{\text{shield}}$ is constructed by fixing a step-wise safety constraint $\epsilon_t \leq \epsilon$. Therefore, we write an expression for $\pi_{\text{shield}}$ in terms of the $\epsilon_t$-recoverable states. Formally,
\begin{equation}
    \pi_{\text{shield}}(x(\cdot)) = \begin{cases}
                    \widehat\pi(x(\cdot)) & \text{if $x(\cdot) \in \mathcal{X}^{\epsilon_t}_{\text{rec}}(N)$} \\
                    \pi_{\text{backup}}(x(\cdot)) & \text{otherwise}
                \end{cases} \label{eq:shield}
\end{equation}

Intuitively, if each state encountered during execution is verified to be $\epsilon_t$-recoverable and the backup controller guarantees invariance within its certified region, then the resulting shielded policy $ \pi_{\text{shield}}$ is $\epsilon$-safe. This is formalized in the following theorem.

\begin{theorem}[$\epsilon$-safe policy] \label{thm:epsilonsafe} Assume: (i) $\mathcal{X}_{\text{inv}}$ is an invariant set for $\pi_{\text{backup}}$, (ii) $\mathcal{X}_0 \subseteq \mathcal{X}_{\text{inv}} \subseteq \mathcal{X}_{\text{safe}}$ and (iii) $\sum^{T}_{t=0}\epsilon_t = \epsilon$, then, $\pi_{\text{shield}}$ is $\epsilon$-safe for the time horizon $T \in \mathbb{N}$. 
\end{theorem}

\begin{remark}
    The proof of Theorem \ref{thm:epsilonsafe} (c.f., Appendix~\ref{sec:proof}) 
    is obtained by establishing the following invariant: ``we can always return to $\mathcal{X}_{\text{inv}}$ with probability at least $1  - \epsilon_t$ by using the backup policy $\pi_{\text{backup}}$ for $N$ timesteps''. 
    %\new{
    After establishing this invariant (via recursive feasibility), a union bound over the per-step safety tolerances $\epsilon_t$ is then taken. In practice we use either a constant $\epsilon_t \equiv \epsilon/T$ over a finite time horizon $T$, or a decreasing sequence satisfying $\sum_{t=0}^\infty \epsilon_t \leq \epsilon$.\footnote{E.g., setting $\epsilon_t = \epsilon/((t+1)^2\pi^2)$ yields $\sum_{t=0}^\infty \epsilon_t \leq \epsilon$.} This permits extending the guarantee to infinite horizons while keeping $\epsilon$ arbitrarily small.
    %}
\end{remark}

In practice, both the recovery horizon $N$ and step-wise tolerance $\epsilon_t$ are treated as hyperparameters, the choice of $N$ has no consequences on the proof of $\epsilon$-safety, whereas $\epsilon_t$ directly dictates the theoretical safety bound that can be obtained. The choice of $N$ will however, directly impact the computational overhead of action selection and the restrictiveness of the shield. This overhead dominates training time when shielding is invoked at every timestep during training. Concrete runtime details can be found in Appendix \ref{sec:samplinganalysis}. 
We note that the overhead grows roughly linearly in $N$ for our analytic method, compared to $O(NK)$ for sampling baselines, where $K$ is the number of samples. $N$ should be treated as a compute-conservatism trade-off, smaller $N$ will reduce the overhead, but 
might make the shield more conservative (less time to return to $\mathcal{X}_{\text{inv}}$). There is usually a sweet spot, which once reached, increasing $N$ has no significant impact on the shield.

\section{GP uncertainty prediction}
\label{sec:gaussianprocess}
As detailed in Sec.~\ref{sec:dynamicsmodelling} the unknown system dynamics $f$ are modelled using $n$ GP models. We now detail how the uncertainty sets (c.f., Fig.~\ref{fig:uncertaintysets}) are analytically computed and propagated through the dynamics model for robustly identifying $\epsilon_t$-recoverable states (i.e., checking that $\widehat x(\cdot) \in \mathcal{X}_{\text{rec}}^{\epsilon_t}$).

\paragraph{Prediction at uncertain inputs} We model each output dimension with an independent GP; when inputs are uncertain, output correlations arise through uncertainty propagation (moment matching). In particular, the goal is to predict $x(t)$ given $p(x(t-1))$, to do this we require the joint $p(x(t-1), u(t-1))$. This can be calculated by first integrating out the state $x(t-1)$, giving mean $\mu_u(t-1)$ and covariance $\Sigma_u(t-1)$. The cross covariance $\text{cov}[x(t-1), u(t-1)]$ is computed and then the joint $p(\tilde x(t-1)) = p(x(t-1), u(t-1))$ is approximated as a Gaussian with correct mean $\tilde\mu(t-1)$ and covariance $\tilde \Sigma(t-1)$. For linear controllers this computation can be done analytically \cite{deisenroth2010efficient}.
%, for more sophisticated policy parametrizations, a sampling-based approach would be required.

We now assume a joint Gaussian distribution $p(\tilde x (t-1)) = \mathcal{N}(\tilde x(t-1) \mid \tilde \mu (t-1), \tilde \Sigma(t-1))$ at timestep $t-1$. The distribution of the next step deltas is given by the integral,
\begin{equation}
    p(\Delta(t))  = \int p (f(\tilde x (t-1))\mid \tilde x (t-1)) p(\tilde x(t-1)) d \tilde x(t-1) \label{eq:integral}
\end{equation}
This integral \eqref{eq:integral} is intractable, however a common approach is to approximate it as a Gaussian via moment matching \cite{candela2003propagation}. The first two moments of $\Delta(t)$ are analytically computed, ignoring all higher order moments, furthermore, this makes predicting $\Delta(t+1)$ much easier as our distribution over $\Delta(t)$ is still Gaussian. Assuming the mean $\mu_{\Delta}(t)$ and covariance $\Sigma_{\Delta}(t)$ of the predictive distribution $p(\Delta(t))$ are known then the Gaussian approximation of the distribution $p(x(t))$ is given by $\mathcal{N}(x(t) \mid \mu(t), \Sigma(t))$, where,
\begin{align}
    &\mu(t) = \mu(t-1) + \mu_{\Delta}(t) \label{eq:nextmu}\\
    \begin{split}
        &\Sigma(t) = \Sigma(t-1) + \Sigma_{\Delta}(t) \\
        & \quad + \text{cov}[\Delta(t), x(t-1)] + \text{cov}[x(t-1), \Delta(t)] 
    \end{split}\label{eq:nextsigma}\\
    \begin{split}
        &\text{cov}[x(t-1), \Delta(t)] = \\
        & \quad \text{cov}[x(t-1), u(t-1)]\Sigma^{-1}_u(t-1) \text{cov}[u(t-1), \Delta(t)] 
    \end{split}\label{eq:crosscov}
\end{align}

Again, the computation of the cross-covariances (\ref{eq:crosscov}) here is dependent on the parametrization of the policy, e.g., for linear controllers the computation is analytical \cite{deisenroth2010efficient}. For more sophisticated policy parametrizations we would need to resort to sampling and approximate \eqref{eq:integral} directly. When we demand high-level of safety, sampling, e.g., Monte Carlo methods, can suffer from poor sample complexity and are known to be inefficient for assessing the probability of rare events. 
For the analytic computation of the mean $\mu_{\Delta}(t)$ and covariance $\Sigma_{\Delta}(t)$ we refer the reader to \cite{deisenroth2011pilco}, for completeness we also provide our own derivation in Appendix~\ref{sec:meancovpred}. 

\paragraph{Uncertainty propagation.} Combined with the linear backup policy, the $n$ GP dynamics models can now provide analytic Gaussian confidence ellipsoids $\mathcal{E}(0), \ldots, \mathcal{E}(N) \subseteq \mathcal{X}$. The ellipsoids capture both aleatoric uncertainty (observed disturbances and noise) and epistemic uncertainty (lack of data), that enclose the state trajectory with probability $1-\epsilon_t$. Formally, we let $\mu(0) = \widehat x(0)$ and $\Sigma(0) 
 = \text{diag}([\sigma^2_1, \ldots \sigma^2_n])$, then the analytic means $\mu(0), \mu(1), \ldots \mu(N)$ and covariances $\Sigma(0), \Sigma(1), \ldots\Sigma(N)$ computed as described above, form the uncertainty sets $\mathcal{E}(0), \ldots, \mathcal{E}(N) \subseteq \mathcal{X}$, that are defined in terms of the $z$-sigma confidence ellipsoids,
\begin{equation}
    \mathcal{E}(t) = \{ x \in \mathbb{R}^n : (x - \mu(t))^T \Sigma(t)^{-1}(x - \mu(t)) \leq z^2 \} \label{eq:ellipsoid}
\end{equation}
where $z > 0$ satisfies $\Phi(z) \geq 1 - \epsilon_t$ (where $\Phi$ is the CDF of the std. normal), and $\epsilon_t$ is the desired step-wise tolerance from \eqref{eq:shield}. To certify that a state is $\epsilon_t$-recoverable, we must check that these ellipsoids remain in $\mathcal{X}_{\text{safe}}$ and are eventually fully contained in $X_{\text{inv}}$.

%\new{
\paragraph{Containment and exclusion checks.} Given an ellipsoid $\mathcal{E}(t)$ with mean $\mu(t)$ and covariance $\Sigma(t)$:

$\bullet$ \textbf{Box inclusion:} $\mathcal{E}(t) \subseteq \mathcal{X}_{\text{box}}$ iff,
    \begin{equation}
        a_i \leq \mu_i(t) - z\sqrt{\Sigma_{ii}(t)}, \quad
        \mu_i(t) + z\sqrt{\Sigma_{ii}(t)} \leq b_i \label{eq:checkbox}
    \end{equation}
    for all $i = 1,\ldots, n$.
    
 $\bullet$ \textbf{Convex hull inclusion:} $\mathcal{E}(t) \subseteq \{x: Ax \leq b\}$ iff,
    \begin{equation}
        \max_{x \in \mathcal{E}(t)} A_j x - b_j \leq 0 \quad \forall j = 1, \ldots, l \label{eq:checkhull}
    \end{equation}
    which can be solved with the following quadratic program (QP) via Sequential Least Squares Programming (SLSQP) \cite{kraft1988software}:
    \begin{equation}
        \max_x (A_i x - b_i) \text{ subject to } (x - \mu(t))^T\Sigma(t)^{-1}(x - \mu(t)) \leq z^2 \label{eq:quadraticprogram}
    \end{equation}
    
 $\bullet$ \textbf{Exclusion (obstacle avoidance):} For an obstacle region $\mathcal{X}_{\text{obs}}$, we require
    \begin{equation}
        \mathcal{E}(t) \cap \mathcal{X}_{\text{obs}} = \varnothing \label{eq:checkobs}
    \end{equation}
    For convex $\mathcal{X}_{\text{obs}}$, this is checked by solving \eqref{eq:quadraticprogram} with the reversed inequality.
    
 $\bullet$ \textbf{Non-convex combinations:} If $\mathcal{X}_{\text{safe}}$ is a union/intersection of convex sets, containment reduces to checking each constituent.

%}
%\new{
We outline the full procedure in Algorithm \ref{alg:shield}. Starting from the current state distribution, uncertainty sets $\mathcal{E}(0),\dots,\mathcal{E}(N)$ are propagated under the $n$ GP models and backup controller (lines 1-4), see also Fig.~\ref{fig:uncertaintysets}. Lines~5-6 implement box and convex hull inclusion, and/or line~5 also implicitly encodes obstacle exclusion exclusion checks. Line~7 verifies recoverability before deferring to $\pi$ or $\pi_{\text{backup}}$.
%} 
We now state the central theorem of our paper.

\begin{theorem}\label{thm:ellipsoids} Assume: (i) the unknown dynamics $f$ and $\pi_{\text{backup}}$ are both Lipschitz continuous in the $L_1$-norm, (ii) the $n$ GP dynamics models are well calibrated in the sense that with probability (w.p.) at least $1 - \delta$ there exists $\beta(\tau)>0$ such that $\forall t=0 \ldots N$ and $\forall w(t) \in \mathcal{W}$ we have $\lVert f(x(t), u(t), w(t)) - \mu(t) \rVert_1 \leq \beta(\tau) \text{tr}(\Sigma(t))$, (iii) $\forall t=0 \ldots N$ the higher order terms (e.g., skewness, kurtosis) of the true distribution $p(x(t))$ are negligible and can be ignored. Then if $\mathcal{E}(0), \ldots, \mathcal{E}(N) \subseteq \mathcal{X}_{\text{safe}}$ and if there exists $t \in \{0, \ldots N \}$ such that $\mathcal{E}(t) \subseteq \mathcal{X}_{\text{inv}}$ then $\widehat x(0)$ is $\epsilon_t$-recoverable with probability at least $1-\delta$. Thus $\pi_{\text{shield}}$ is $\epsilon$-safe by Theorem \ref{thm:epsilonsafe} (w.p., $1-\delta$). 
\end{theorem}

\begin{remark} We comment on the assumptions for Theorem~\ref{thm:ellipsoids}, which generally hold in our experimental evaluation and analyses. First, (i) holds for many interesting control systems and $\pi_{\text{backup}}$ is linear and thus Lipschitz; (ii) follows from \cite{berkenkamp2017safe} to ensure the confidence intervals we build are robust to the disturbance set $\mathcal{W}$, GPs typically satisfy this assumption under the usual regularity assumptions, more details are provided in Appendix~\ref{sec:gaussianprocessappendix};
(iii) ensures that our analytic Gaussian approximation of the uncertainty sets $\mathcal{E}(t)$ capture at least $1-\epsilon_t$ of the probability pass of the true integral \eqref{eq:integral}, which is not always Gaussian \cite{deisenroth2011pilco}, we provide a thorough empirical analysis in Appendix~\ref{sec:skewkurtanalysis}
suggesting that the uncertainty sets broadly follow multivariate Gaussian distributions in most of our environments, thus validating this assumption in practice. 
\end{remark}

\begin{algorithm}[t]
    \caption{\texttt{Shield} ($\widehat x(0)$, $\widehat \pi$, $\pi_{\text{backup}}$, $\epsilon_t$, $N$).}
    \label{alg:shield}
    \raggedright
    \textbf{Input}: $\widehat x(0) \in \mathcal{X}$, $\widehat \pi$, $\pi_{\text{backup}}$, $\epsilon_t \in [0,1]$, $N \in \mathbb{N}$\\
    \textbf{Output}: Safe action $u \in \mathcal{U}$
    \begin{algorithmic}[1] %[1] enables line numbers
        \State $z \gets \Phi^{-1}(1 - \epsilon_t)$
        \State $\mu(0) \gets \widehat x(0)$, $\Sigma(0) \gets \text{diag}([\sigma_1^2, \ldots, \sigma_n^2])$
        \For{$t = 1, \ldots N$}
            \State Compute $\mathcal{E}(t)$ with (\ref{eq:ellipsoid}).
            \State Check $\mathcal{E}(t) \subseteq \mathcal{X}_{\text{safe}}$ via \eqref{eq:checkbox}, \eqref{eq:checkhull} or \eqref{eq:checkobs}.  
        \EndFor
        \State Check $\mathcal{E}(N) \subseteq \mathcal{X}_{\text{inv}}$ via \eqref{eq:checkbox} or \eqref{eq:checkhull}.
        \If {$\mathcal{E}(0), \ldots, \mathcal{E}(N) \subseteq \mathcal{X}_{\text{safe}} \land \mathcal{E}(N) \subseteq \mathcal{X}_{\text{inv}}$}
        \State \textbf{return} $\widehat\pi(\widehat x(0))$
        \Else
        \State \textbf{return} $\pi_{\text{backup}}(\widehat x (0))$
        \EndIf
    \end{algorithmic}
\end{algorithm}

\section{Implementation}
\label{sec:implementation}

Due to space constraints we defer the pseudocode of the full learning algorithm to Appendix~\ref{sec:algorithms}
However, in this section, we outline the key implementation details and relevant technical definitions.

%\new{
\paragraph{Replay buffer and policy optimization} During training we store all environment interactions of the shielded policy $\pi_{\text{shield}}$ in a replay buffer $\mathcal{D} = \{(x_t,u_t,r_t,x_{t+1})\}$ consisting of tuples of state, action, reward and next state. During policy optimization, we sample \emph{starting states} $x \sim \mathcal{D}$, from which we rollout the unshielded policy $\widehat{\pi}$ for a relatively short time horizon $H$ inside the $n$ GP dynamics models. These simulated rollouts provide on-policy samples without risking safety violations, since they do not interact with the real environment. Since policy optimization is done on relatively short rollouts we train $\widehat{\pi}$ with advantage actor-critic (A2C) \cite{sutton2018reinforcement}. A2C is often trained on shorter horizon rollouts and preferred here over other methods such as, PPO \cite{schulman2017proximal} and TRPO \cite{schulman2015trust}, which expect long horizon rollouts that may accumulate significant noise in the $n$ GP models, thus, destabilising learning. 

For a trajectory $\{(x_t,u_t,r_t,x_{t+1})\}_{t=0}^T$, the temporal-difference residual is,
\begin{equation}
    \delta_t = r_t + \gamma V_\phi(x_{t+1}) - V_\phi(x_t),
\end{equation}
where $V_\phi$ is the critic network and $\gamma$ the discount factor. 
The generalized advantage estimate (GAE) \cite{schulman2015high} is defined as,
\begin{equation}
    \widehat{A}_t^\lambda = \sum_{l=0}^{\infty} (\gamma \lambda)^l \, \delta_{t+l}
\end{equation}
with $\lambda \in [0,1]$ controlling the bias-variance trade-off.
The policy $\widehat \pi$ is updated via policy gradient, i.e., $\nabla J(\widehat{\pi}) = \widehat{A}^\lambda_t \cdot \nabla \log \widehat{\pi}(u_t \mid x_t)$, and the critic parameters $\phi$ are optimized by minimizing the squared Bellman error: $(V_\phi(x_t) - \widehat{R}^\lambda_t)^2$, where $\widehat{R}^\lambda_t$ is the bootstrapped TD-$\lambda$ return. Both the actor and critic are implemented as a two-layer feed-forward neural network with $\tanh$ activations. To improve stability, we use a target network $V_{\phi'}$ and \emph{Polyak averaging} \cite{polyak1992acceleration} to update a target network parameters: $\phi' \gets \tau \phi + (1-\tau)\phi'$. We also use \emph{symlog predictions} \cite{hafner2023mastering} to manage large or varying magnitude rewards. Given a raw return $R$, the symlog transform is,
\begin{equation}
    \text{symlog}(R) = \mathrm{sign}(R) \, \log(1 + |R|),
\end{equation}
with inverse $\text{symexp}(z) = \mathrm{sign}(z)(\exp(|z|)-1)$.
The critic is trained on the $\text{symlog}(R)$ targets, and regularized towards its target network predictions, reducing both sensitivity to reward magnitudes and critic overestimation.

\paragraph{GP Inference} Computing the exact GP posterior has cubic complexity $\mathcal{O}(\vert\mathcal{D}\vert^3)$ in the dataset size $\vert\mathcal{D}\vert$. We therefore use sparse approximations: \emph{Sparse GP regression} (SGPR) \cite{titsias2009variational} introduces $M \ll \vert\mathcal{D}\vert$ inducing inputs $Z = \{z_m\}_{m=1}^M$ with corresponding inducing outputs $u$, yielding a low-rank approximation to the kernel Gram matrix; \emph{Sparse variational GP regression} (SVGP) \cite{NIPS2015_6b180037} maintains a variational distribution $q(u) = \mathcal{N}(m,S)$ over inducing outputs, minimizing $\text{KL}[q(u)\,\|\,p(u \mid y)]$. We found SGPR to be more effective in practice, but include both implementations for flexibility. The inducing inputs and other GP hyperparameters (e.g., length-scales, signal variance, noise variance) are jointly trained by marginal likelihood maximization \cite{deisenroth2010efficient}.

%Computing the full GP posterior is computationally impractical for large datasets. In such instances, one could adopt a replay buffer approach -- only using a sliding window of the past experience for computing the GP posterior. In this paper we use sparse approximations of the GP posterior, parametrized by a set of inducing inputs (pseudo points). In our implementation, both (collapsed) sparse GP regression (SGPR) \cite{titsias2009variational} and (uncollapsed) sparse variational GP regression (SVGP) \cite{NIPS2015_6b180037} are available, although we found the former to be more effective. 

%\new{
\paragraph{Additional Features.} To encourage better exploration we explored the use of \emph{Prioritized replay}, inspired by \cite{schaul2015prioritized}, we implemented a novelty-based prioritized replay buffer. The state space is discretized into bins; the probability of sampling a starting state $x$ for policy optimization is set inversely proportional to the number of visits to its bin. Furthermore, we dealt with unsafe simulated rollouts (under $\widehat{\pi}$), by zeroing out subsequent rewards and terminating the trajectory early. This corresponds to zeroing out the gradients from beyond the first unsafe timestep, preventing conflict between the learned policy and the shield, and biasing towards safe exploration.
%}
%\fb{Consider presenting the algorithm at the beginning of the section, and then explain its working in the different subsections (basically as it is already done now.)}

\section{Experimental evaluation}
\label{sec:experiments}

We now present our experimental results. We first summarize the environment and baselines used in this paper. For additional environment details we refer the reader to Appendix~\ref{sec:environmentdescriptions}.

\paragraph{Environments} We evaluate our method on four types of continuous control tasks:
(1) \texttt{cartpole} \cite{6313077}: the goal is balance a pole on a moving cart; the safety constraint corresponds to preventing the pole from falling. For \texttt{cartpole} we consider two rewards: the usual $+1$ for maintaining upright balance, and \texttt{cartpole2}, where $+1$ is given for achieving a target velocity of $+0.1$. 
(2) \texttt{mountain\_car} \cite{Moore90efficientmemory-based}: the goal is drive an under-actuated car to the top of a mountain (delayed reward); the safety constraint corresponds to avoiding a collision with a wall on the opposite side of the valley. 
(3) \texttt{obstacle}: a 2D navigation task where the goal is to reach a target position while avoiding a single obstacle; \texttt{obstacle2} is more challenging due to obstacle placement, \texttt{obstacle3} and \texttt{obstacle4} are more challenging still with a non-convex combination of constraints. 
(4) \texttt{road}, \texttt{road\_2d}: are 1D and 2D road environments requiring the agent to reach a target location with a fixed speed limit. 

\paragraph{Baselines.} Due to assumption and guarantee mismatch we found it challenging to find relevant baselines. Regardless, we consider two model-free CMDP-based approaches: \emph{Constrained Policy Optimization} (CPO) \cite{achiam2017constrained} and \emph{PPO} with \emph{Lagrangian relaxation} (PPO-Lag) \cite{ray2019benchmarking}. These algorithms assume no knowledge of system dynamics, making them natural references for sample efficiency. However, since our approach leverages a precomputed backup controller and control invariant set, we make strictly stronger assumptions; thus the comparison is not directly fair, but favours the baselines in terms of assumptions. We also consider two model-based approaches: MPS \cite{bastani2021safe} and a recent extension DMPS \cite{banerjee2024dynamic}. These methods assume knowledge of the \emph{deterministic} dynamics of the system, a much stronger assumption than ours. To ensure compatibility, we omit disturbances and observation noise when running MPS/DMPS, effectively simplifying the environments. This adjustment makes the comparison unfair in the opposite direction; our method operates under strictly harder conditions.

\paragraph{Results.} We summarize the results for our approach (A2C-GP-Shield), CPO, PPO-Lag, MPS and DMPS in Tab.~\ref{tab:results}. In all environments, A2C-GP-Shield achieves \emph{perfect safety probability} once the $n$ GP models are properly calibrated, demonstrating the effectiveness of our analytic shielding framework in enforcing strict safety. Importantly, this is achieved without sacrificing return: in 7/10 cases our method achieves the highest mean reward across the baselines that strictly enforce the safety constraint with probability $1$. There is a clear trade-off in some environments; the CMDP-based approaches (CPO and PPO-Lag) optimize constraint satisfaction only in expectation; as a result, they often achieve higher raw returns (e.g., \texttt{cartpole2}, \texttt{obstacle2}) but at the cost of non-negligible safety violations ($>0$ unsafe probability). This highlights their inability to prevent irrecoverable failures during training. In contrast, our approach guarantees strict constraint satisfaction at every timestep, effectively ruling out unsafe trajectories. On the other hand, MPS and DMPS also achieve perfect safety but under much stronger assumptions (e.g., known deterministic dynamics). Overall, A2C-GP-Shield closes the gap between the two extremes: unlike CMDP methods, it enforces safety strictly (not just in expectation), and unlike MPS/DMPS, it does not require access to exact dynamics. This balance explains why A2C-GP-Shield consistently achieves both high safety and competitive return across all tested environments.
%}

\begin{table*}[t]
  \caption{Empirical mean return and safety probability (at the end of training), \textbf{bold text} denotes the best score, asterisk (*) denoted the best reward that also satisfies the safety constraint with probability 1, standard error (SE) bars (averaged over 5 independent runs) are reported. We also include the theoretical lower bound established by Thm.~\ref{thm:epsilonsafe} (based on $\epsilon_t= 10^{-5}$).}
  \label{tab:results}
  \centering
  \footnotesize
  \begin{tabular}{llccccc|c}
    \toprule
    %\multicolumn{2}{c}{Part}                   \\
    %\cmidrule(r){1-2}
      \textbf{Env.} & \textbf{Metric} & \textbf{A2C-GP-Shield} &\textbf{MPS} & \textbf{DMPS} & \textbf{CPO} & \textbf{PPO-Lag} & \textbf{Thm. \ref{thm:epsilonsafe}}  \\
    \midrule
    \texttt{cartpole} & Return & $\bm{200.0 \pm 0.00}$ & $\bm{200.0 \pm 0.00}$ & $\bm{200.0 \pm 0.00}$& $\bm{200.0 \pm 0.00}$& $ 188.0 \pm 5.53 $& \multirow{2}{*}{$1-\epsilon= 0.98$}\\
    & Safety Prob. & $\bm{1.000 \pm 0.00}$ & $\bm{1.000 \pm 0.00}$ & $\bm{1.000 \pm 0.00}$& $\bm{1.000 \pm 0.00}$& $ 0.910 \pm 0.04 $ & \\
    \midrule
    \texttt{cartpole2} & Return & $ 30.7 \pm 10.4 $ & $43.2\pm11.1$ & $65.1\pm23.1$*& $ 152.0 \pm 31.2 $& $\bm{171.0 \pm 7.53}$& \multirow{2}{*}{$1-\epsilon= 0.98$}\\
    & Safety prob. & $\bm{1.000 \pm 0.00}$& $\bm{1.000 \pm 0.00}$& $\bm{1.000 \pm 0.00}$ & $0.980 \pm 0.01$& $ 0.811 \pm 0.12 $ &\\
    \midrule
    \texttt{mountain\_car} & Return & $\bm{91.3 \pm 1.08}$* & $ 85.1 \pm 8.17$& $ 81.2 \pm 0.28 $ & $ -30.4 \pm 6.74 $& $ 73.2 \pm 18.7 $& \multirow{2}{*}{$1-\epsilon= 0.9$}\\
    & Safety prob. & $\bm{1.000 \pm 0.00}$& $\bm{1.000 \pm 0.00}$& $\bm{1.000 \pm 0.00}$ & $0.995 \pm 0.00$& $ 0.884 \pm 0.05 $&\\
    \midrule
    \texttt{obstacle} & Return & $ 32.2 \pm 0.10 $& $8.31 \pm 34.5$ & $32.7 \pm 0.21$& $ 32.5 \pm 0.30 $& $\bm{32.9 \pm 0.01}$*& \multirow{2}{*}{$1-\epsilon= 0.98$}\\
    & Safety Prob. & $\bm{1.000 \pm 0.00}$& $\bm{1.000 \pm 0.00}$& $\bm{1.000 \pm 0.00}$ & $\bm{1.000 \pm 0.00}$& $ \bm{1.000 \pm 0.00} $&\\
    \midrule
    \texttt{obstacle2} & Return & $ 22.9 \pm 5.76 $*& $-1.82 \pm 3.24$  & $20.2\pm 13.9$ &$ 15.3 \pm 6.08 $& $\bm{34.2 \pm 0.01}$& \multirow{2}{*}{$1-\epsilon= 0.98$}\\
    & Safety Prob. & $\bm{1.000 \pm 0.00}$& $\bm{1.000 \pm 0.00}$& $\bm{1.000 \pm 0.00}$ & $0.939 \pm 0.02$& $ 0.000 \pm 0.00 $&\\
    \midrule
    \texttt{obstacle3} & Return & $4.69 \pm 6.40$*& $-0.67 \pm 2.12$ & $4.28\pm 5.92$ &$ 8.56 \pm 12.3 $& $\bm{33.4 \pm 0.114}$& \multirow{2}{*}{$1-\epsilon= 0.98$}\\
    (non-convex)& Safety Prob. & $\bm{1.00 \pm 0.00}$& $\bm{1.00 \pm 0.00}$&$\bm{1.00 \pm 0.00}$ & $ 0.900 \pm 0.141 $& $ 0.148 \pm 0.296 $&\\
    \midrule
    \texttt{obstacle4} & Return & $15.5 \pm 6.11$*& $-1.12 \pm 3.13$ & $11.3 \pm 8.04$ &$ 8.56 \pm 12.3 $& $\bm{33.4 \pm 0.114}$& \multirow{2}{*}{$1-\epsilon= 0.98$}\\
    (non-convex)& Safety Prob. & $\bm{1.00 \pm 0.00}$& $\bm{1.00 \pm 0.00}$ & $\bm{1.00 \pm 0.00}$ &$ 0.924 \pm 0.0833 $& $ 0.00 \pm 0.00 $&\\
    \midrule
    \texttt{road} & Return & $\bm{23.0 \pm 0.01}$*& $22.7 \pm 0.04$& $22.8 \pm 0.02$ & $ 22.9 \pm 0.05 $& $ 22.9 \pm 0.01 $& \multirow{2}{*}{$1-\epsilon= 0.99$}\\
    & Safety Prob. & $\bm{1.000 \pm 0.00}$& $\bm{1.000 \pm 0.00}$& $\bm{1.000 \pm 0.00}$ & $0.974 \pm 0.01$& $ 0.000 \pm 0.00 $ &\\
    \midrule
    \texttt{road\_2d} & Return & $ 23.9 \pm 0.26 $& $24.0 \pm 0.22$* & $24.0 \pm 0.22$* & $24.0 \pm 0.09$& $\bm{24.1 \pm 0.02}$& \multirow{2}{*}{$1-\epsilon= 0.99$}\\
    & Safety Prob. & $\bm{1.000 \pm 0.00}$& $\bm{1.000 \pm 0.00}$& $\bm{1.000 \pm 0.00}$ & $0.969 \pm 0.01$& $ 0.073 \pm 0.07 $&\\
    \midrule
    \texttt{Hopper-v5} & Return & $1000 \pm 0$* & $1000 \pm 0$* & $1000 \pm 0$* &  $\bm{1005 \pm 9.98}$& $ 624 \pm 162 $ & \multirow{2}{*}{$1-\epsilon= 0.90$}\\
    & Safety Prob. & $\bm{1.000 \pm 0.00}$ & $\bm{1.000 \pm 0.00}$ & $\bm{1.000 \pm 0.00}$ & $0.976 \pm 0.0150$& $ 0.00 \pm 0.00 $&\\
    \bottomrule
  \end{tabular}
\end{table*}
%\new{
%\fb{relevance of what follows?}
%\ag{TODO: Move to related work?}

%}

\paragraph{Assessing sample efficiency} To fairly assess sample efficiency we compare only against model-free baselines (CPO, PPO-Lag), excluding MPS and DMPS since they assume full knowledge of deterministic dynamics. We also report A2C-Eval, which simply executes the learned policy $\widehat{\pi}$ without shielding after each policy update. Although unsafe in practice (as $\widehat{\pi}$ cannot be verified), this baseline is useful to monitor the learning progress of the underlying policy. Results in Fig.~\ref{fig:learningcurves} show that A2C-GP-Shield converges substantially faster than both CPO and PPO-Lag: in \texttt{cartpole}, $\widehat{\pi}$ reaches optimal performance more than one order of magnitude quicker, while in \texttt{mountain\_car} it achieves rapid convergence after the $n$ GP models become well calibrated. However, during early training, imperfect calibration leads to temporary safety violations. At the end of training, even though the underlying policy $\widehat \pi$ is unsafe while optimizing for rewards (demonstrating a clear trade-off), the shield prevents safety violations entirely. Full learning curves, including override statistics are provided in Appendix~\ref{sec:additional plots}.

\paragraph{Higher-dimensional systems} We further evaluated our approach on the \texttt{Hopper-v5} environment from Gymnasium \cite{towers2024gymnasium}, a standard MuJoCo \cite{todorov2012mujoco} benchmark featuring a simulated 2D robot that must hop forward without falling. The observation and action spaces are $n=11$ and $m=3$, respectively, with rewards consisting of a healthy survival bonus and forward velocity term. Each episode is limited to $T=1000$ timesteps. We designed a linear feedback controller around the nominal equilibrium configuration using discrete-time LQR, with $(A,B)$ matrices obtained from a finite-difference linearization of the MuJoCo dynamics (c.f., Appendix~\ref{sec:hopper}). 
The corresponding control invariant set was computed via ellipsoidal expansion and convex-hull verification. With $\epsilon_t = 10^{-4}$ and recovery horizon $H=100$, the shield intervened at every step, maintaining perfect safety and achieving the maximum healthy reward of $+1000.0$ as detailed in Tab.~\ref{tab:results}. %
These results demonstrate that while our framework is theoretically sound and guarantees safety by construction, its practical performance depends critically on the quality of the backup controller and the size of the verified control invariant set. For complex high-dimensional systems with contact dynamics, limit cycles, and other non-linear effects, verifying sufficiently large invariant sets remains an open challenge orthogonal to our approach. Developing richer or learned backup controllers capable of stabilizing over larger regions may therefore be essential for extending recovery-based shielding to such settings.

\begin{figure}[t]
    \centering
    \begin{subfigure}[t]{\linewidth}
        \centering
        \includegraphics[width=\textwidth]{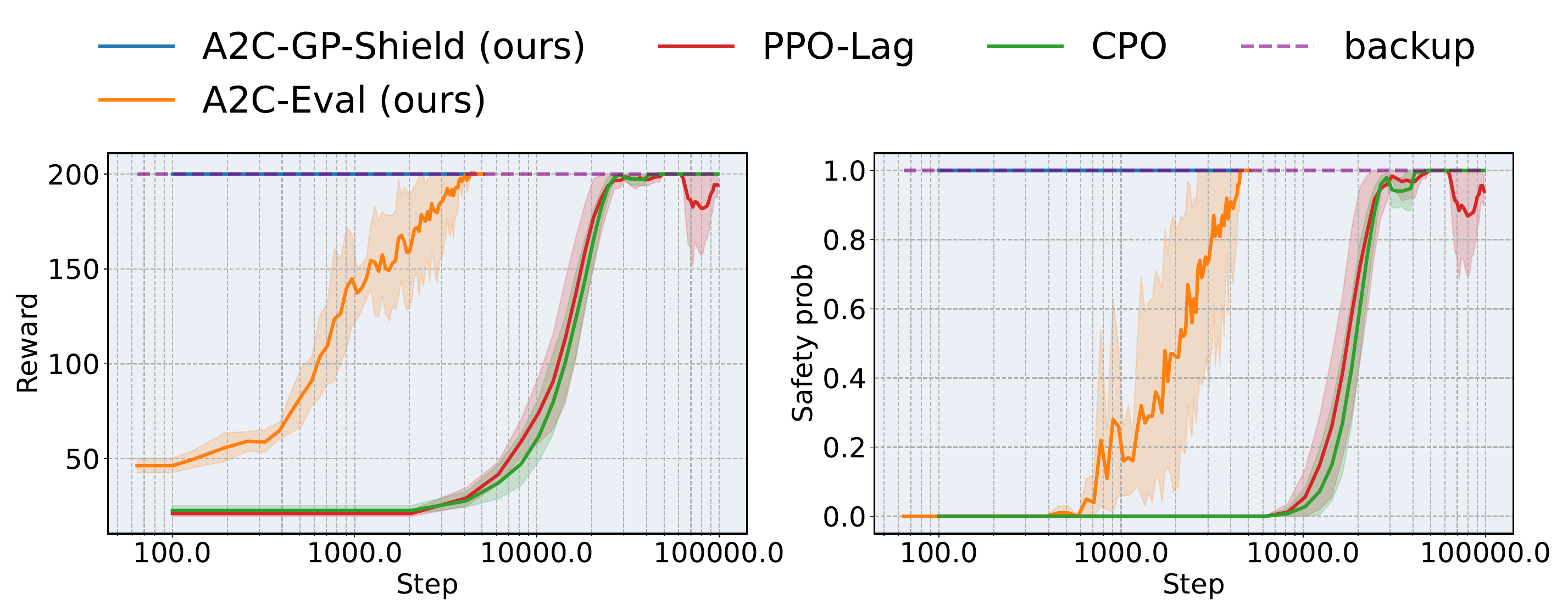}
        \caption{\texttt{cartpole}}
        \label{fig:cartpoleresults}
    \end{subfigure}
    \hfill
    \begin{subfigure}[t]{\linewidth}
        \centering
        \includegraphics[width=\textwidth]{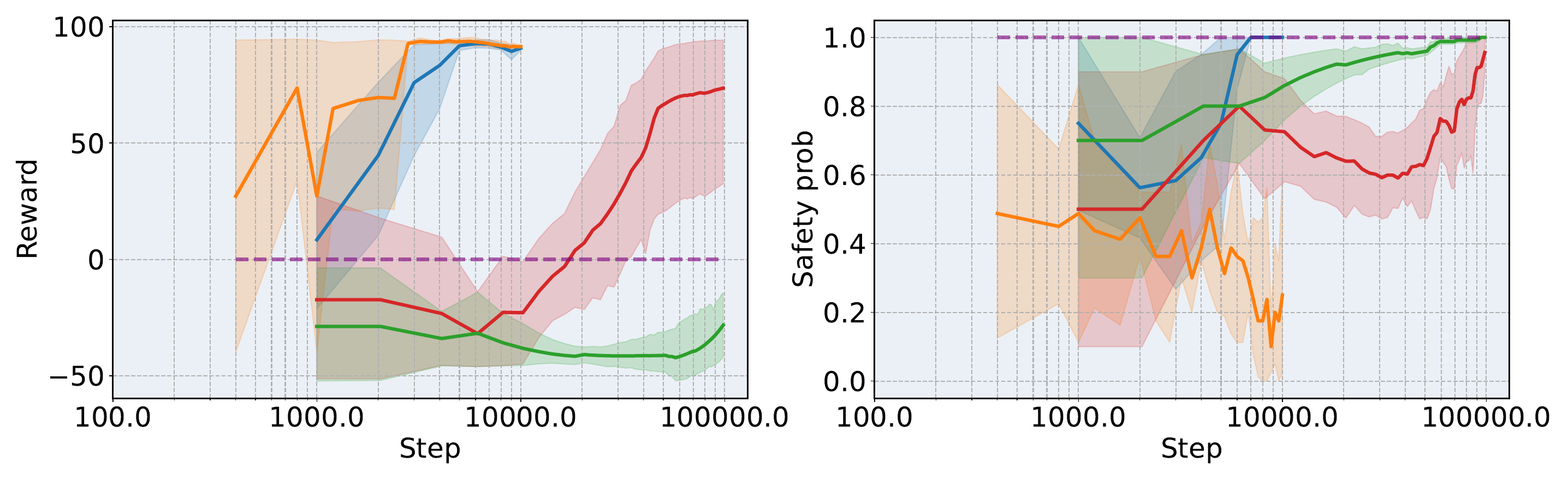}
        \caption{\texttt{mountain\_car}}
        \label{fig:mountaincarresults}
    \end{subfigure}
    \caption{Learning curves (log scale on the x-axis).}
    \label{fig:learningcurves}
\end{figure}

\section{Related work}
\label{sec:relatedwork}

%\new{
\paragraph{Constrained RL} The most widely adopted formulation for safe RL is the CMDP framework \cite{altman1999constrained}. In a CMDP, the agent optimizes its policy to maximize reward while satisfying a cost constraint (typically an expected cost or chance constraint) at each step or over each episode. A number of safe RL algorithms build on this framework, for example by using Lagrange relaxations \cite{ray2019benchmarking, as2022constrained}, or policy projection (e.g., \emph{Constrained Policy Optimization} \cite{achiam2017constrained}). While CMDP-based methods can improve the average safety of learning, they provide only indirect convergence guarantees, they often ensure constraints in expectation or asymptotically, which means they may still allow occasional, potentially dangerous violations during training. There exists GP methods that fall within this category, i.e., which model an unknown safety cost function through GPs \cite{wachi2018safe, wachi2020safe}. These methods are not without their limitations, making often unrealistic assumptions, e.g., assuming known deterministic dynamics or access to a privileged ``emergency reset'' button \cite{wachi2024safe}. In summary, CMDP guarantees tend to be \emph{weaker} and not suited for highly safety-critical environments \cite{voloshin2022policy}. 
%}

\paragraph{Shielding} In contrast to CMDP methods, \emph{shielding} often offers more rigorous guarantees. For RL, shielding was first introduced by Alshiekh et al. \cite{alshiekh2018safe} for safeguarding against formal logic specifications in discrete-state systems. Provided with a suitable (logical) safety-abstraction of the environment, shielding filters out unsafe actions proposed by the agent, enjoying key properties, such as \emph{correctness} and \emph{minimal interference}. For stochastic environments, this approach has recently been extended to optimality preserving \emph{probabilistic} guarantees for generic \emph{reachability properties} \cite{court2025probabilisticshieldingsafereinforcement}. In continuous control domains, implementing shielding requires computing a safe region of the state space and a policy to keep the agent within it. One line of work employs Hamilton-Jacobi (HJ) reachability analysis \cite{akametalu2014reachability} to derive the maximal safe set and a corresponding safe controller. Fisac et al. \cite{fisac2018general} generalized this idea to RL with uncertain dynamics by formulating a safety value function (via an HJ partial differential equation) that triggers an override when the agent is about to exit the safe set. HJ-based shielding methods offer strong safety guarantees for nonlinear systems, including the ability to handle non-convex state constraints \cite{wabersich2023data}. However, as noted earlier, they suffer from poor scalability. To improve scalability, other works focus on computing control invariant sets (often convex approximations of the safe region) for complex dynamics \cite{bastani2021safe,bastanib2021safe,li2020robust}. For instance, robust invariant set computation techniques have been applied to design shields that work for certain classes of non-linear systems \cite{li2020robust}. Furthermore, The concept of shielding has also been extended to higher dimensional systems with learned \emph{world models} \cite{he2022androids,goodall2023approximate}, although these methods come with some theory \cite{goodall2023approximate} they often any lack strict guarantees. 
%\new{
Our work addresses this more general case where the system dynamics are initially unknown, requiring the agent to learn them, and thus we build on the concept of shielding combined with online model learning. Sampling based techniques, e.g., \emph{Statistical MPS} \cite{bastanib2021safe}, can be used for stochastic models of the environment; for a step-wise safety probability $\epsilon_t$ and confidence $\delta_t$ \citet{bastanib2021safe} show that the required number of samples satisfies, $K(\epsilon_t, \delta) \geq \frac{\log(1/\delta_t)}{\log(1/(1-\epsilon_t))} + 1$. Substituting $\epsilon_t=10^{-5}$ (which was used in our experiments) and $\delta_t=0.01$ yields $K \approx 46{,}000$ samples per check, which is computationally infeasible, even with approximate sampling regimes \cite{Pinder2022} (c.f., Appendix~\ref{sec:samplinganalysis}). 
Notably, our method does away with the confidence $\delta_t$ (essentially absorbing it into the GP learning) and as discussed earlier can handle arbitrarily small $\epsilon_t$ without additional overhead. 
%}
\balance

%\paragraph{Safe RL with GP} 
\emph{Gaussian Process} (GP) models have long been used in RL for data-efficient learning of dynamics and rewards. PILCO \cite{deisenroth2011pilco} is a landmark model-based RL approach that uses GP dynamics models to achieve impressive sample efficiency. However, PILCO's original formulation had no safety considerations. An extension by Polymenakos et al. \cite{polymenakos2019safe} incorporated a safety constraint into the PILCO framework; this approach is limited by considering only a small, specific class of policies (e.g., linear and RBF) and did not provide rigorous safety guarantees for those policies. More generally, GPs have been integrated into safe RL algorithms primarily through the CMDP lens discussed above \cite{wachi2018safe, wachi2020safe, wachi2024safe}. 
%\new{
A different approach is to use learned models within a control-theoretic safety scheme. \emph{Learning-based Model Predictive Control} (MPC) has been explored as a way to enforce safety during RL. Koller et al. \cite{koller2018learning} pioneered a GP-based MPC for safe exploration, which uses GPs to construct high-confidence bounds on the system’s behaviour and then optimizes a control sequence that satisfies state constraints with high probability. Although, their framework only considers optimization over a sequence of time-dependent linear controllers. Similar work from Cheng et al. \cite{cheng2019end} consider control barrier functions (CBF) with GP uncertainty estimates, solving a QP for high-probability guarantees. More recently, Zhao et al. \cite{zhao2023probabilistic} proposed a \emph{Probabilistic Safeguard} framework that also leverages GP models for safety. In their approach, the agent first gathers an \emph{offline dataset} of safe trajectories to train a GP model, then defines a \emph{safety index} (analogous to a CBF). This safeguard method shares our philosophy of combining learning with a safety filter; however, it relies on designing an appropriate safety index and primarily addresses the offline training phase, whereas our work emphasizes online safe exploration and learning.

\section{Conclusion}
\label{sec:conclusion}

We proposed a recovery-based shielding framework that integrates Gaussian process (GP) dynamics models with formal safety semantics to guarantee a probabilistic lower bound on safety. Our approach propagates analytic uncertainty sets through the GP model and leverages a precomputed backup policy to certify $\epsilon$-recoverability of states. This expands the set of \emph{operationally safe} states beyond the control invariant set, while avoiding sampling-based verification and thus remaining efficient even under strict safety requirements.
Our implementation, A2C-GP-Shield, supports both SGPR \cite{titsias2009variational} and SVGP \cite{NIPS2015_6b180037}, enabling scalable training with large replay buffers. Nonetheless, GP inference can become computationally demanding in higher dimensions (e.g., $n>20$), suggesting that deep kernel learning may be a promising direction for future work. Extending our framework to handle moving or dynamic obstacles is another natural step: while our method applies directly when obstacle dynamics are known, learning or estimating unknown obstacle dynamics presents an open challenge. Finally, removing the reliance on an a priori backup policy, by computing or optimizing the backup controller online, could make the framework more broadly applicable in scenarios where such a controller is unavailable.

\begin{acks}
The research described in this paper was partially supported by the EPSRC (grant number EP/X015823/1).
\end{acks}

%%%%%%%%%%%%%%%%%%%%%%%%%%%%%%%%%%%%%%%%%%%%%%%%%%%%%%%%%%%%%%%%%%%%%%%%

%%% The next two lines define, first, the bibliography style to be 
%%% applied, and, second, the bibliography file to be used.

\bibliographystyle{ACM-Reference-Format} 
\bibliography{sample}

%%%%%%%%%%%%%%%%%%%%%%%%%%%%%%%%%%%%%%%%%%%%%%%%%%%%%%%%%%%%%%%%%%%%%%%%
\newpage
\onecolumn
\appendix

\section{Algorithms}
\label{sec:algorithms}

\begin{algorithm}[H]
    \caption{Recovery-based Shielding with GP}
    \label{alg:fullalgorithm}
    \raggedright
    \textbf{Input}: $\mathcal{X}_{\text{safe}}$, $\pi_{\text{backup}}$, $\mathcal{X}_{\text{inv}}$, $\epsilon_t \in [0,1]$, $N \in \mathbb{N}$\\
    \textbf{Output}: Policy $\widehat \pi$, value function $V_\phi$ (and target network $V_{\phi'}$) and $n$ GP dynamics models $p_{\theta}$.
    
    \begin{algorithmic}[1] %[1] enables line numbers
        \State Initialize $\widehat\pi$, $\phi$, $\phi'$ and $\theta$ appropriately and replay buffer $\mathcal{D}$ with random experience.
        \State Observe the initial state $\widehat x(0)$.
        \For{$t = 0, \ldots, \texttt{num\_steps}$}
            \State $u(t) \gets \texttt{Shield}(\widehat x(t), \widehat \pi, \pi_{\text{backup}}, \epsilon_t, N)$
            \State Use $u(t)$ and observe $\widehat x(t+1)$.
            \State Append $\langle u(t), x(t), x(t+1) \rangle$ to $\mathcal{D}$.
            \If{\texttt{time\_to\_update}}
                \State Update GP parameters $\theta$ with dataset $\mathcal{D}$.
                \State Sample a batch of starting states $B \in \mathcal{D}$.
                \State From each $x(\cdot) \in B$ rollout $\widehat\pi$ with $p_{\theta}$.
                \State Compute GAE ($\widehat{A}_t^{\lambda}$) and TD-$\lambda$ returns ($R^\lambda_t$) from rollouts.
                \State Update $\phi$ via MSE with $\text{symlog}(R^\lambda_t)$ and target network regularization.
                \State Update $\phi'$ via Polyak averaging \cite{polyak1992acceleration}.
                \State Update $\widehat\pi$ via policy gradient. 
            \EndIf
        \EndFor
        \State \textbf{return} $\widehat \pi$ and $p_{\theta}$.
    \end{algorithmic}
\end{algorithm}

\section{Proof of Theorem \ref{thm:epsilonsafe}}
\label{sec:proof}

\begin{restatedthm}{\ref{thm:epsilonsafe} (restated)}[$\epsilon$-Safe Policy] Assume: (i) $\mathcal{X}_{\text{inv}}$ is an invariant set for $\pi_{\text{backup}}$, (ii) $\mathcal{X}_0 \subseteq \mathcal{X}_{\text{inv}} \subseteq \mathcal{X}_{\text{safe}}$ and (iii) $\sum^{T}_{t=0}\epsilon_t = \epsilon$, then, $\pi_{\text{shield}}$ is $\epsilon$-safe for the time horizon $T \in \mathbb{N}$. 
\end{restatedthm}

\begin{proof}[Proof]
The proof of Theorem \ref{thm:epsilonsafe} is obtained by establishing the following invariant: ``we can always return to $\mathcal{X}_{\text{inv}}$ with probability at least $1  - \epsilon_t$ by using the backup policy $\pi_{\text{backup}}$ for $N$ timesteps''. Since the initial state $x(0) \in \mathcal{X}_{\text{inv}}$ and $\mathcal{X}_{\text{inv}}$ is invariant for $\pi_{\text{backup}}$, then this holds for $t=0$. By the inductive assumption we assume that the invariant has been established for $t-1$, if $x(t) \in \mathcal{X}^{\epsilon_t}_{\text{rec}}(N)$ we are safe to use the learned policy $\widehat\pi$ and by definition of $\mathcal{X}^{\epsilon_t}_{\text{rec}}(N)$ we have established the invariant. If $x(t) \not\in \mathcal{X}^{\epsilon_t}_{\text{rec}}(N)$ then we use the backup policy $\pi_{\text{backup}}$, which will return us to $\mathcal{X}_{\text{inv}}$ within $N$ timesteps as we have established the invariant in timestep $t-1$. The remainder of the proof is obtained by a union bound over the step-wise safety constraints $\epsilon_0,\epsilon_1,\ldots, \epsilon_T$.
\end{proof}

\section{Mean and covariance prediction (PILCO)}
\label{sec:meancovpred}
Here we detail the analytic mean and covariance prediction proposed in PILCO \cite{deisenroth2010efficient}.

\paragraph{Mean prediction ($\mu_{\Delta}$).} For the state dimensions $i = 1, \ldots,n$, we let $\mu_{\Delta}(t) = [\mu_1, \ldots \mu_n]^T$. Following the law of iterated expectations, we then obtain,
\begin{align}
    \mu_i &= \mathbb{E}_{\tilde x(t-1)}[\mathbb{E}_f[f(\tilde x (t-1))\mid \tilde x (t-1)]] \label{eq:deltapred}\\
    &= \mathbb{E}_{\tilde x(t-1)}[m_f(\tilde x (t-1))]\\
    &= \int m_f(\tilde x(t-1))\mathcal{N}(\tilde x(t-1) \mid \tilde\mu(t-1), \tilde \Sigma(t-1)) d \tilde x (t-1)\\
    & = \beta_i^T Q_i
\end{align}
with $\beta_i = (K_i + \sigma^2_{\varepsilon_i})^{-1}y_i$ and $Q_i= [ q_{1_i}, \ldots q_{D_i} ]^T$. Here $m_f$ is given by (\ref{eq:meanpred}), see Sec.~\ref{sec:dynamicsmodelling}. For $a = 1, \ldots, D$, the entries of $Q_i \in \mathbb{R}^D$ are then given by,
\begin{align}
    q_{a_i} &=\int k_i(\tilde x_a(\cdot), \tilde x(t-1)) \mathcal{N}(\tilde x(t-1) \mid \tilde\mu(t-1), \tilde \Sigma(t-1)) d \tilde x (t-1)\\
    &= \frac{\alpha^2_i}{\sqrt{\lvert \tilde \Sigma(t-1)\Lambda_i^{-1} + \mathbf{I} \rvert}}\exp(-\frac{1}{2} v_a^T(\tilde \Sigma(t-1) + \Lambda_i)^{-1}v_a)\\
    & \text{where }  \;v_a := (\tilde x_a(\cdot) - \tilde\mu(t-1)) \label{eq:vi}
\end{align}
The value $v_a$ is the difference between the $a^{\text{th}}$ training input, i.e., $\tilde x_a(\cdot)$ and the mean of the input distribution $p(\tilde x(t-1)) = p(x(t-1), u(t-1))$ calculated earlier, see Sec.~\ref{sec:dynamicsmodelling}.

\paragraph{Covariance prediction ($\Sigma_{\Delta}$).} To present the computation the covariance matrix $\Sigma_{\Delta}(t) \in \mathbb{R}^{n \times n}$ it is convenient to distinguish between diagonal $\sigma_{ii}^2$ and off diagonal elements $\sigma_{ij}^2$ ($i \neq j$) of $\Sigma_{\Delta}(t)$. Using the law of iterated variances for the state dimensions $i, j = 1, \ldots n$, we obtain,
\begin{align}
        \sigma_{ii}^2 &= \mathbb{E}_{\tilde x(t-1)}[\text{var}_f[\Delta_i \mid \tilde x(t-1)] + \mathbb{E}_{f, \tilde x(t-1)}[\Delta^2_i] - \mu_i^2\\
    \sigma_{ij}^2 &= \mathbb{E}_{f, \tilde x(t-1)}[\Delta_i\Delta_j]-\mu_i\mu_j \quad (i \neq j) \label{eq:offdiag}
\end{align}
where $\mu_i$ is computed using (\ref{eq:deltapred}) from before. We note that the usual covariance term is missing, in this case $\mathbb{E}_{\tilde x (t-1)}[\text{cov}_f[\Delta_i, \Delta_j \mid \tilde x (t-1)]]$, this is because given $ \tilde x (t-1)$ output dimensions are independent. Continuing, we have,
\begin{equation}
    \mathbb{E}_{f, \tilde x(t-1)}[\Delta_i\Delta_j] = \beta_i^TQ\beta_j
\end{equation}
where
\begin{equation}
    Q := \int k_i(\tilde X, \tilde x(t-1))k_j(\tilde X, \tilde x(t-1))^Tp(\tilde x (t-1)) d \tilde x (t-1)
\end{equation}
For $a, b = 1, \ldots D$, The entries of $Q_{ab}$ of $Q \in \mathbb{R}^{D \times D}$ are given by standard integration results for Gaussian distributions,
\begin{equation}
    Q_{ab} = \frac{k_i(\tilde x_a(\cdot), \tilde \mu(t-1))k_j(\tilde x_b (\cdot), \tilde \mu(t-1))}{\sqrt{\lvert R \rvert}} \cdot \exp\left(\frac{1}{2}z_{ab}^T R^{-1}\tilde \Sigma(t-1)z_{ab}\right) \label{eq:qij}
\end{equation}
where $R:= \tilde \Sigma(t-1)(\Lambda^{-1}_i + \Lambda_j^{-1}) + \mathbf{I}$ and $z_{ab}=\Lambda_i^{-1}v_{a} + \Lambda_{j}^{-1}v_b$, where $v_a$ defined in (\ref{eq:vi}). This concludes the computation of the off diagonals, which are now fully determined by Eqs.~(\ref{eq:deltapred})-(\ref{eq:vi}) and Eqs.~(\ref{eq:offdiag})-(\ref{eq:qij}). For the diagonal entries of $\Sigma_{\Delta}$, it remains to compute the additional term,
\begin{equation}
    \mathbb{E}_{\tilde x(t-1)}[\text{var}_f[\Delta_i \mid \tilde x(t-1)] = \alpha^2_i - \text{tr}((K_i + \sigma^2_{\varepsilon_i} \mathbf{I})^{-1}Q)
\end{equation}
with $Q$ given by (\ref{eq:qij}), and $\text{tr}((K_i + \sigma^2_{\varepsilon_i} \mathbf{I})^{-1}Q)$ denoting the trace of the matrix $(K_i + \sigma^2_{\varepsilon_i} \mathbf{I})^{-1}Q$.

\section{Gaussian process theory (Theorem ~\ref{thm:ellipsoids})}
\label{sec:gaussianprocessappendix}
\begin{restatedthm}{\ref{thm:epsilonsafe} (restated)}
Assume: (i) the unknown dynamics $f$ and $\pi_{\text{backup}}$ are both Lipschitz continuous in the $L_1$-norm, (ii) the $n$ GP dynamics models are well calibrated in the sense that with probability (w.p.) at least $1 - \delta$ there exists $\beta(\tau)>0$ such that $\forall t=0 \ldots N$ and $\forall w(t) \in \mathcal{W}$ we have $\lVert f(x(t), u(t), w(t)) - \mu(t) \rVert_1 \leq \beta(\tau) \text{tr}(\Sigma(t))$, (iii) $\forall t=0 \ldots N$ the higher order terms (e.g., skewness, kurtosis) of the true distribution $p(x(t))$ are negligible and can be ignored. Then if $\mathcal{E}(0), \ldots, \mathcal{E}(N) \subseteq \mathcal{X}_{\text{safe}}$ and if there exists $t \in \{0, \ldots N \}$ such that $\mathcal{E}(t) \subseteq \mathcal{X}_{\text{inv}}$ then $\widehat x(0)$ is $\epsilon_t$-recoverable with probability at least $1-\delta$. Thus $\pi_{\text{shield}}$ is $\epsilon$-safe by Theorem \ref{thm:epsilonsafe} (w.p., $1-\delta$). 
\end{restatedthm}
\begin{proof}
To enable the use of Gaussian process and satisfy the Lipschitz continuity property of the unknown function $f$ a common technique is used, which assume the disturbances $w(t) \in \mathcal{W}$ live in some reproducing kernel Hilbert space, corresponding to a differential kernel $k$ and have RKHS norm smaller and $B_\mathcal{W}$. This is a \emph{model agnostic} approach that does not restrict the true function $f$ to the space of Gaussian process priors \cite{srinivas2009gaussian}.
\begin{lemma}[\cite{srinivas2009gaussian} Theorem 6] For some arbitrary function $f : \mathcal{Z} \to \mathbb{R}$ with zero mean prior and kernel $k$ in some RKHS with norm smaller than $B$, with observations $y = f(z) + \sigma$ assuming sub-Gaussian noise in the range $[-\sigma_{\text{max}}, \sigma_{\text{max}}]$. Letting $\mu_n(z)$ and $\sigma^2_n(z)$ denote the mean and variance given $n$ noisy observations. Pick $\beta_n = B_{\mathcal{W}} + 4\sigma\sqrt{\tau_n + 1 + \ln(1/\delta)}$. Then with probability at least $1 - \delta$, for $\delta \in (0, 1)$, for all $n\geq 1$, $z \in \mathcal{Z}$ and $u(\cdot) \in \mathcal{U}$ it holds that $\lVert f(z) - \mu_{n-1}(z)\rVert \leq \beta_n \sigma_{n-1}(z)$
    \label{lem:srin}
\end{lemma}
\begin{remark} Lemma \ref{lem:srin} can be easily extended to an extended parameter space of $\mathcal{X} \times \mathcal{U}$ \cite{berkenkamp2017safe} and extended to multi output functions using the proof techniques in (\cite{chowdhury2017kernelized} Theorem 2).  While Lemma \ref{lem:srin} makes different assumptions than usual GP assumptions (e.g.~$f \sim \mathcal{GP}(0, k(\cdot,\cdot))$, Gaussian noise variances), it demonstrates that the error bounds are conservative enough to capture the true function with high probability.
\end{remark}
We further note that the bound in \ref{lem:srin} depends on the information capacity parameter $\tau_n$, which corresponds to the maximum mutual information that could be gained about the unknown function $f$ from samples. For many commonly used kernels, including the squared exponential kernel (which we consider in this paper), the parameter $\tau_n$ is sub-linear in the number of samples $n$ and can be accurately approximated \cite{srinivas2009gaussian}.

Another approach is to directly assume the unknown function $f$ is a sample from a Gaussian process $\mathcal{GP}(0, k(\cdot,\cdot))$, using the Lipschitz constant of $k$ and $f$ on the closed set $\mathcal{X}\times \mathcal{U}$ we can uniformly bound the error of $f$  (confidence interval) with high probability for all $(x(\cdot), u(\cdot)) \in \mathcal{X}\times \mathcal{U}$. This requires knowledge of the Lipschitz constant of $f$, which can robustly estimated \cite{lederer2019uniform}.

In either cases we have that assumption (ii) holds with high probability, then by assumption (iii) the result of Thm.~\ref{thm:ellipsoids} holds since the uncertainty sets $\mathcal{E}(0), \ldots, \mathcal{E}(N)$ now contain at least $1-\epsilon_t$ of the probability mass of the true distributions $p(x(0)), \ldots p(x(N))$ which implies that $\widehat x(t)$ is $\epsilon_t$-recoverable by the definition of $\epsilon$-recoverable states.
\end{proof}

\clearpage
\newpage
\section{Additional plots}
\label{sec:additional plots}

\begin{figure*}[ht!]
    \centering
    \begin{subfigure}[t]{0.49\linewidth}
        \centering
        \includegraphics[width=0.49\textwidth]{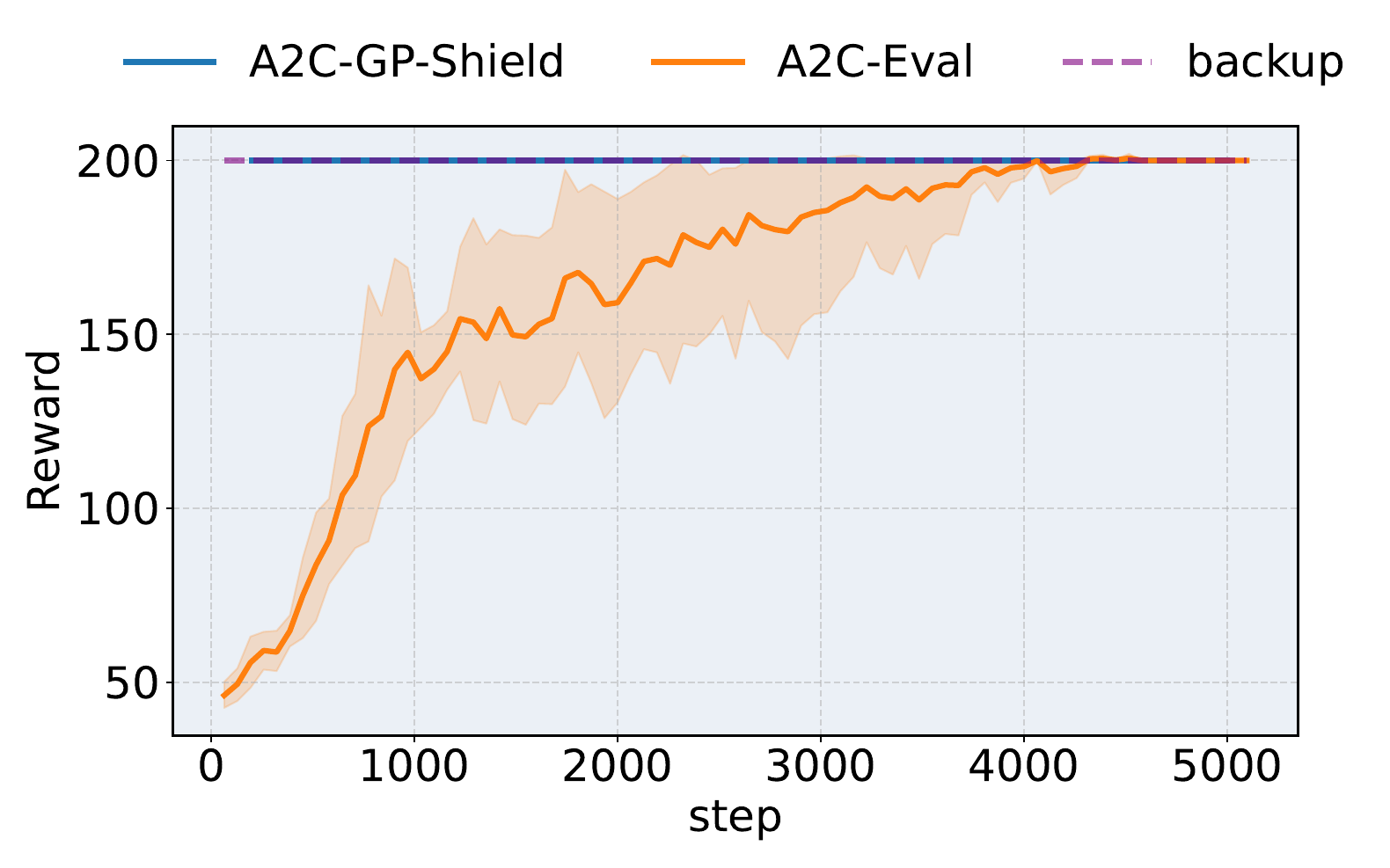}
        \includegraphics[width=0.49\textwidth]{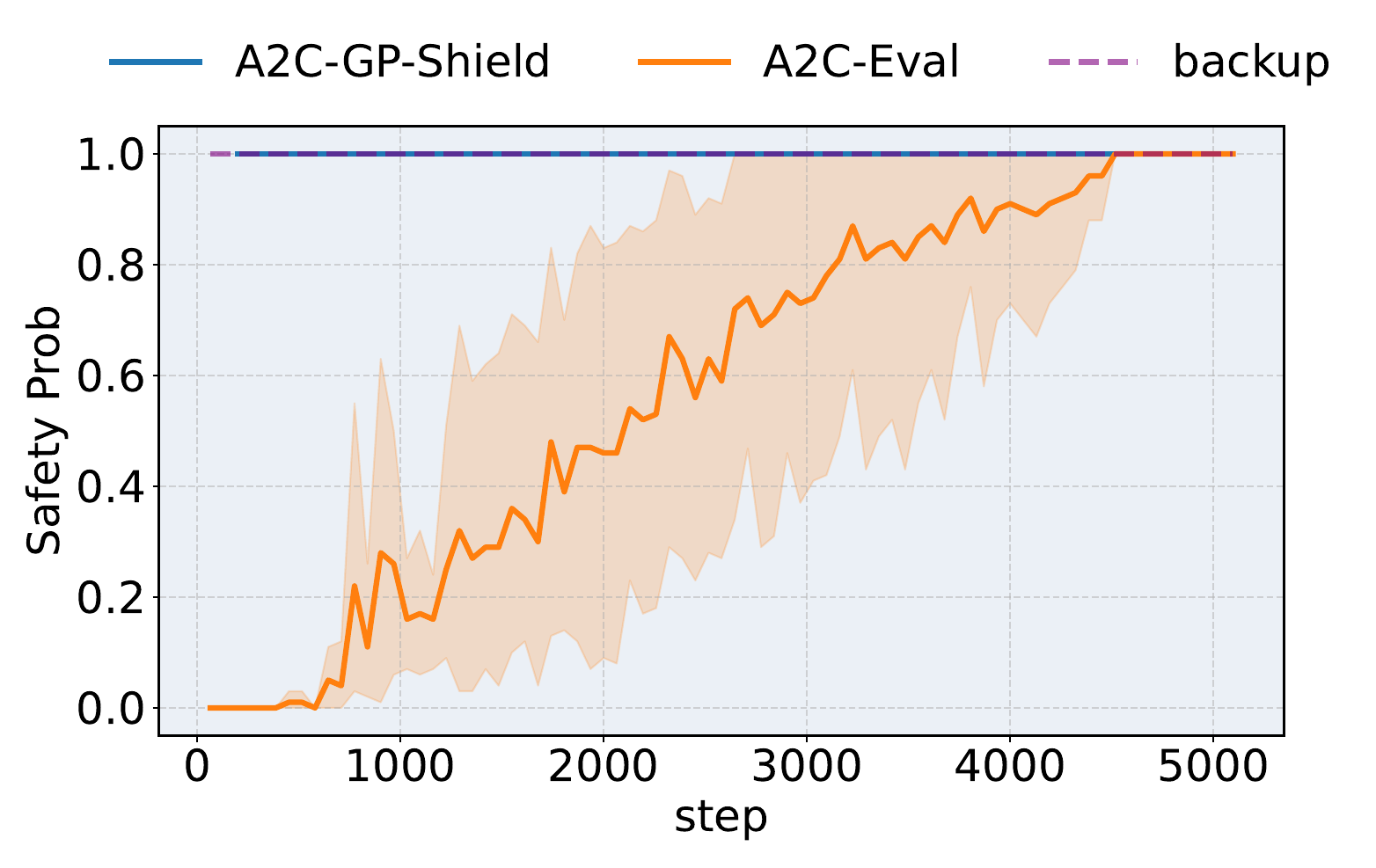}
        \caption{\texttt{cartpole (i)}}
        
    \end{subfigure}
    \hfill
    \begin{subfigure}[t]{0.49\linewidth}
        \centering
        \includegraphics[width=0.49\textwidth]{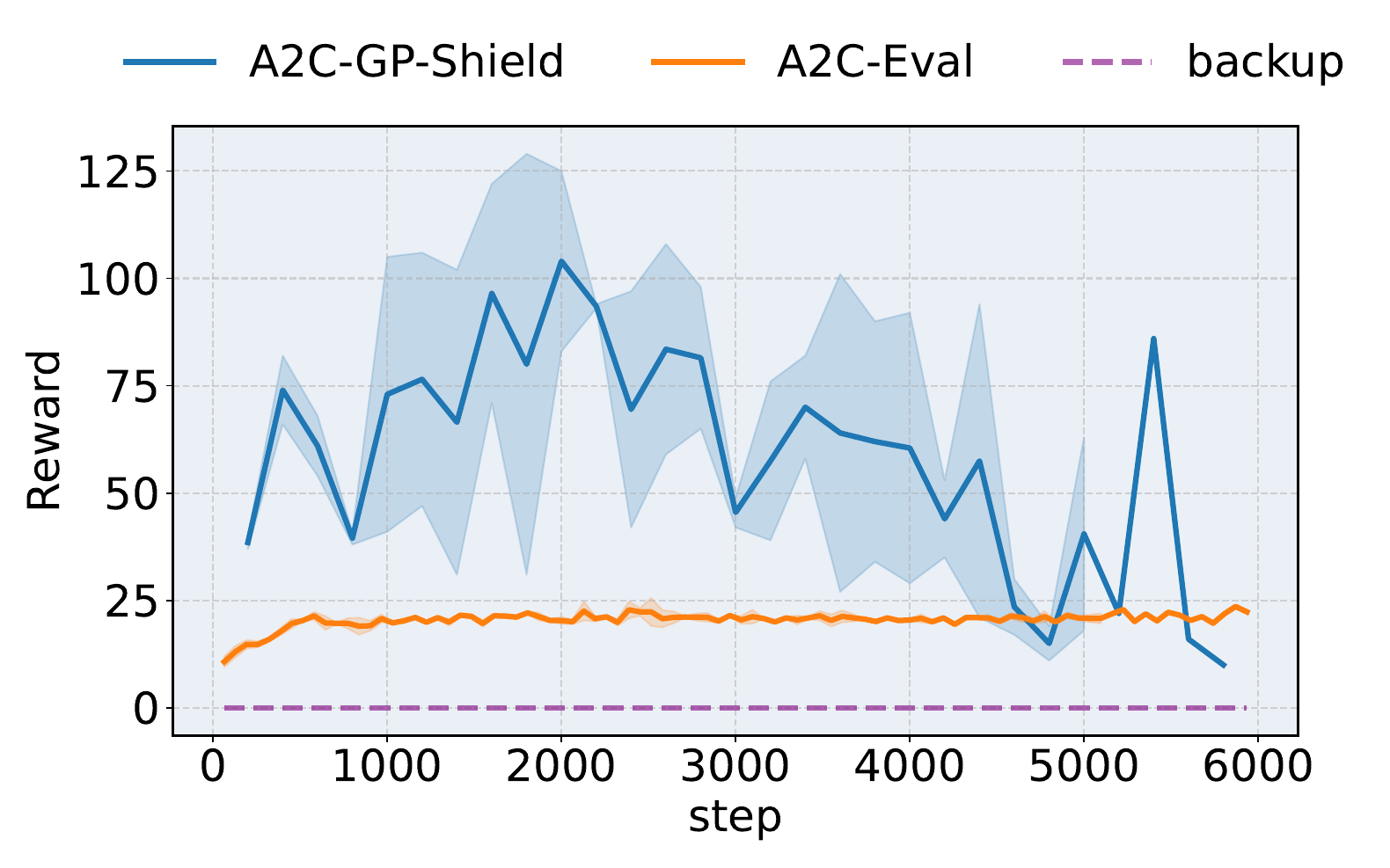}
        \includegraphics[width=0.49\textwidth]{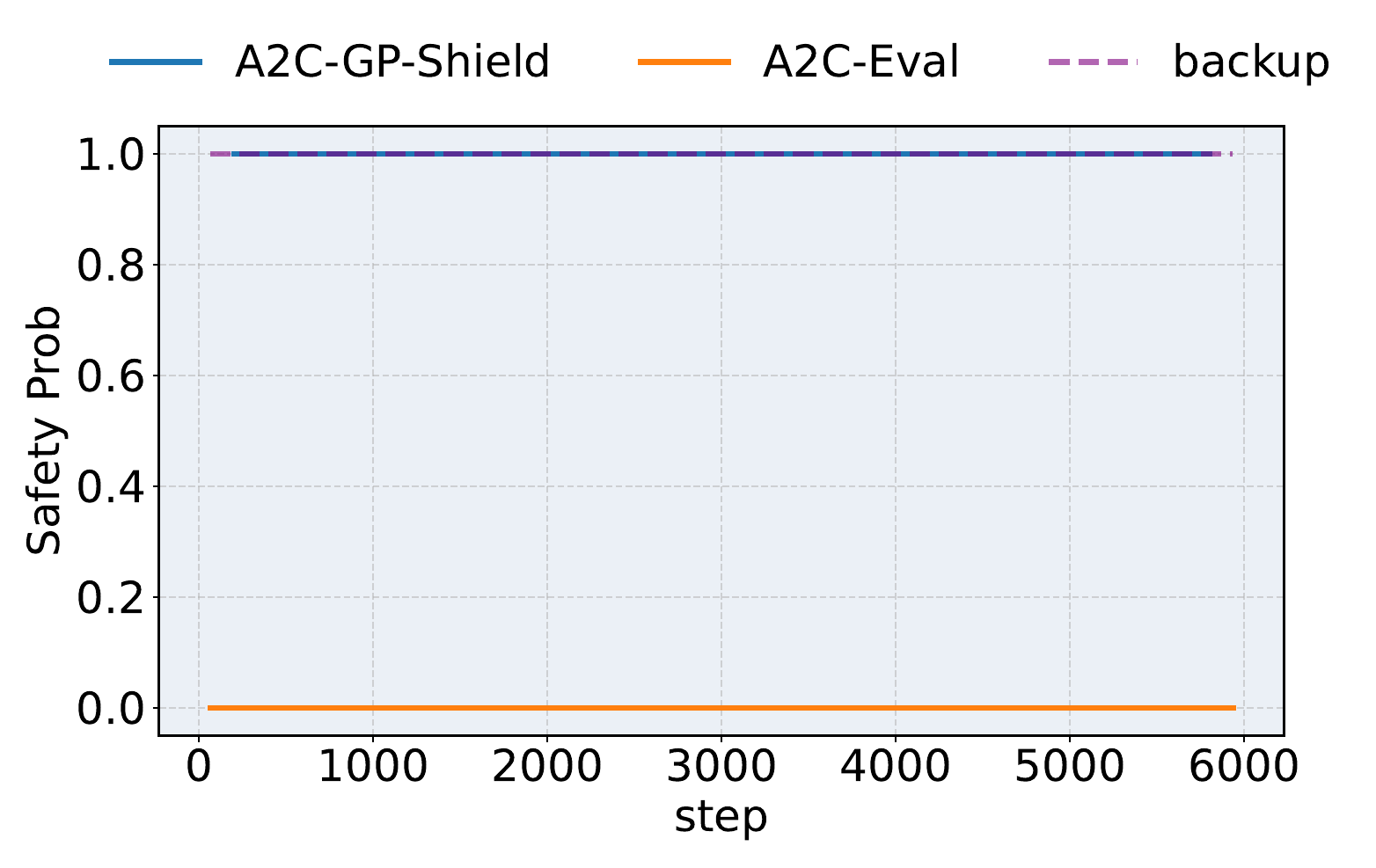}
        \caption{\texttt{cartpole (ii)}}
        
    \end{subfigure}
    \hfill
    \begin{subfigure}[t]{0.49\linewidth}
        \centering
        \includegraphics[width=0.49\textwidth]{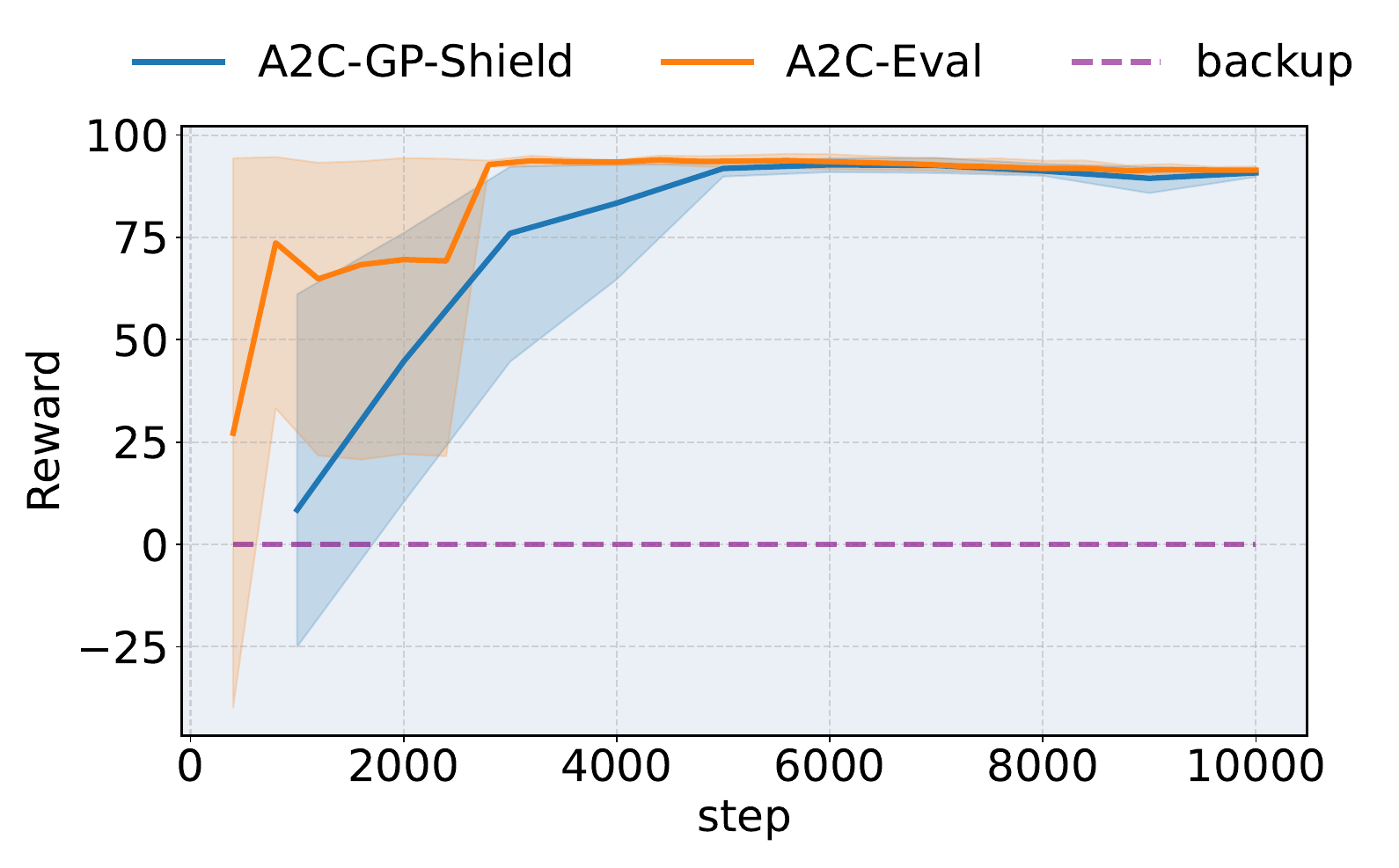}
        \includegraphics[width=0.49\textwidth]{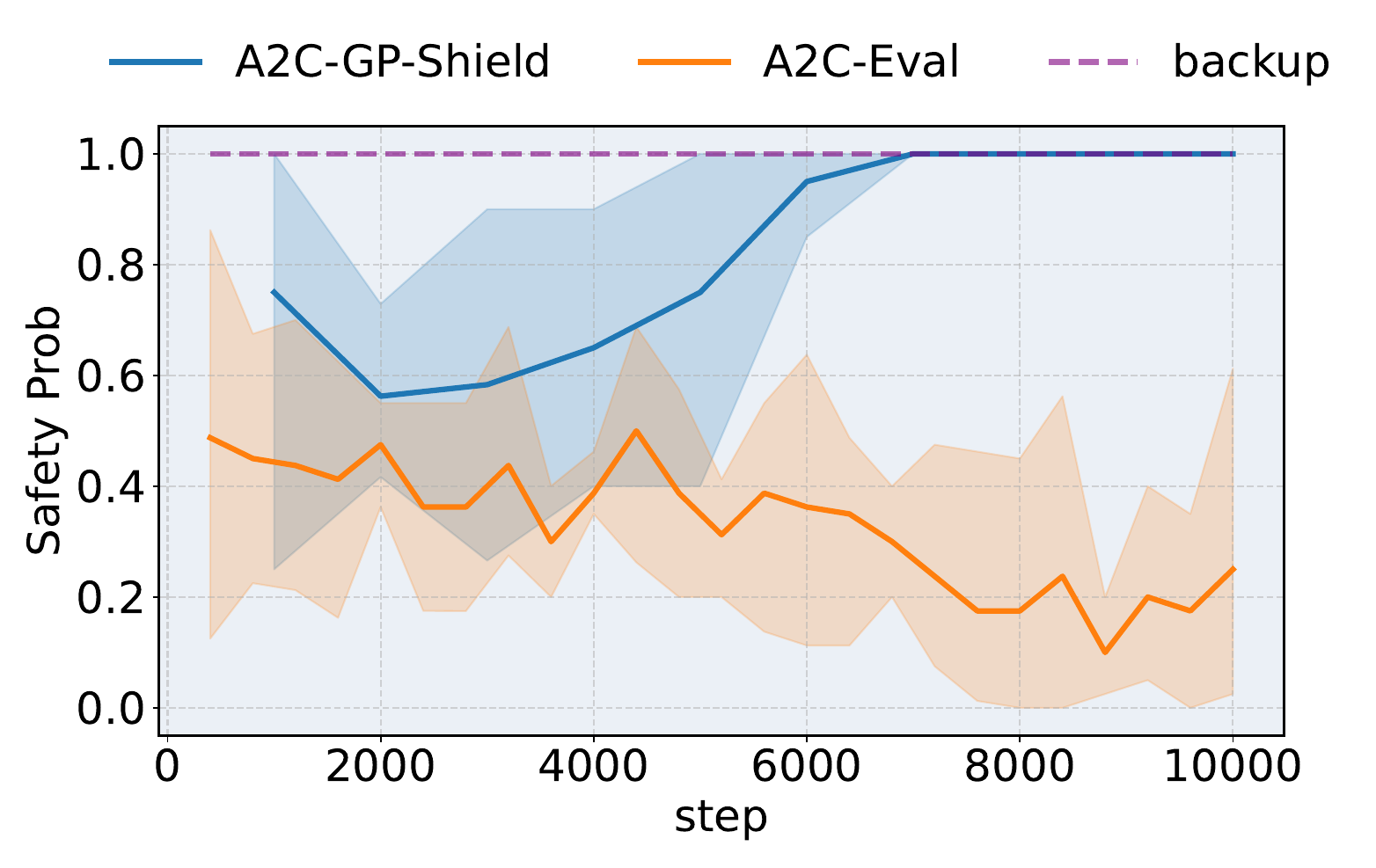}
        \caption{\emph{mountain\_car}}
        
    \end{subfigure}
    \hfill
    \begin{subfigure}[t]{0.49\linewidth}
        \centering
        \includegraphics[width=0.49\textwidth]{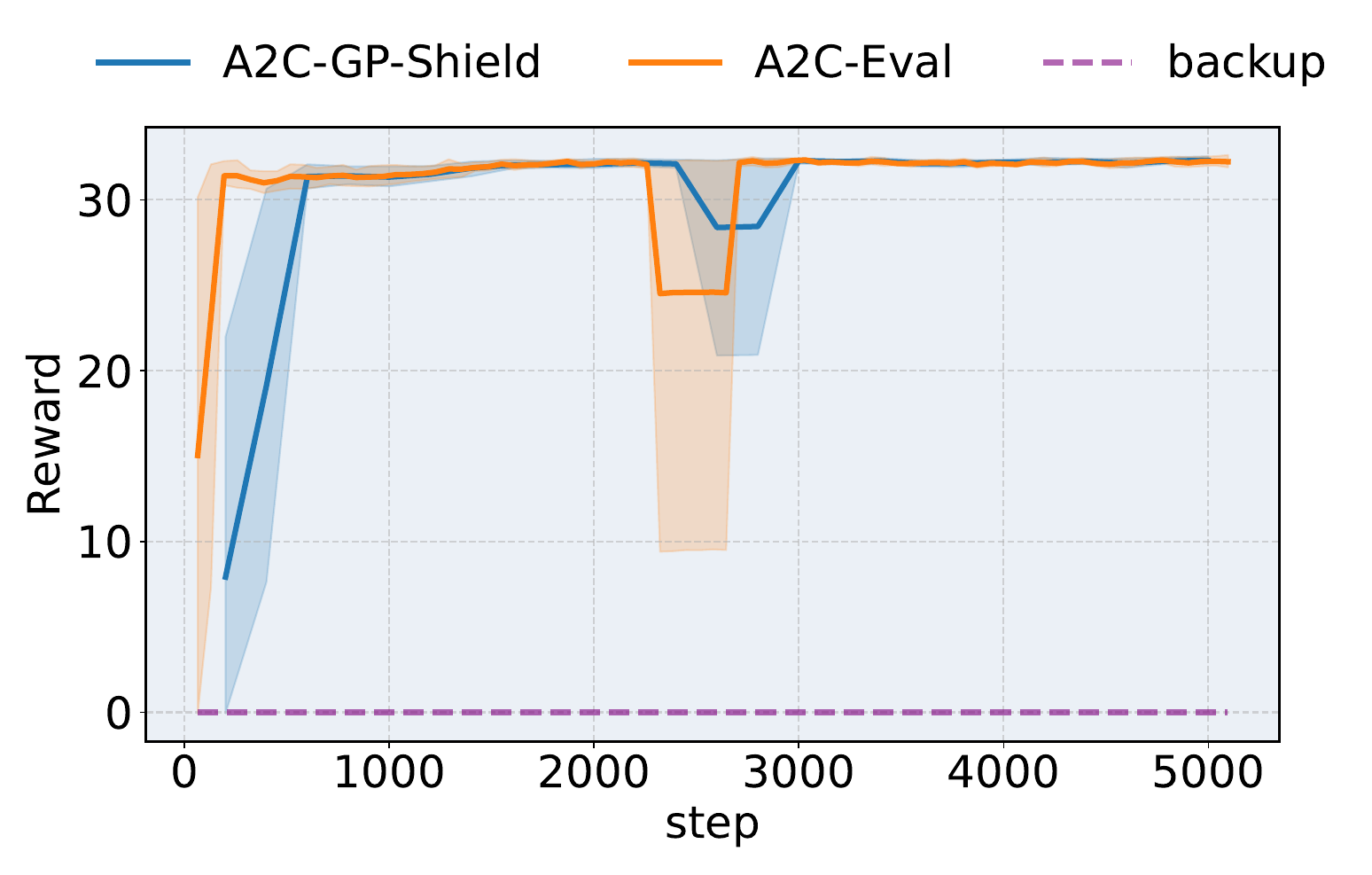}
        \includegraphics[width=0.49\textwidth]{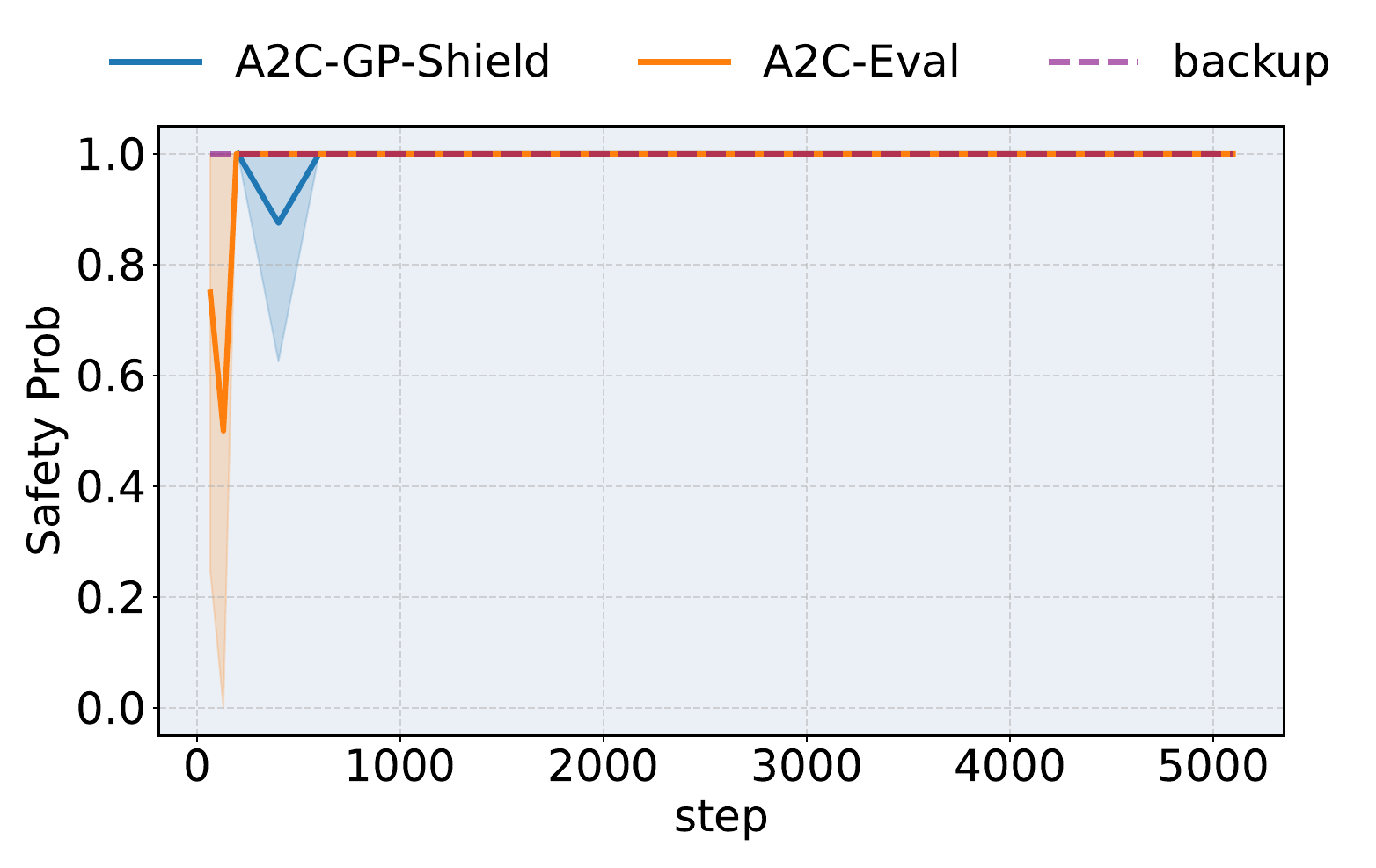}
        \caption{\texttt{obstacle}}
        
    \end{subfigure}
    \hfill
    \begin{subfigure}[t]{0.49\linewidth}
        \centering
        \includegraphics[width=0.49\textwidth]{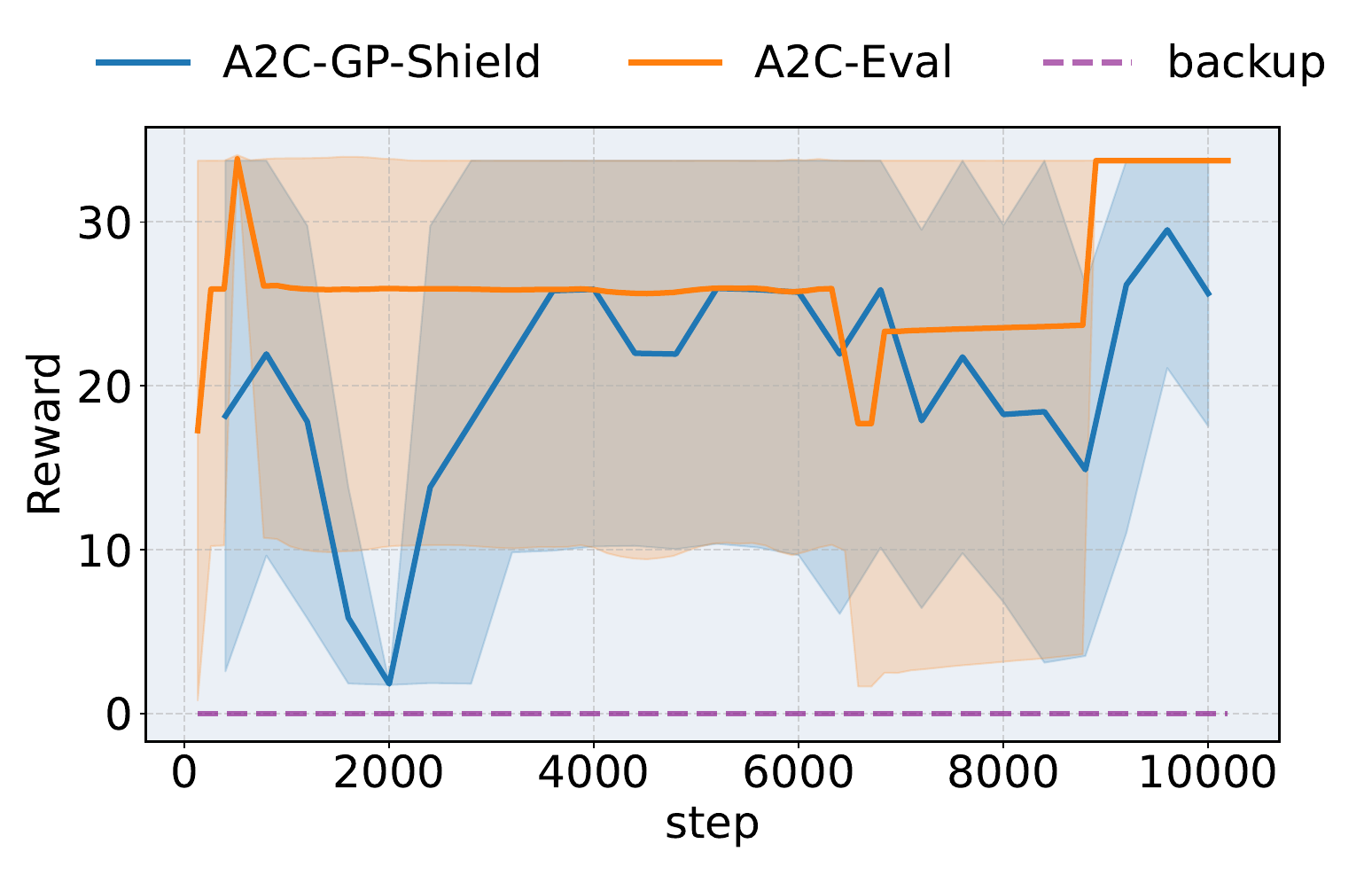}
        \includegraphics[width=0.49\textwidth]{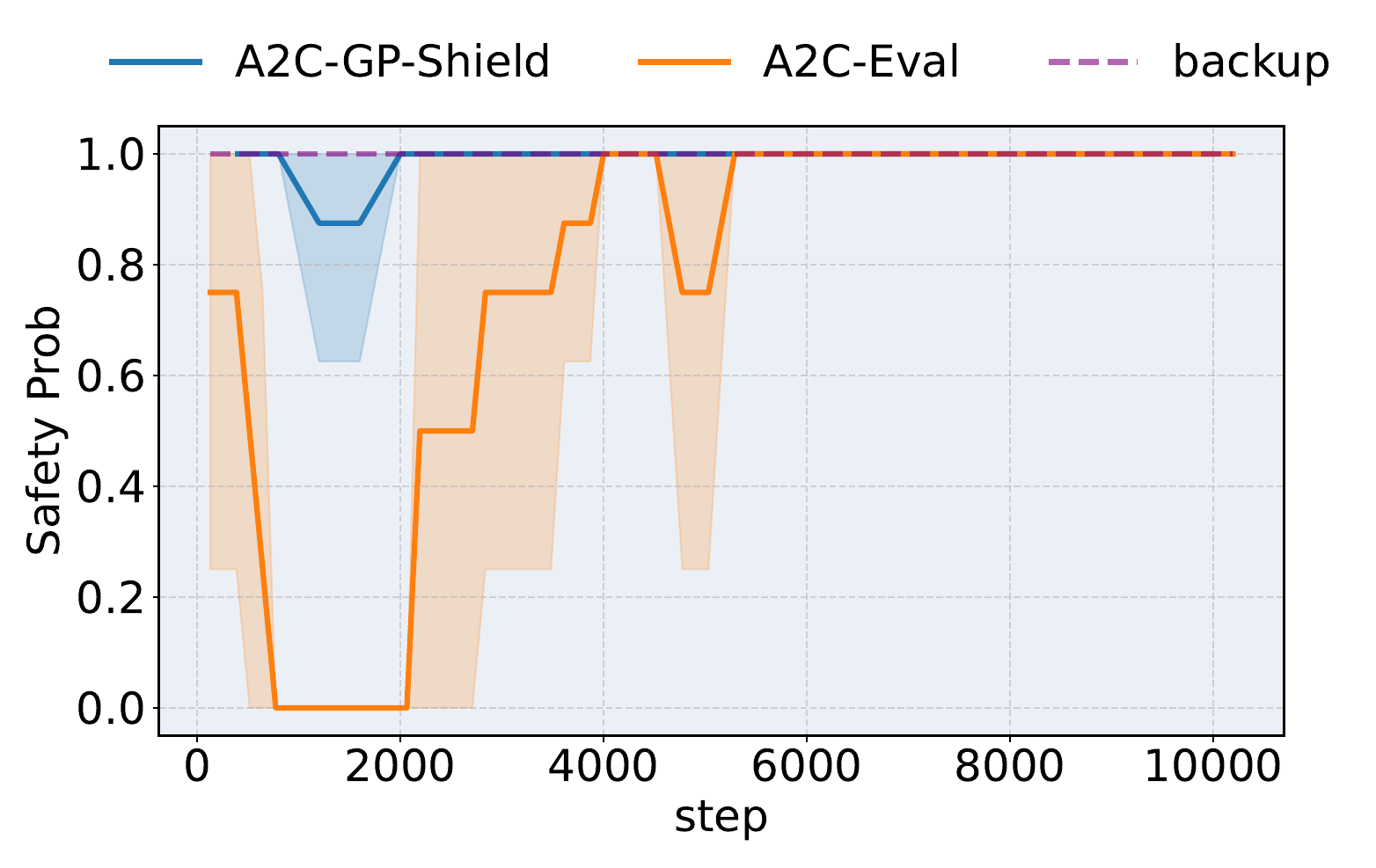}
        \caption{\texttt{obstacle2}}
        
    \end{subfigure}
    \hfill
    \begin{subfigure}[t]{0.49\linewidth}
        \centering
        \includegraphics[width=0.49\textwidth]{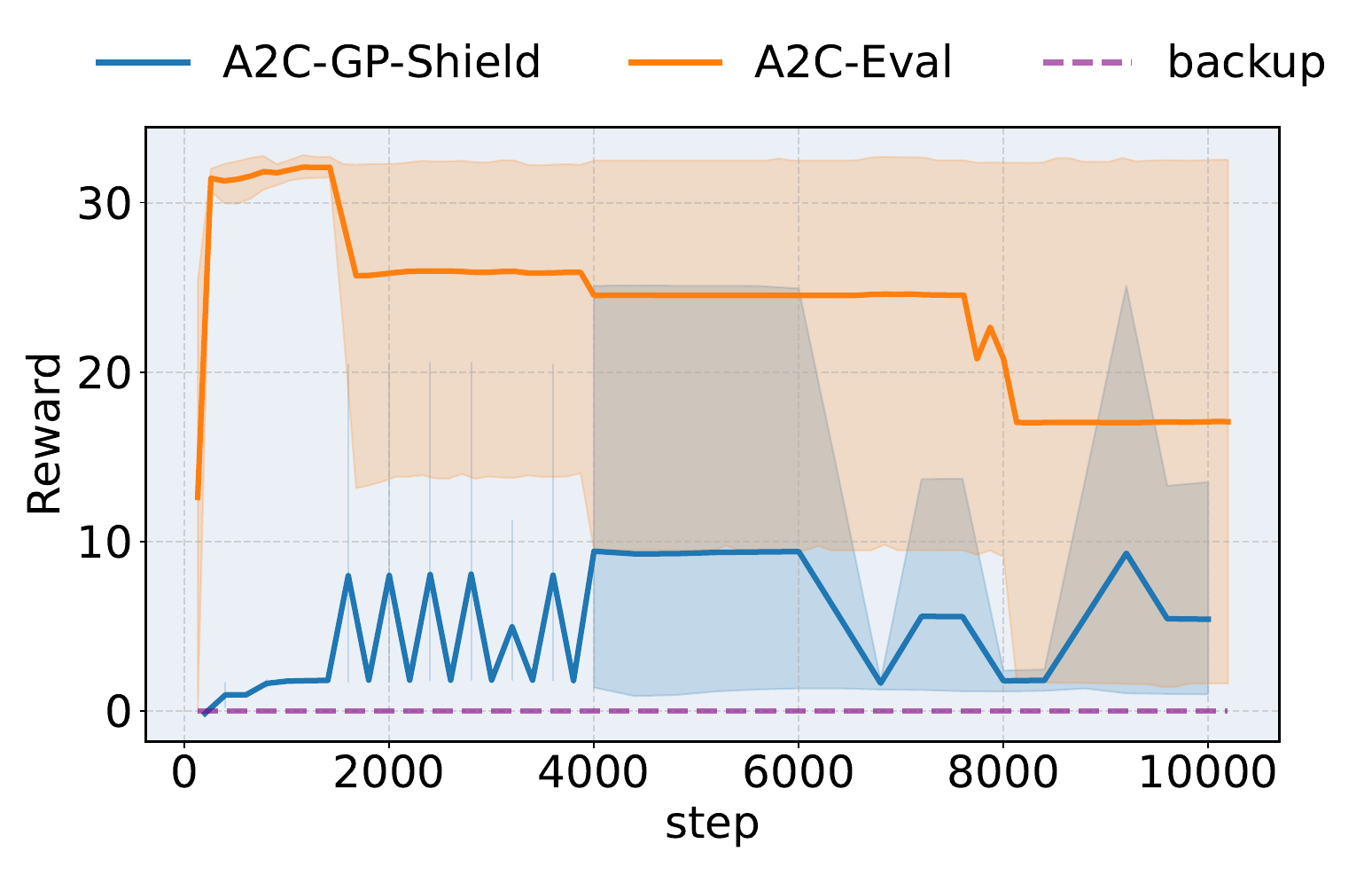}
        \includegraphics[width=0.49\textwidth]{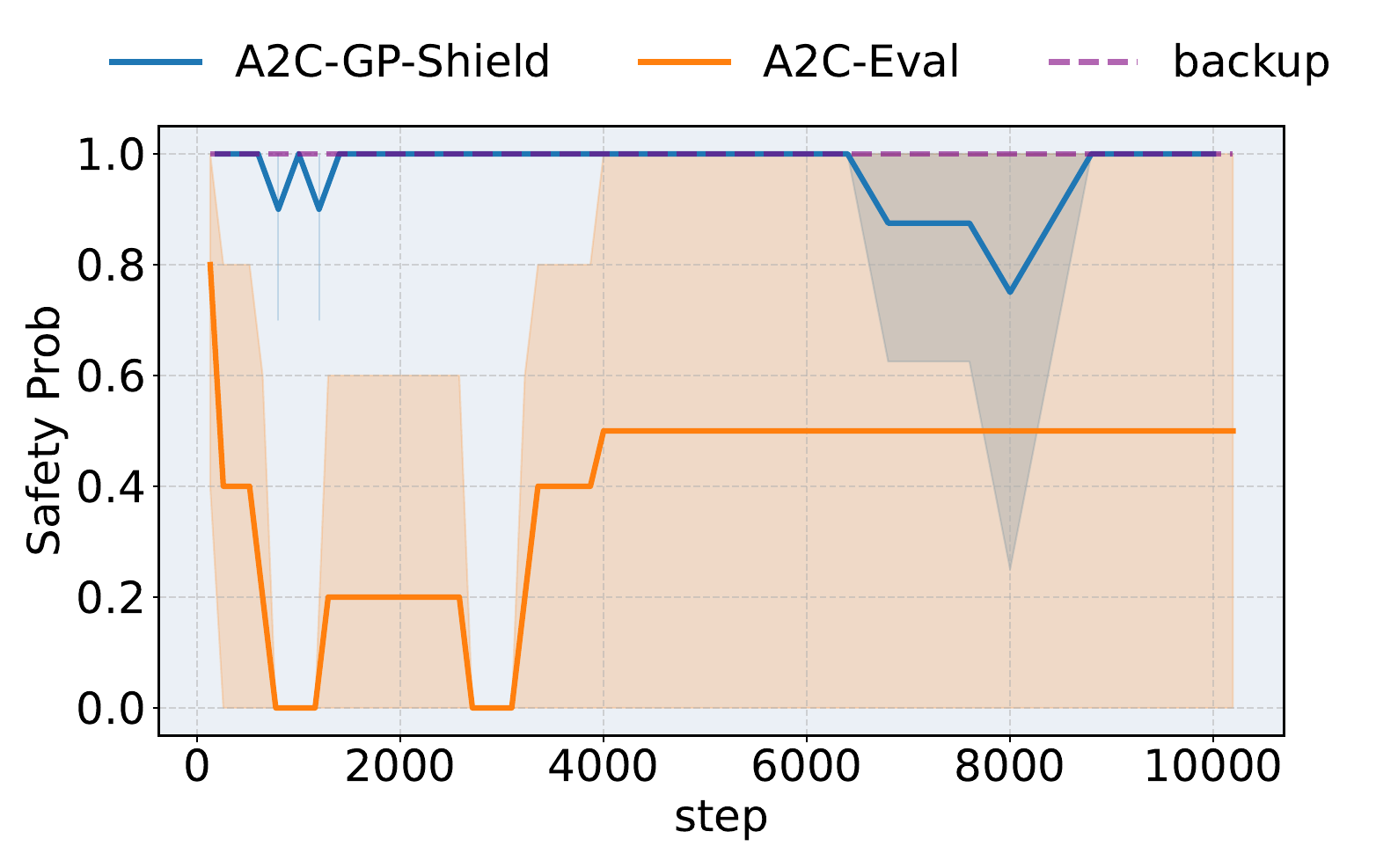}
        \caption{\texttt{obstacle3}}
        
    \end{subfigure}
    \hfill
    \begin{subfigure}[t]{0.49\linewidth}
        \centering
        \includegraphics[width=0.49\textwidth]{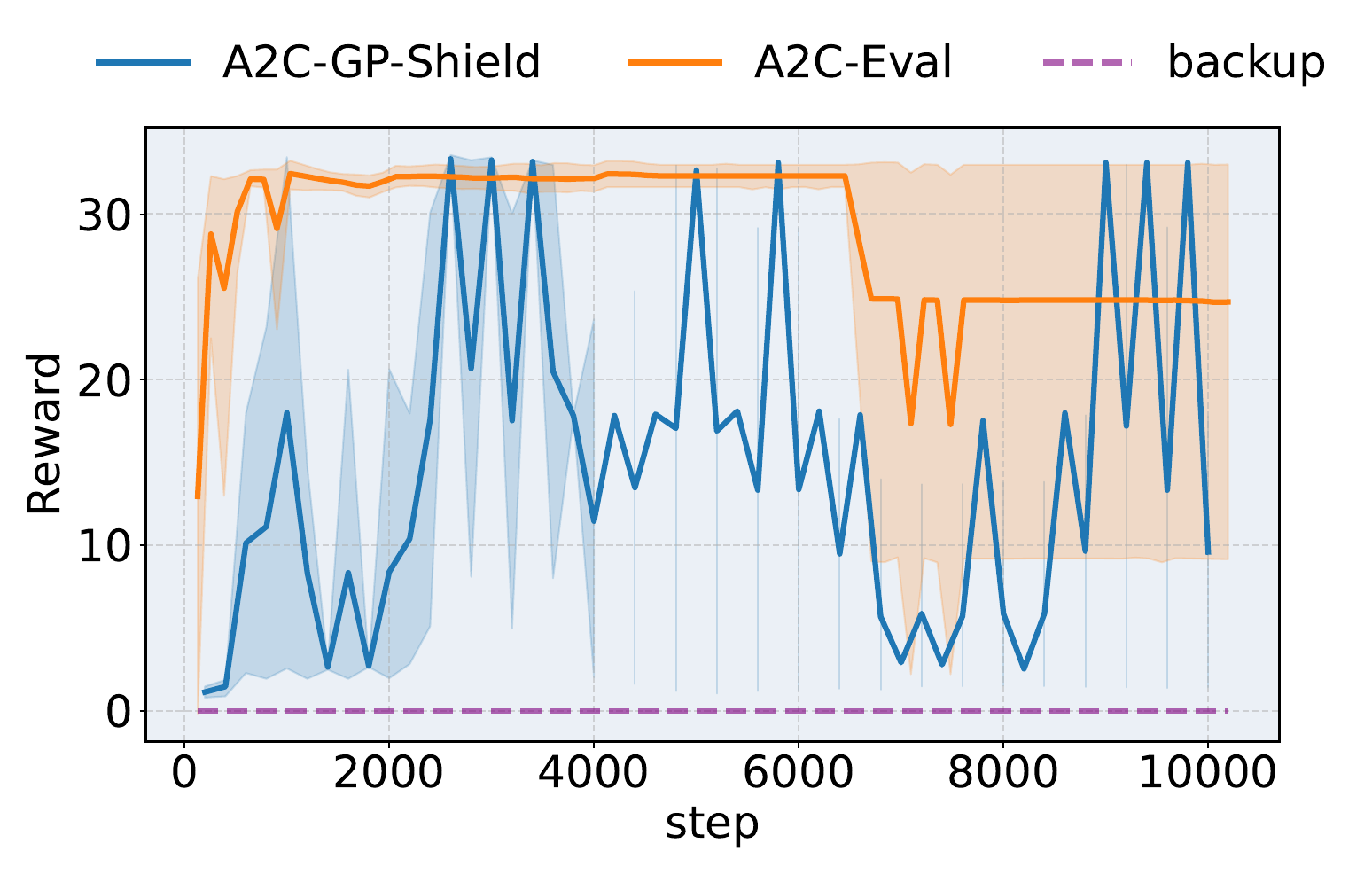}
        \includegraphics[width=0.49\textwidth]{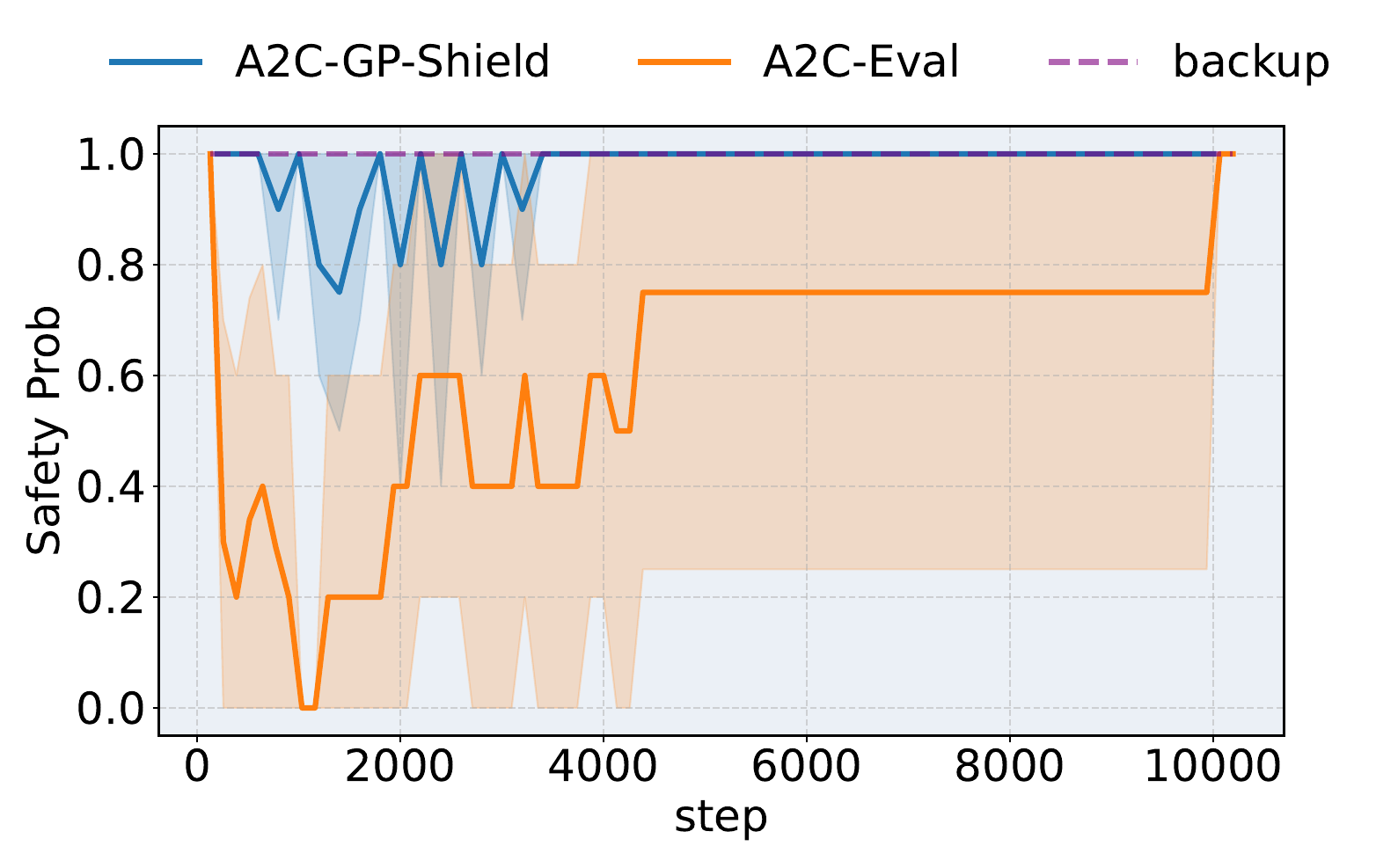}
        \caption{\texttt{obstacle4}}
        
    \end{subfigure}
    \hfill
    \begin{subfigure}[t]{0.49\linewidth}
        \centering
        \includegraphics[width=0.49\textwidth]{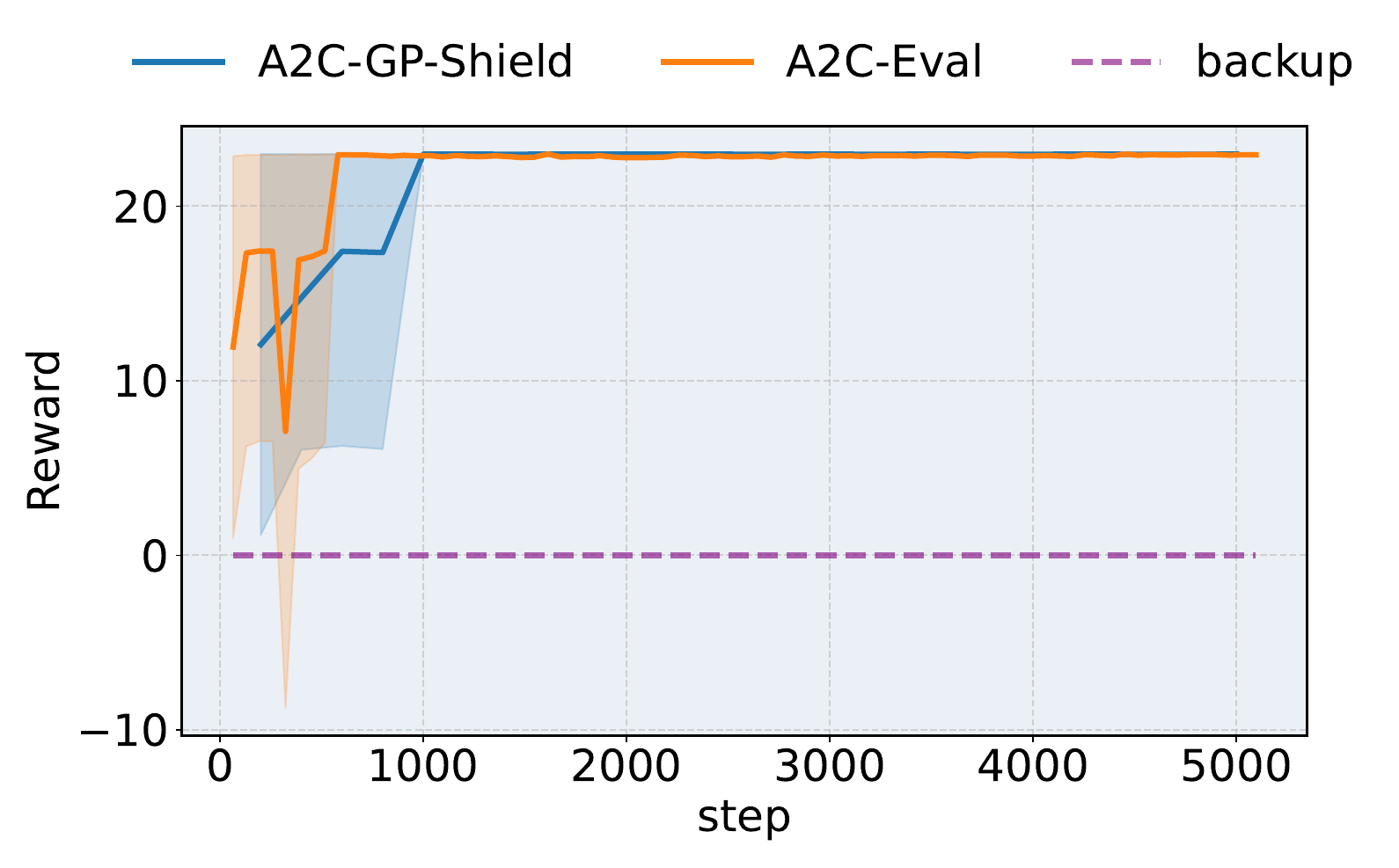}
        \includegraphics[width=0.49\textwidth]{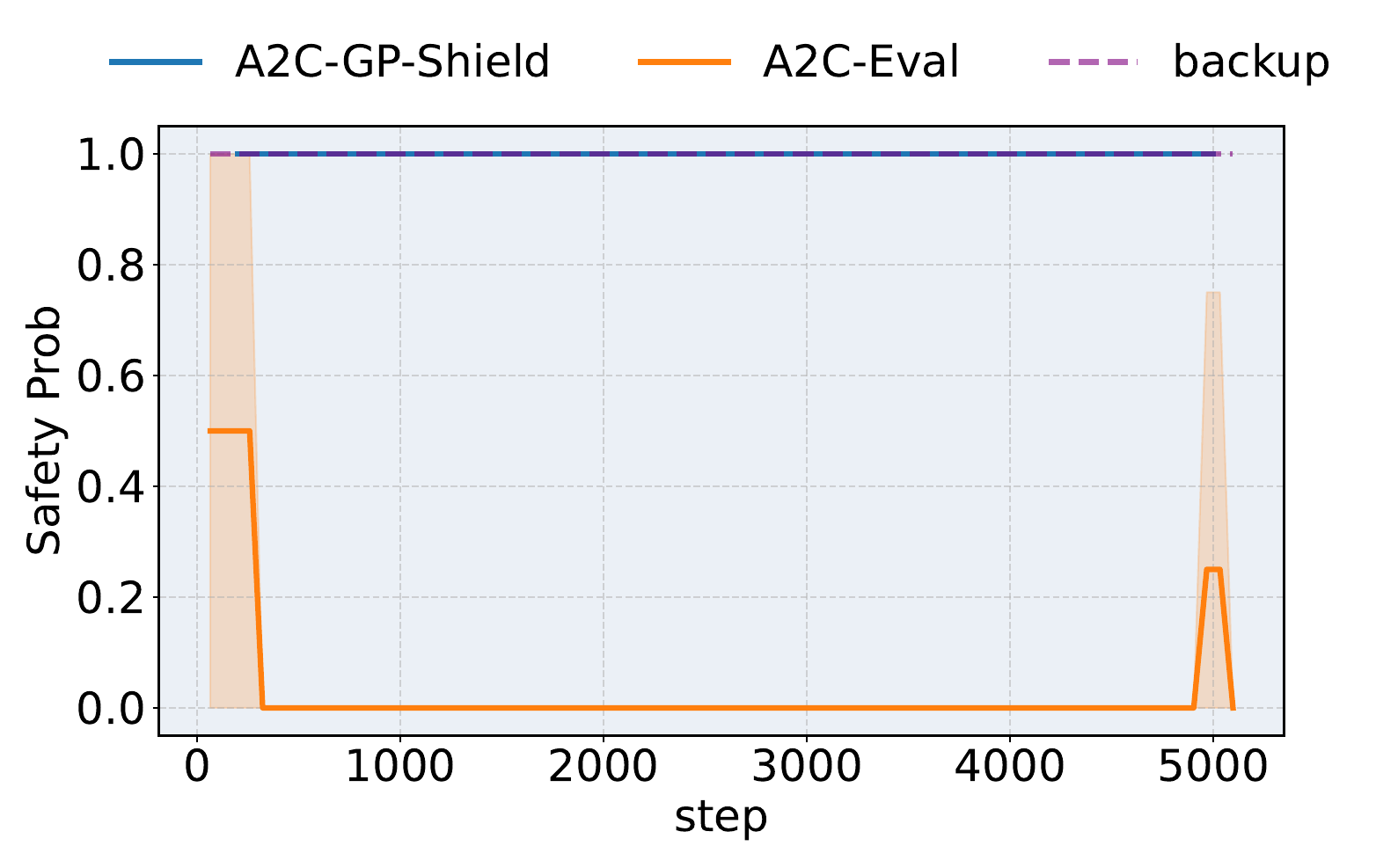}
        \caption{\texttt{road}}
    \end{subfigure}
    \hfill
    \begin{subfigure}[t]{0.49\linewidth}
        \centering
        \includegraphics[width=0.49\textwidth]{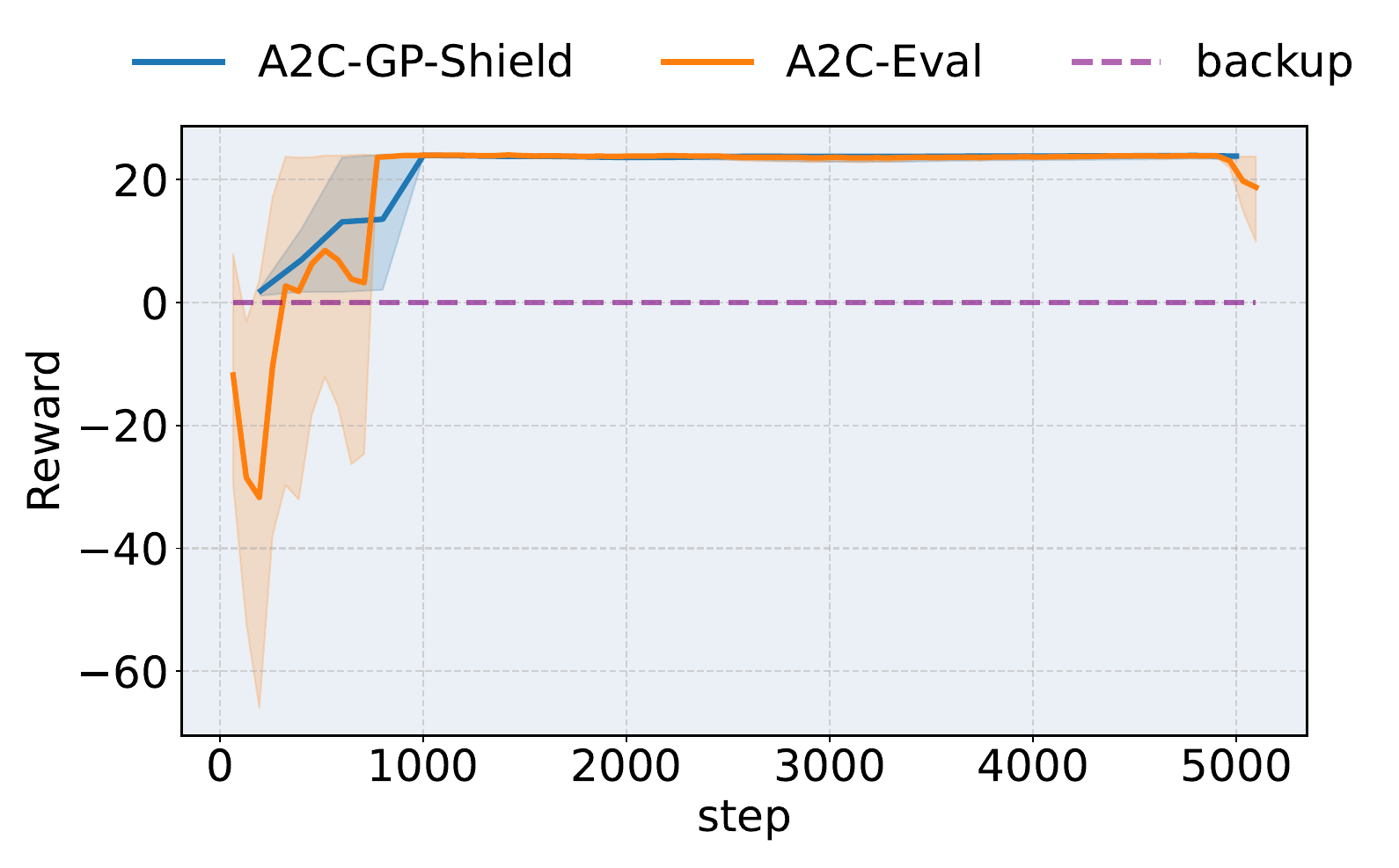}
        \includegraphics[width=0.49\textwidth]{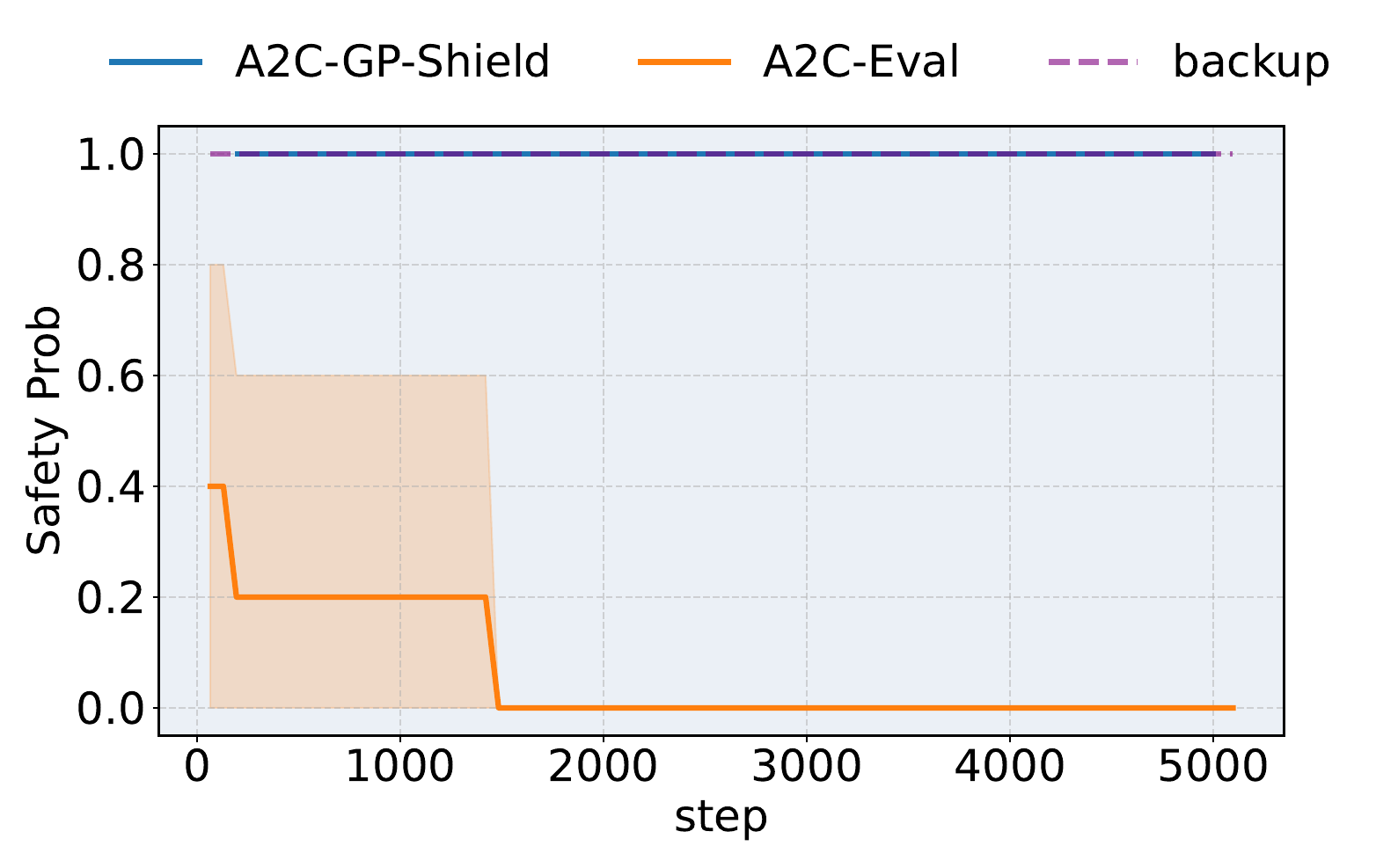}
        \caption{\texttt{road\_2d}}
        
    \end{subfigure}
    \caption{Learning curves (reward and safety probability) for A2C-Shield and A2C-Eval.}
\end{figure*}

\begin{figure*}[ht!]
    \centering
    \hfill
    \begin{subfigure}[t]{0.24\linewidth}
        \centering
        \includegraphics[width=\textwidth]{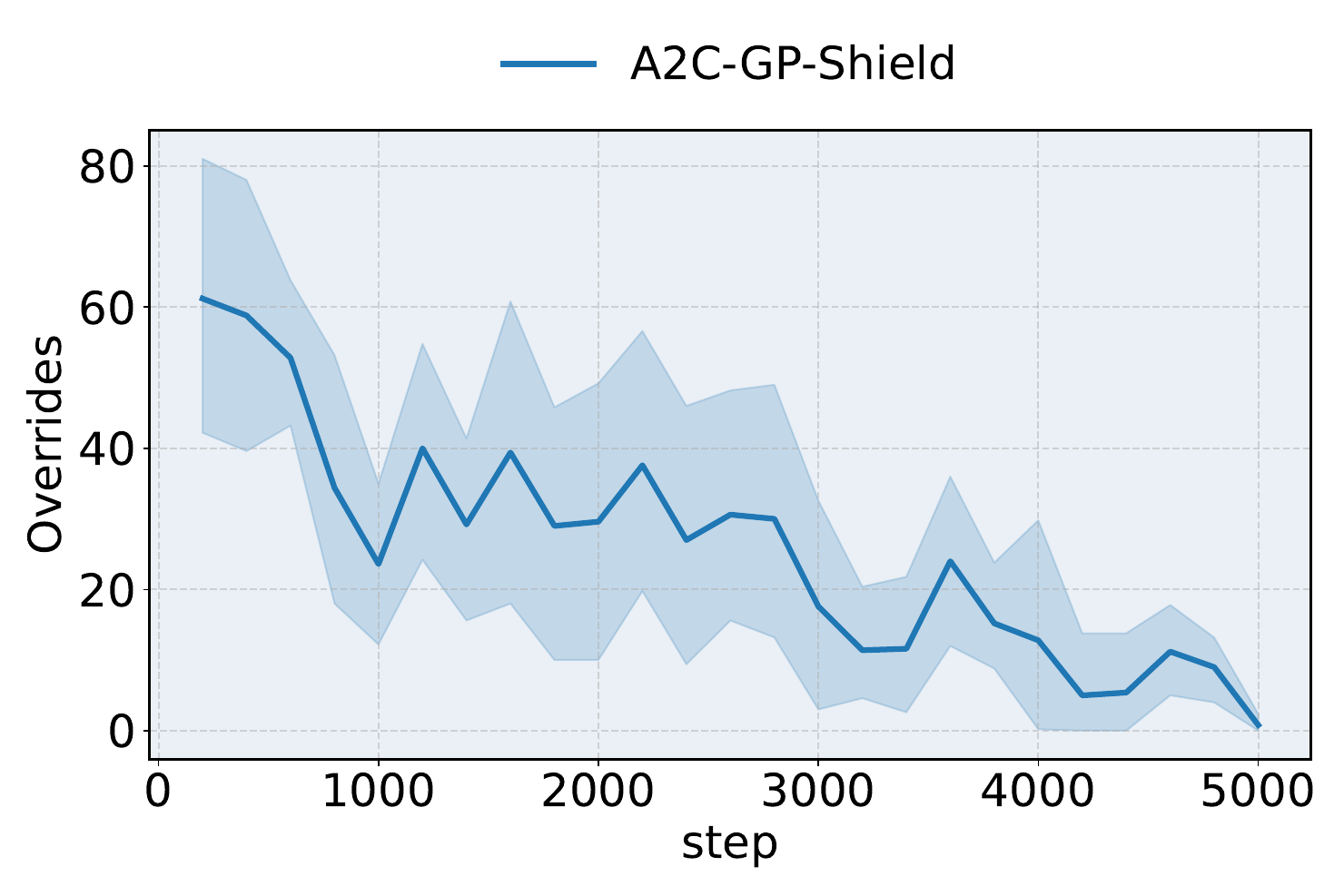}
        \caption{\texttt{cartpole (i)}}
    \end{subfigure}
    \hfill
    \begin{subfigure}[t]{0.24\linewidth}
        \centering
        \includegraphics[width=\textwidth]{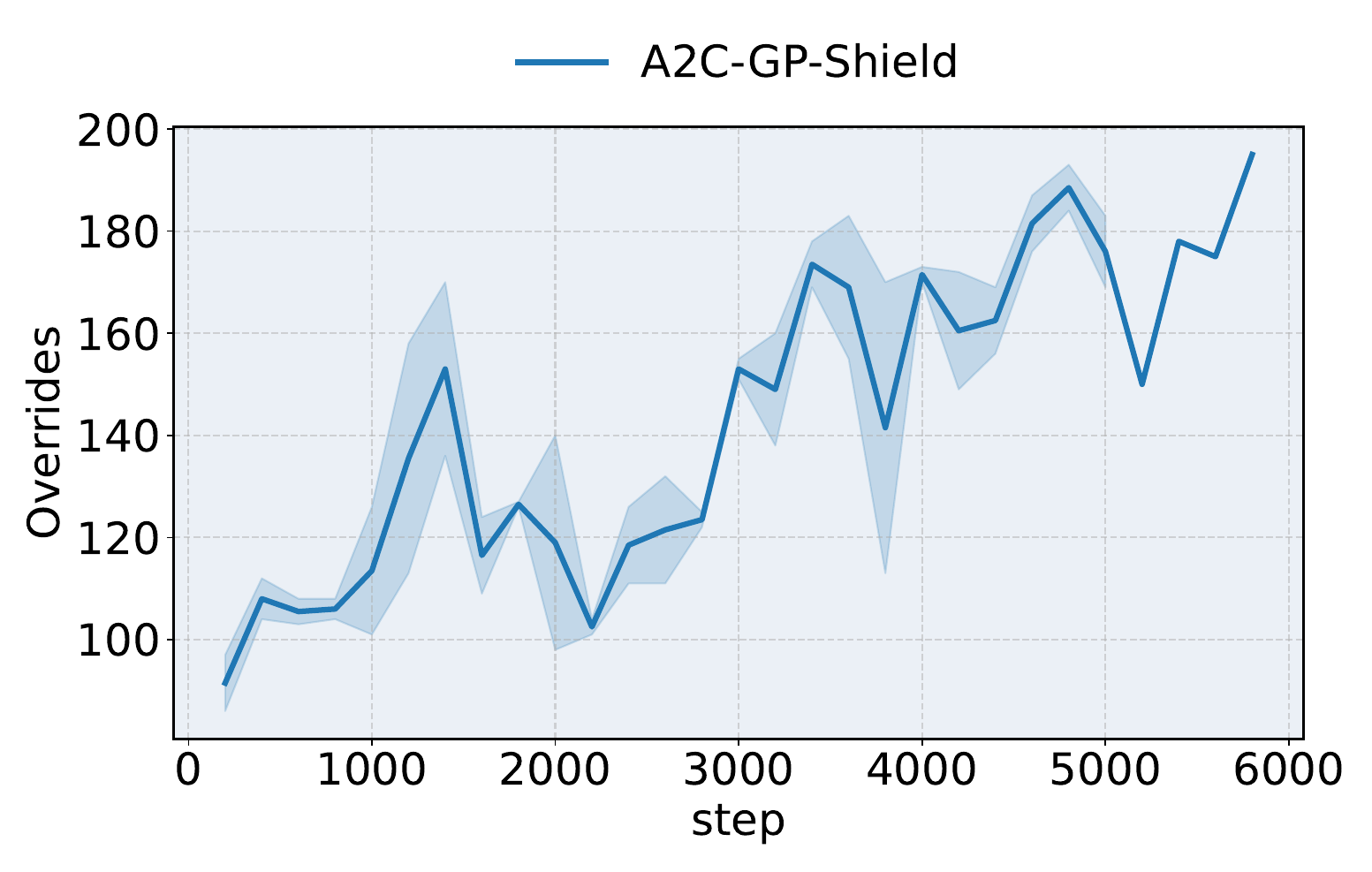}
        \caption{\texttt{cartpole (ii)}}
    \end{subfigure}
    \hfill
    \begin{subfigure}[t]{0.24\linewidth}
        \centering
        \includegraphics[width=\textwidth]{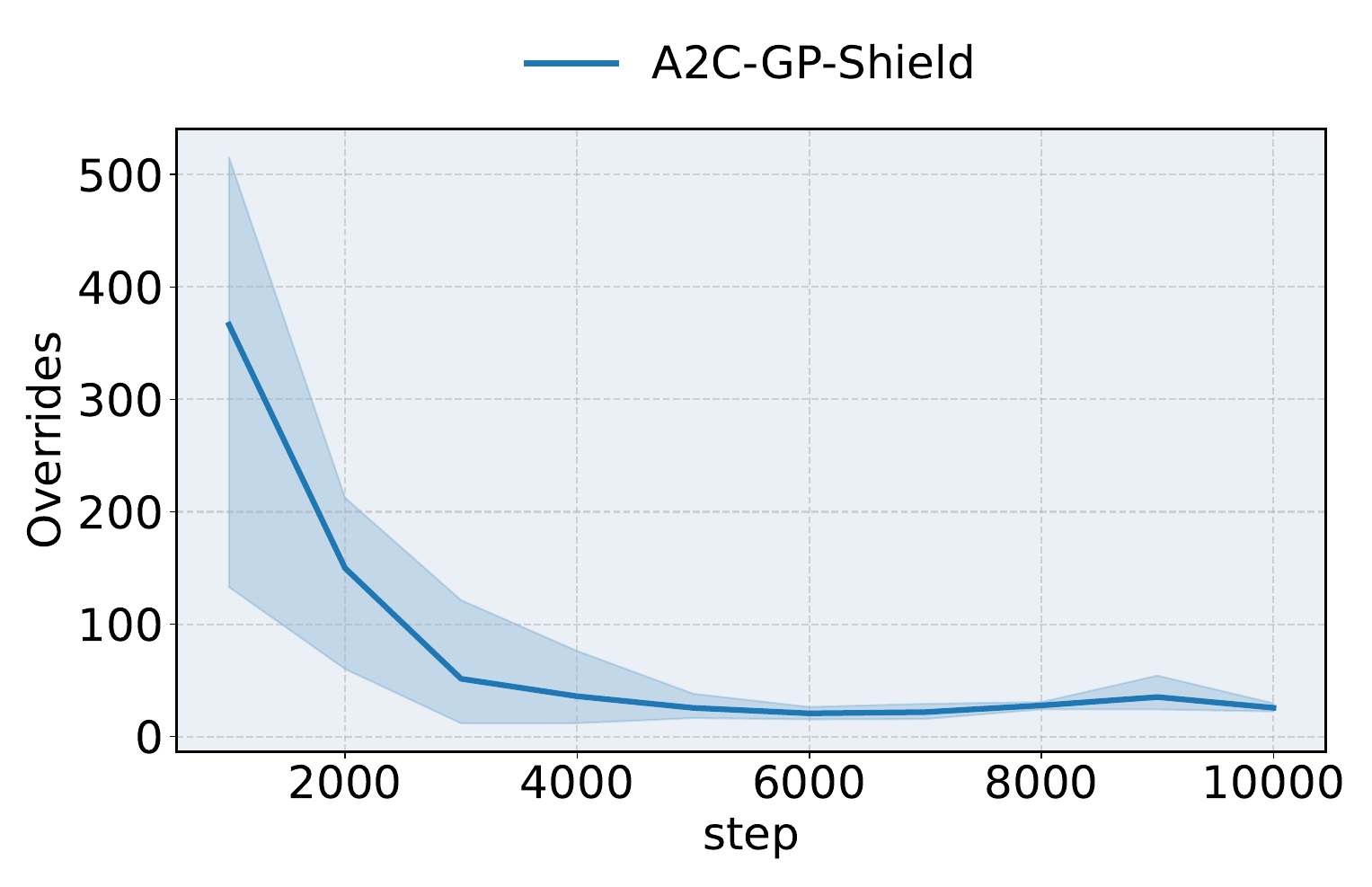}
        \caption{\texttt{mountain\_car}}
    \end{subfigure}
    \hfill
    \begin{subfigure}[t]{0.24\linewidth}
        \centering
        \includegraphics[width=\textwidth]{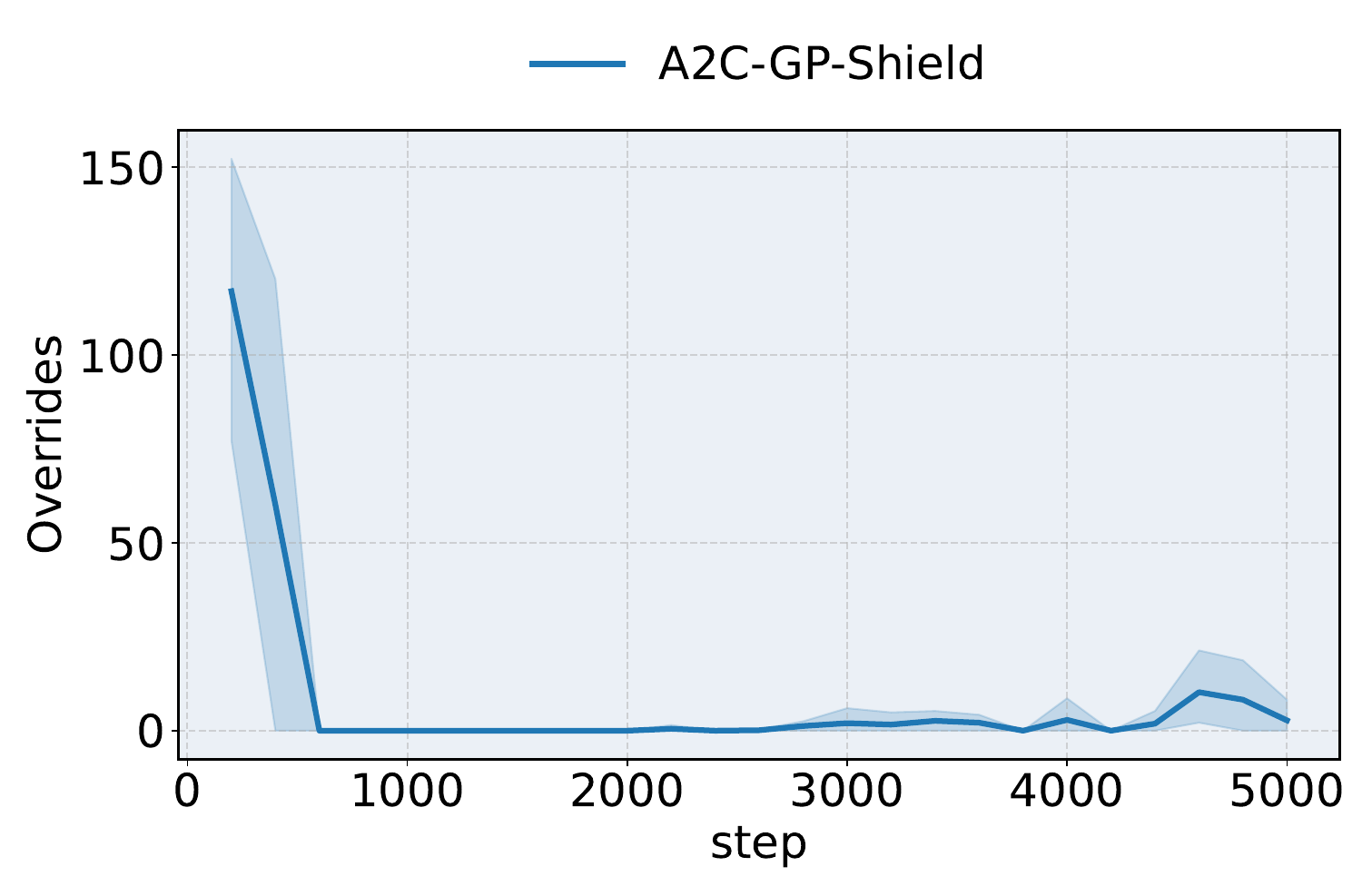}
        \caption{\texttt{obstacle}}
    \end{subfigure}
    \hfill
    \begin{subfigure}[t]{0.24\linewidth}
        \centering
        \includegraphics[width=\textwidth]{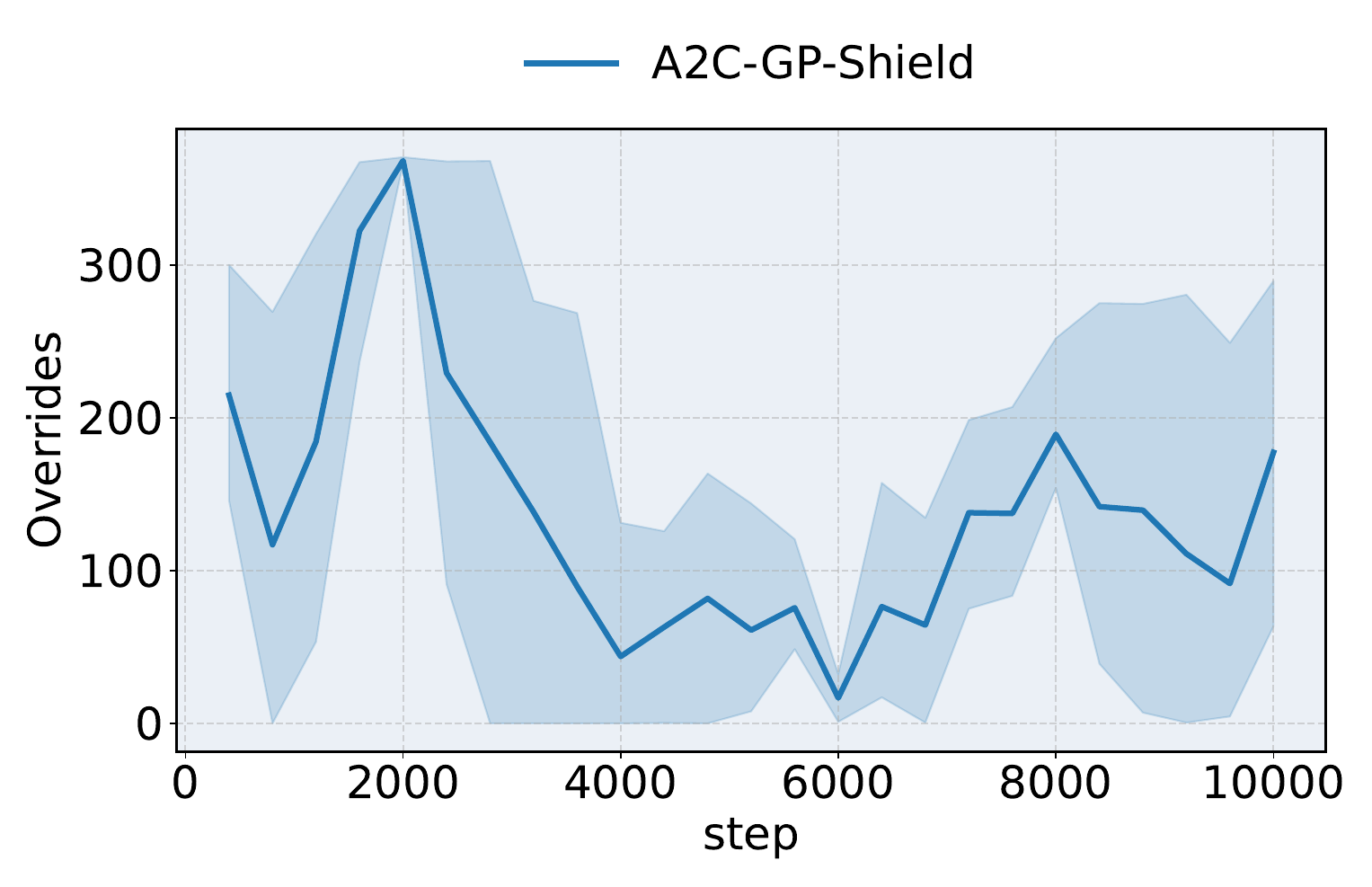}
        \caption{\texttt{obstacle2}}
    \end{subfigure}
    \hfill
    \begin{subfigure}[t]{0.24\linewidth}
        \centering
        \includegraphics[width=\textwidth]{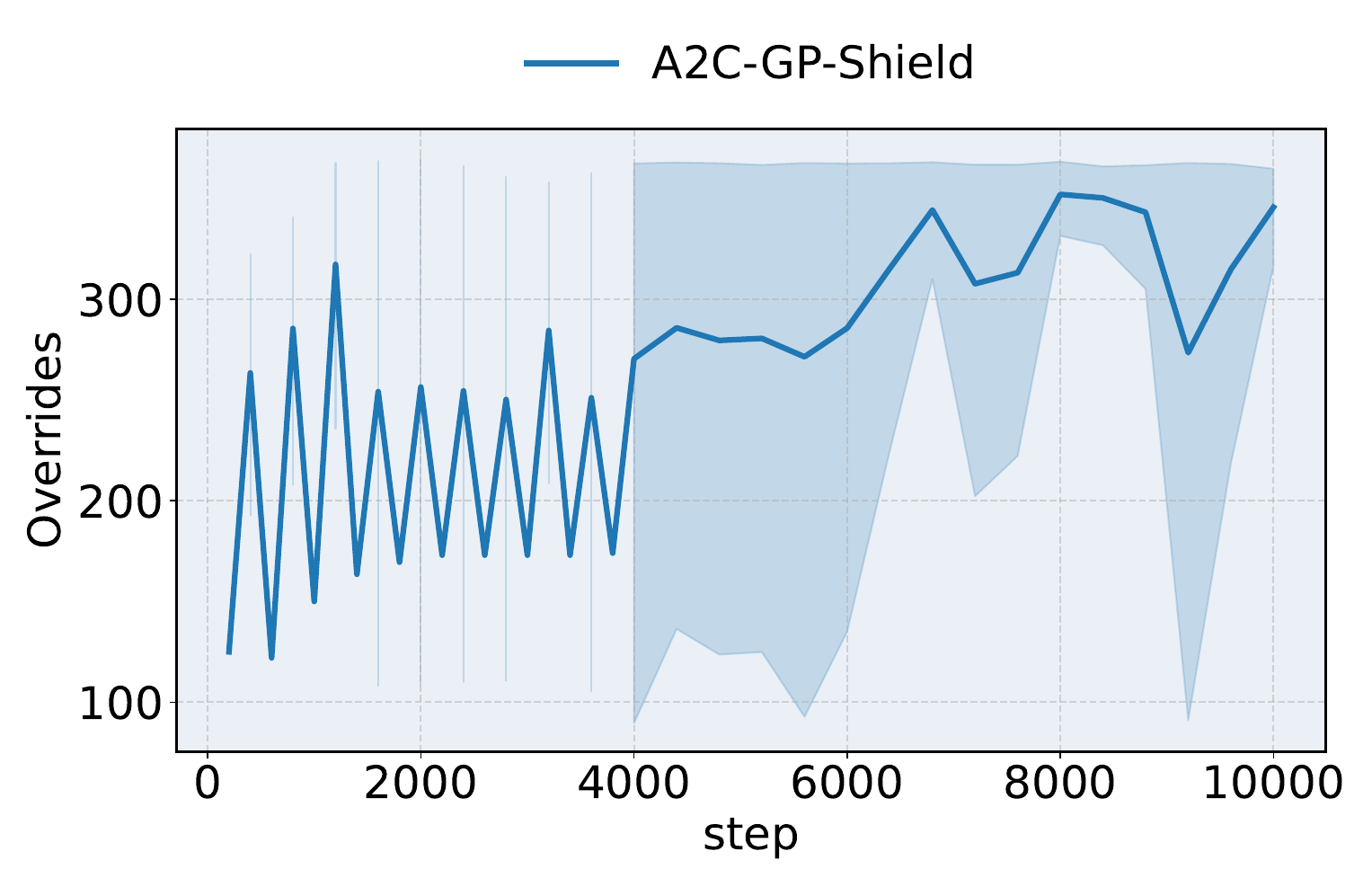}
        \caption{\texttt{obstacle3}}
    \end{subfigure}
    \hfill
    \begin{subfigure}[t]{0.24\linewidth}
        \centering
        \includegraphics[width=\textwidth]{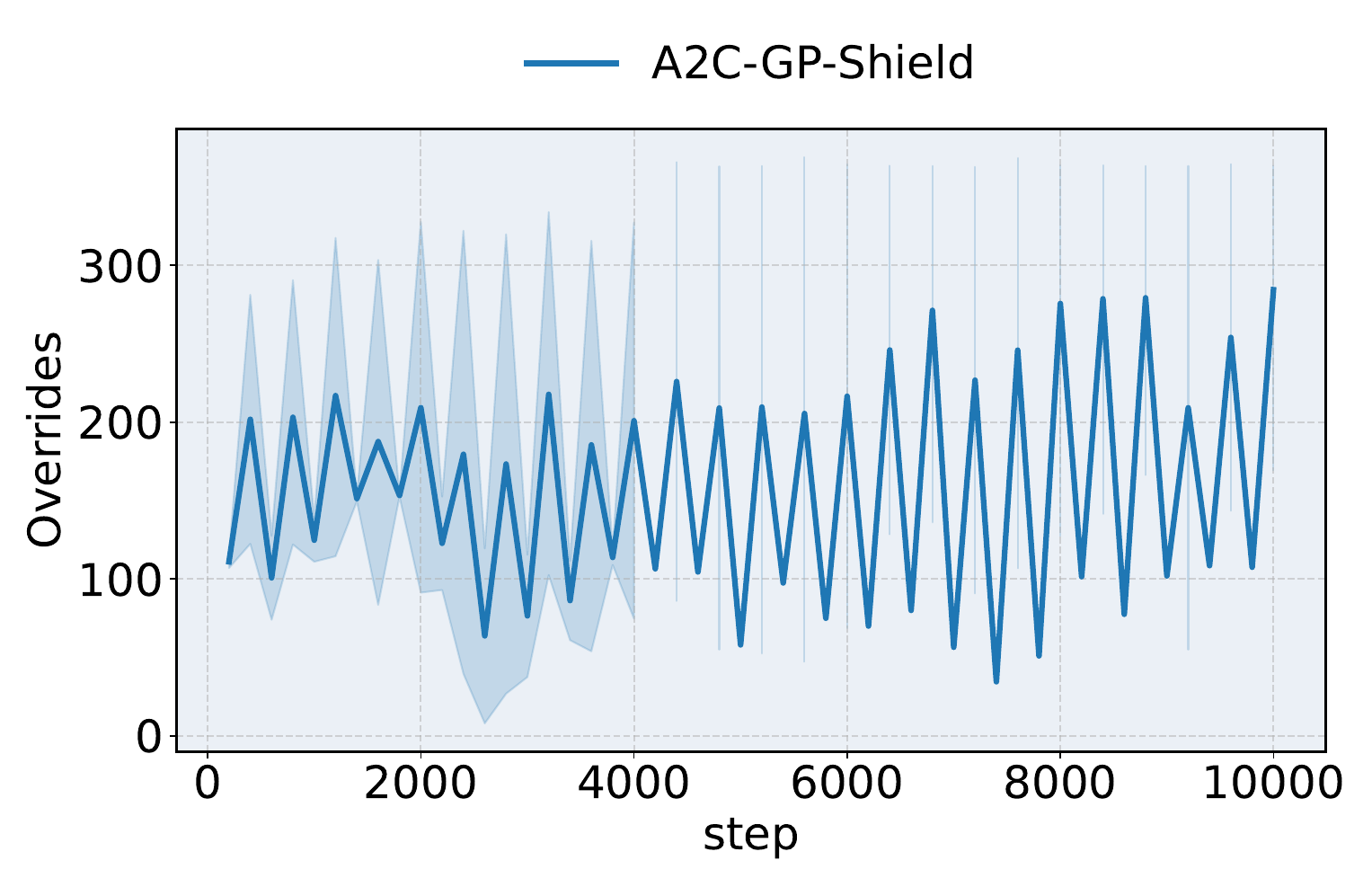}
        \caption{\texttt{obstacle4}}
    \end{subfigure}
    \hfill
    \begin{subfigure}[t]{0.24\linewidth}
        \centering
        \includegraphics[width=\textwidth]{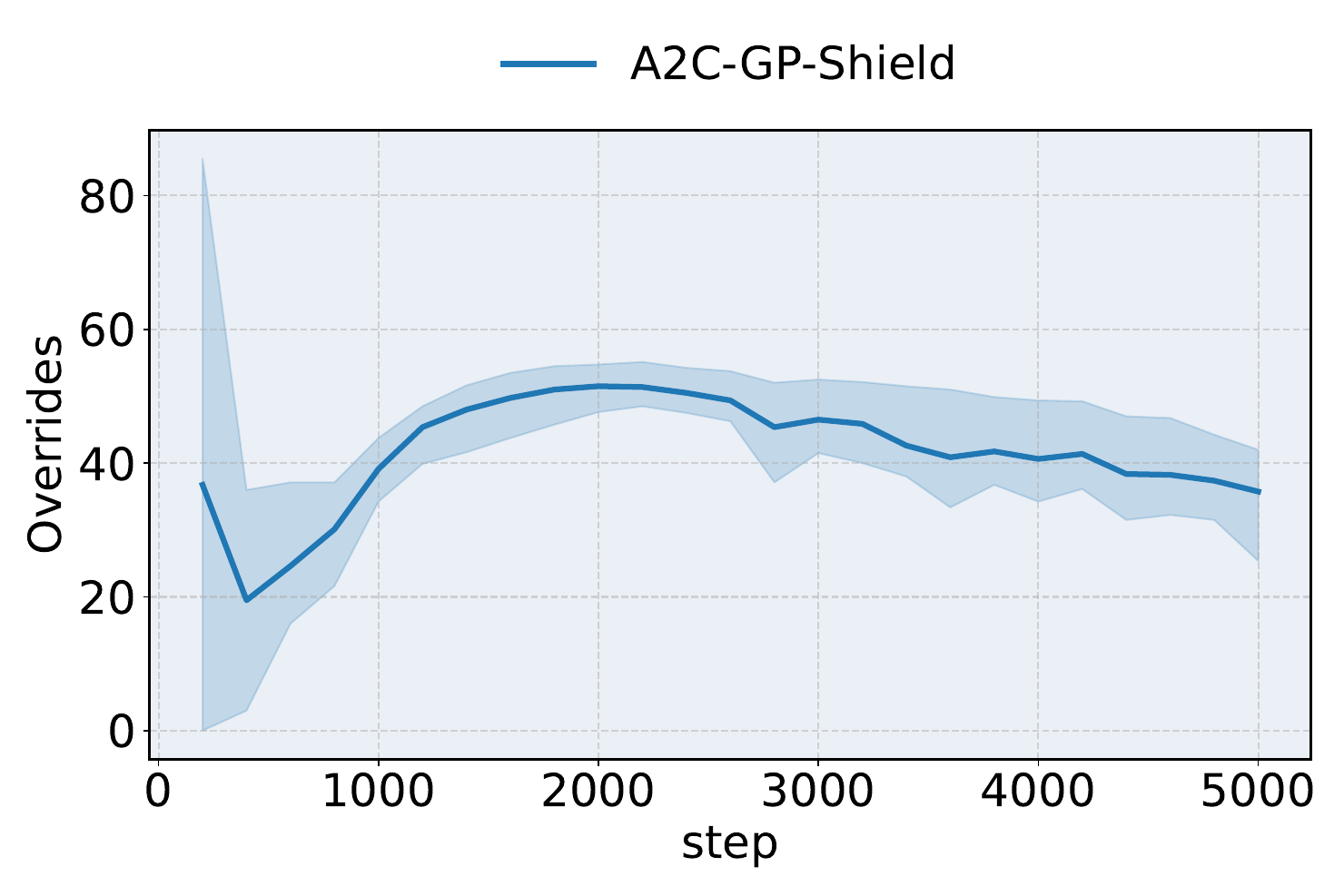}
        \caption{\texttt{road}}
    \end{subfigure}
    \hfill
    \begin{subfigure}[t]{0.24\linewidth}
        \centering
        \includegraphics[width=\textwidth]{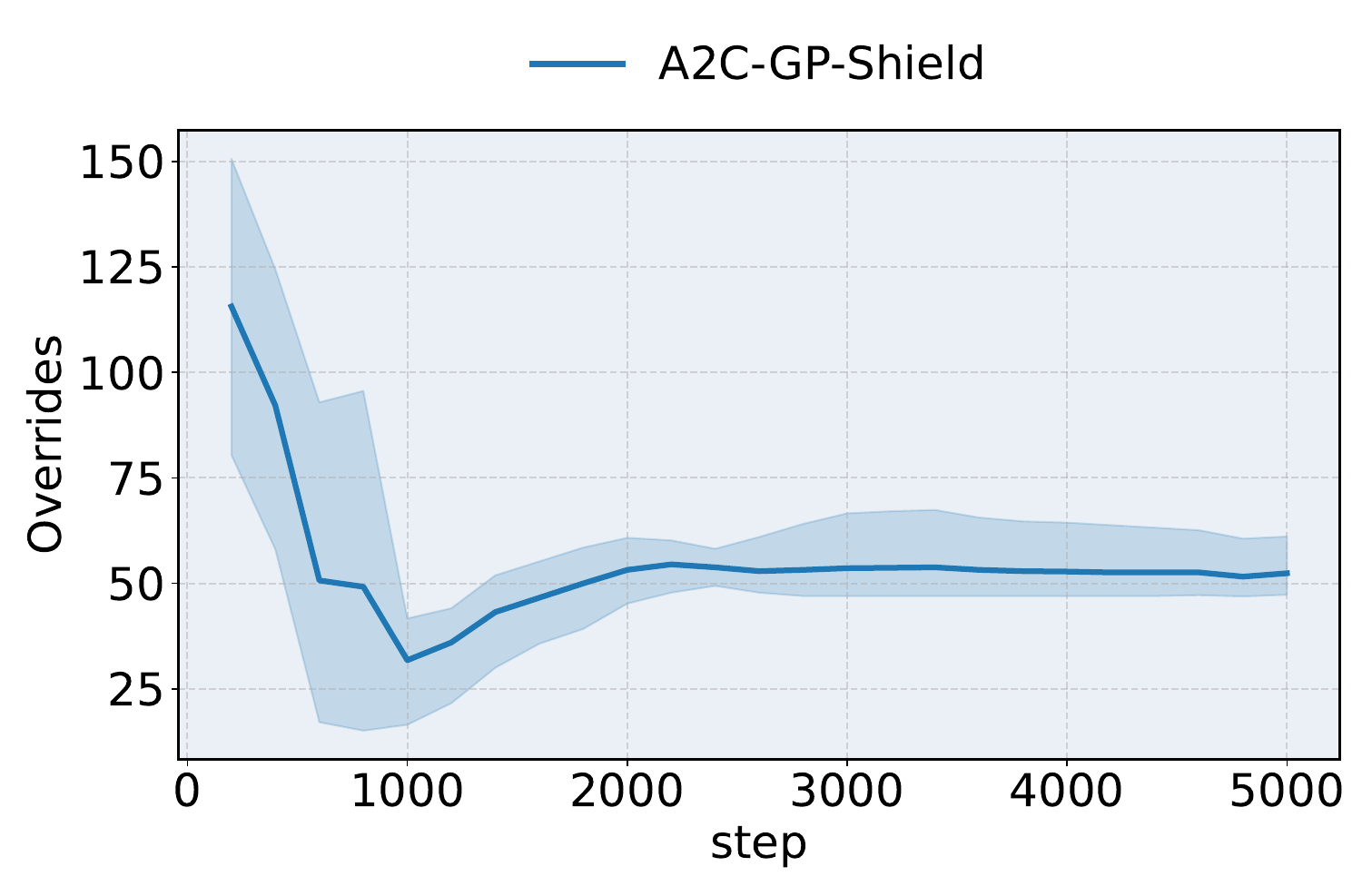}
        \caption{\texttt{road\_2d}}
    \end{subfigure}
    \hfill
    \caption{Overrides (per episode) for A2C-Shield}
\end{figure*}

\newpage
\begin{figure*}[ht!]
    \centering
    \hfill
    \begin{subfigure}[t]{0.49\linewidth}
        \centering
        \includegraphics[width=0.49\textwidth]{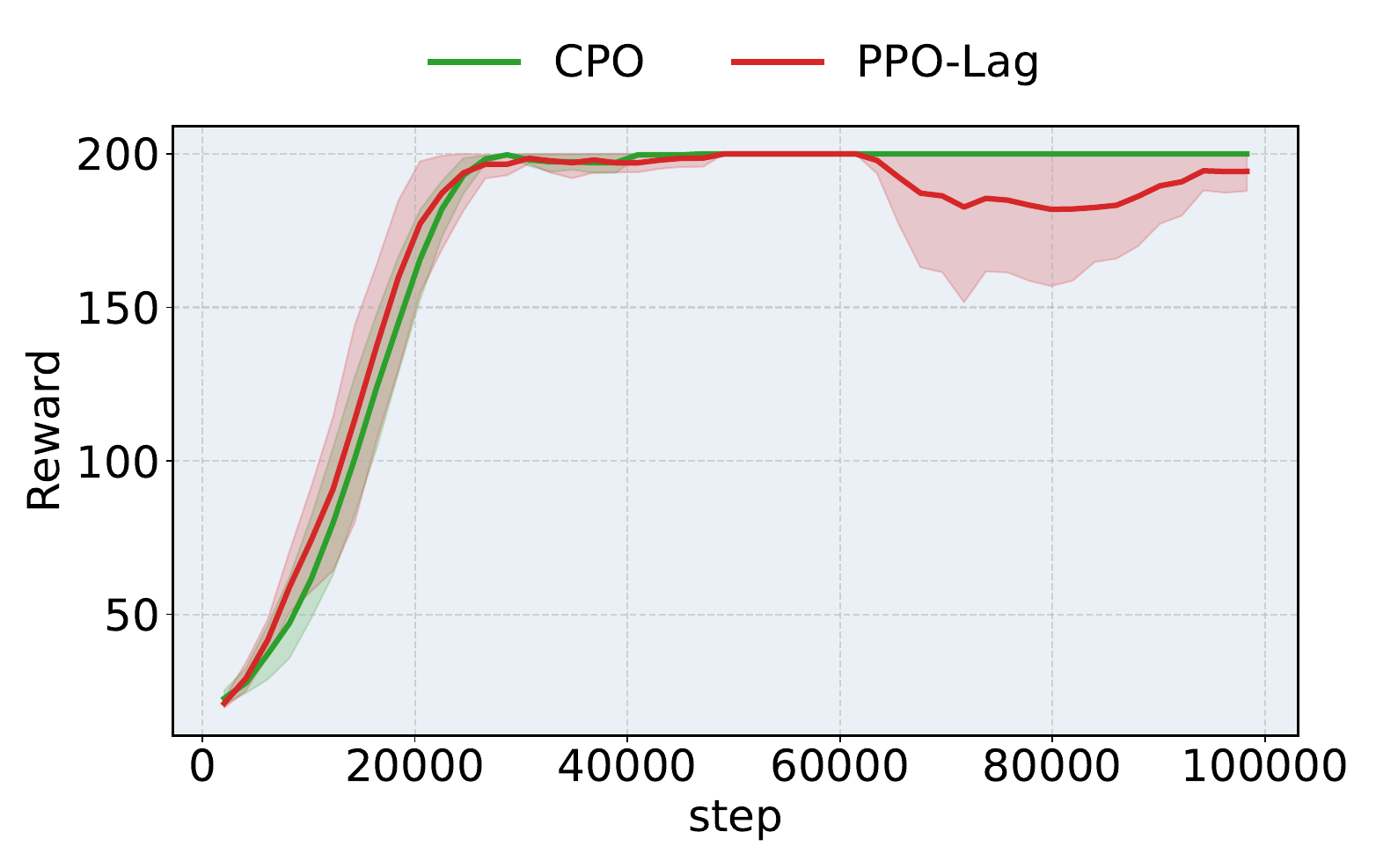}
        \includegraphics[width=0.49\textwidth]{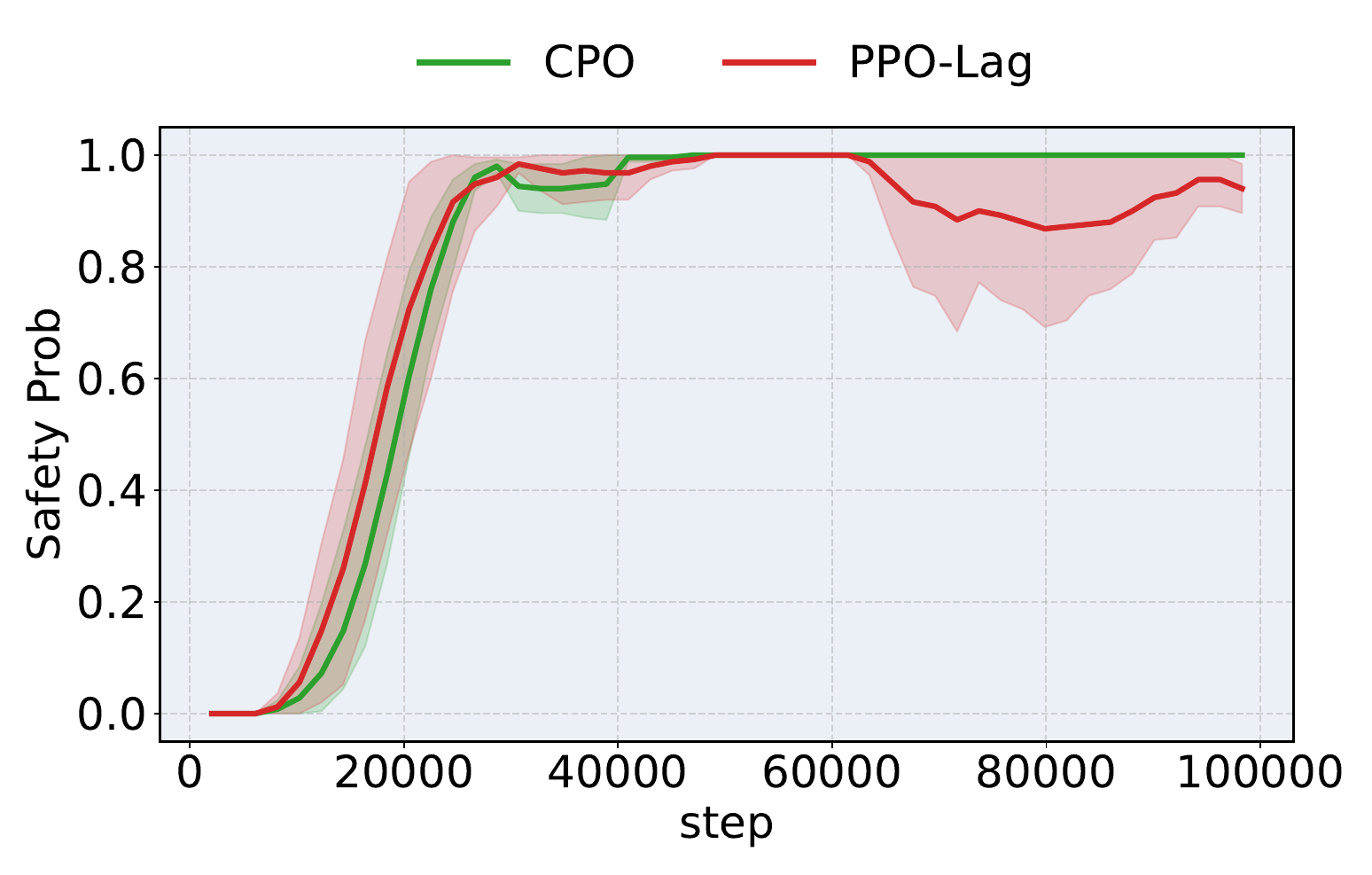}
        \caption{\texttt{cartpole (i)}}
        
    \end{subfigure}
    \hfill
    \begin{subfigure}[t]{0.49\linewidth}
        \centering
        \includegraphics[width=0.49\textwidth]{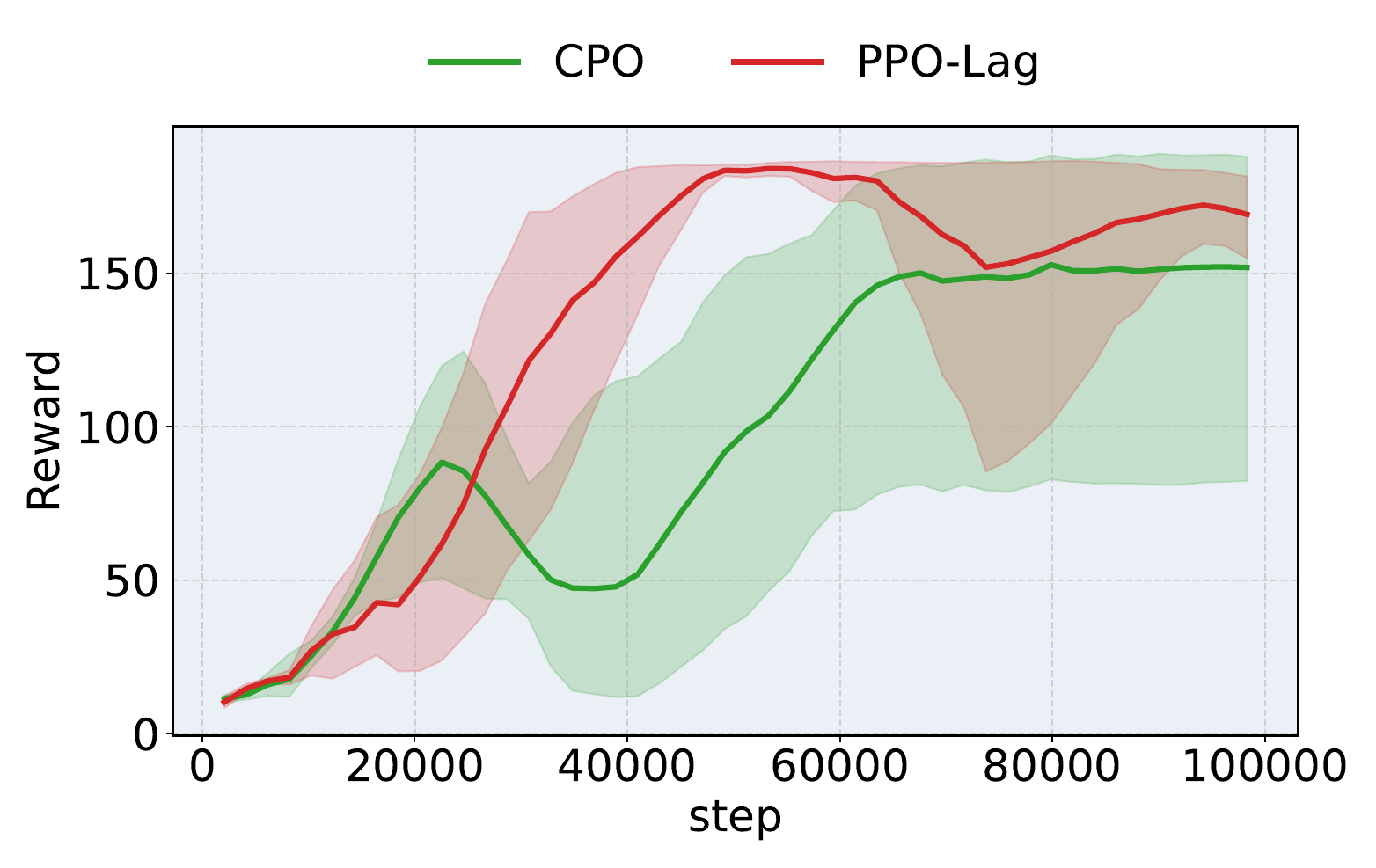}
        \includegraphics[width=0.49\textwidth]{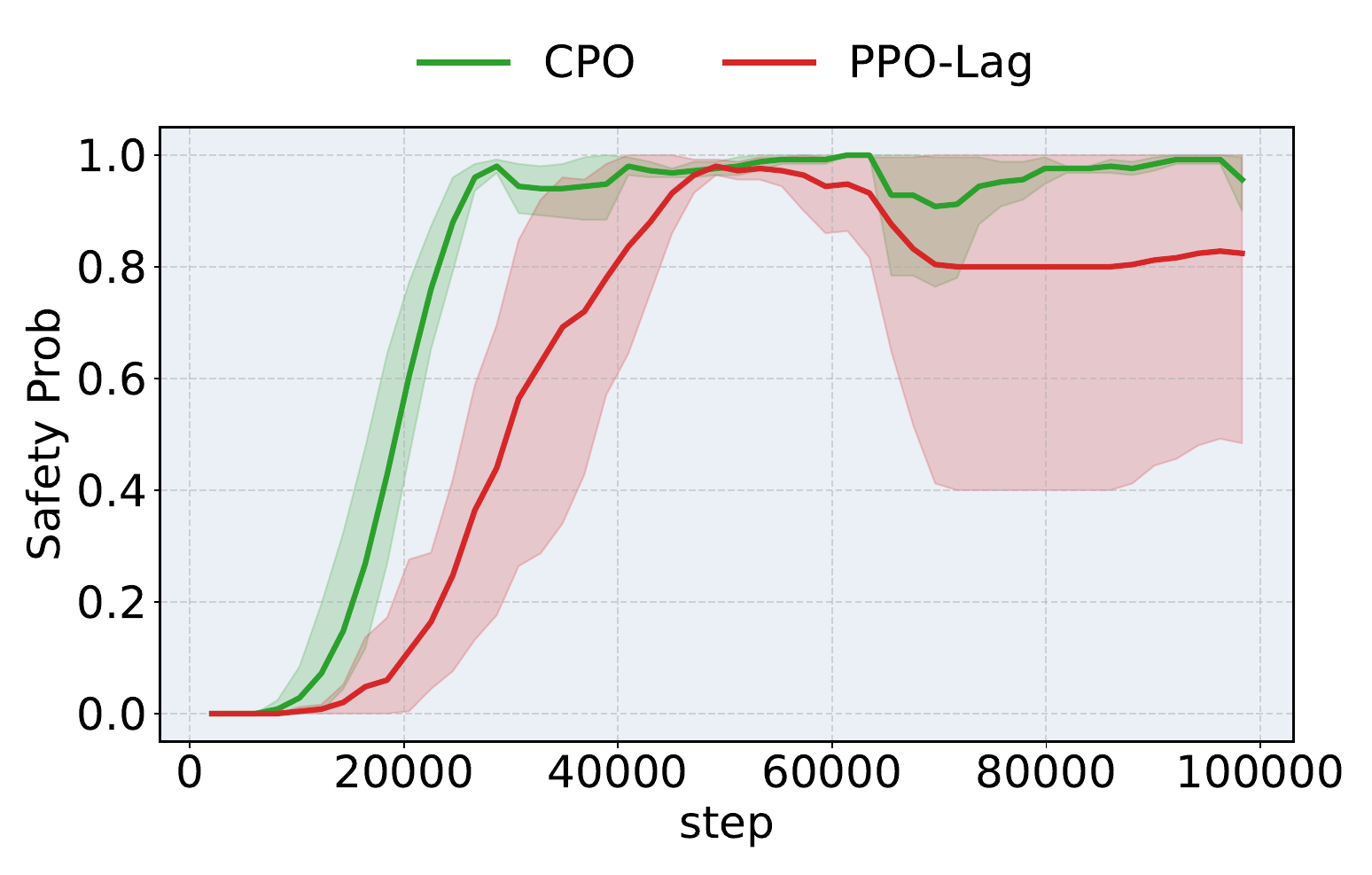}
        \caption{\texttt{cartpole (ii)}}
        
    \end{subfigure}
    \hfill
    \begin{subfigure}[t]{0.49\linewidth}
        \centering
        \includegraphics[width=0.49\textwidth]{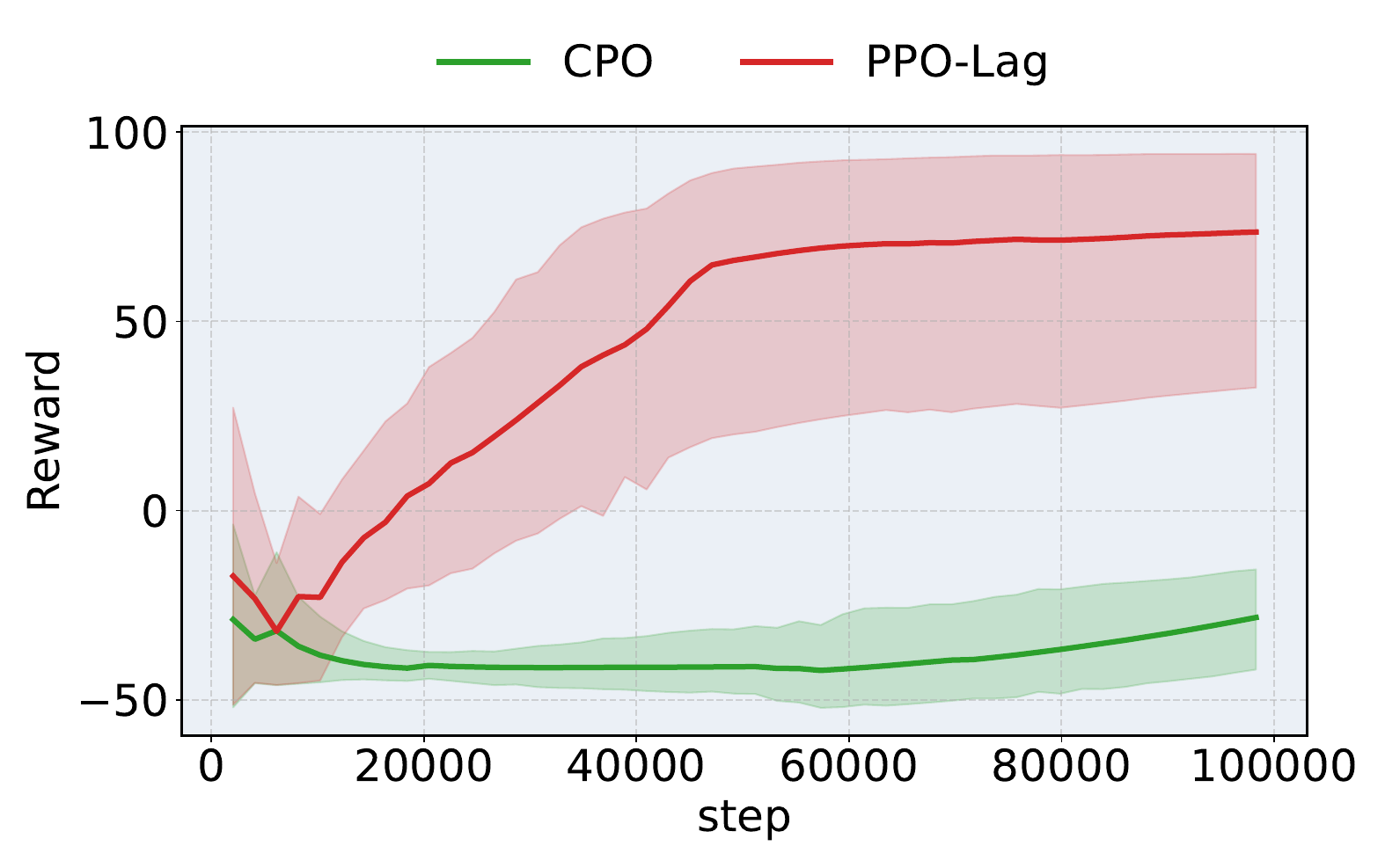}
        \includegraphics[width=0.49\textwidth]{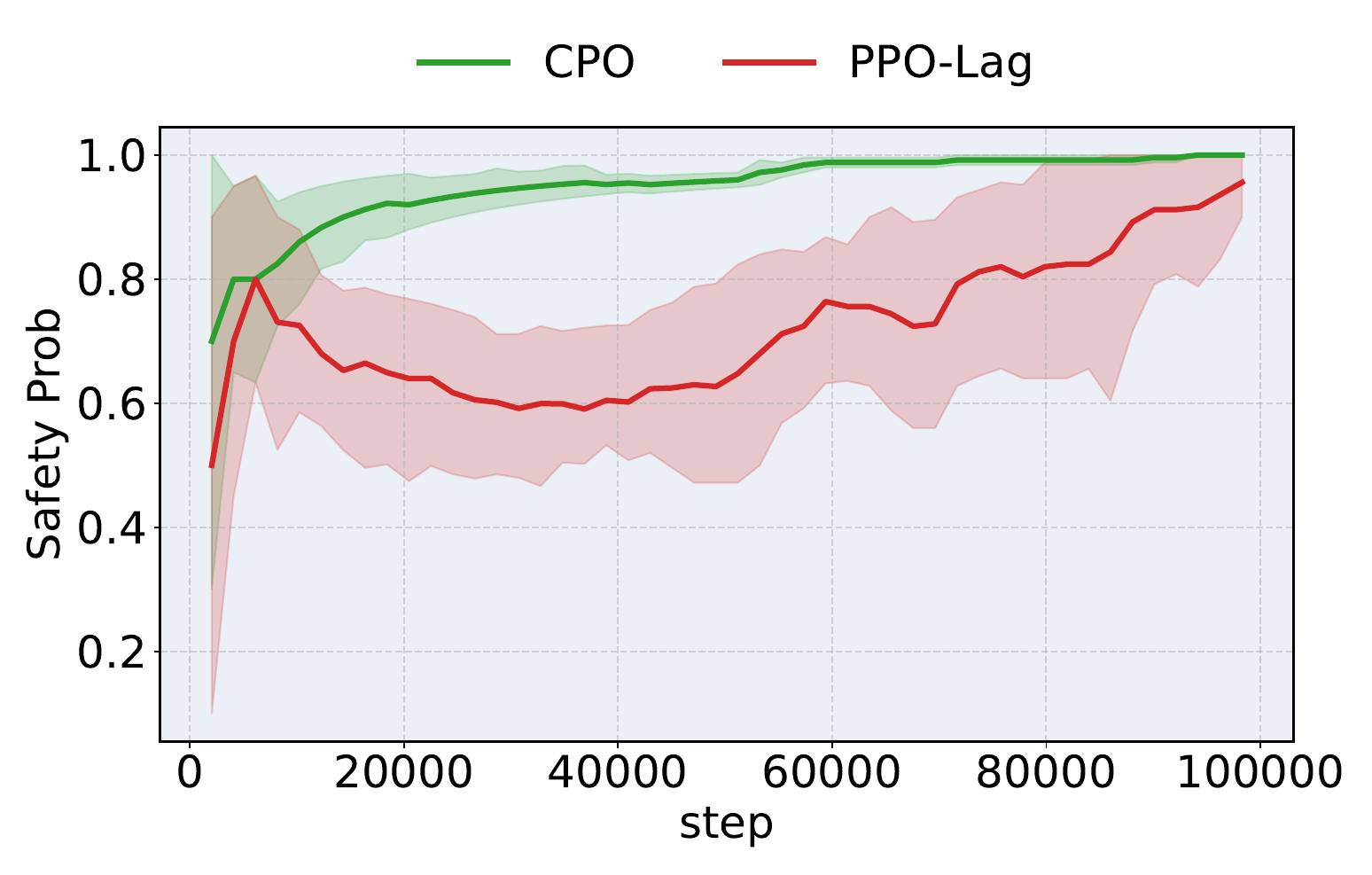}
        \caption{\texttt{mountain\_car}}
        
    \end{subfigure}
    \hfill
    \begin{subfigure}[t]{0.49\linewidth}
        \centering
        \includegraphics[width=0.49\textwidth]{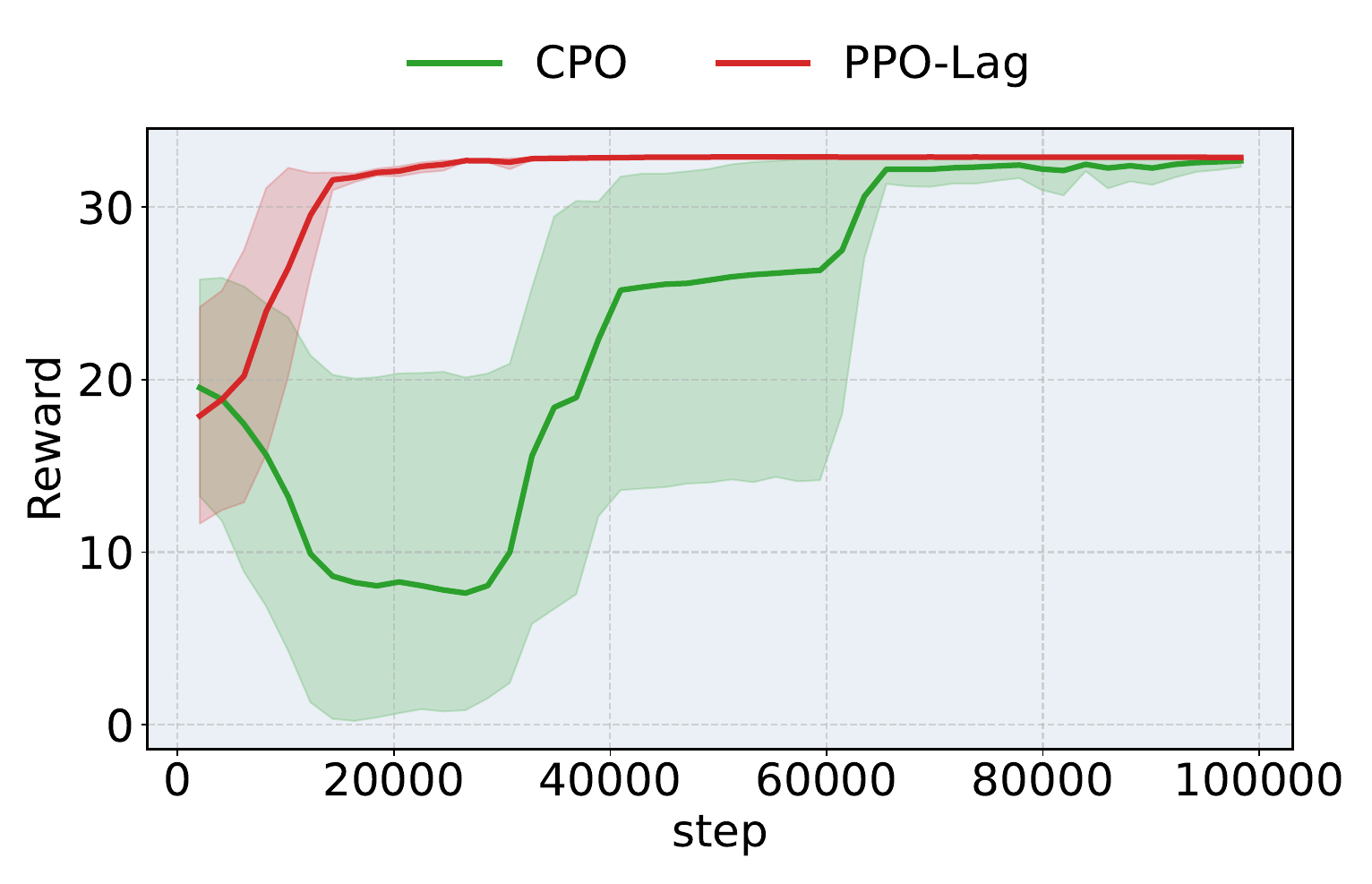}
        \includegraphics[width=0.49\textwidth]{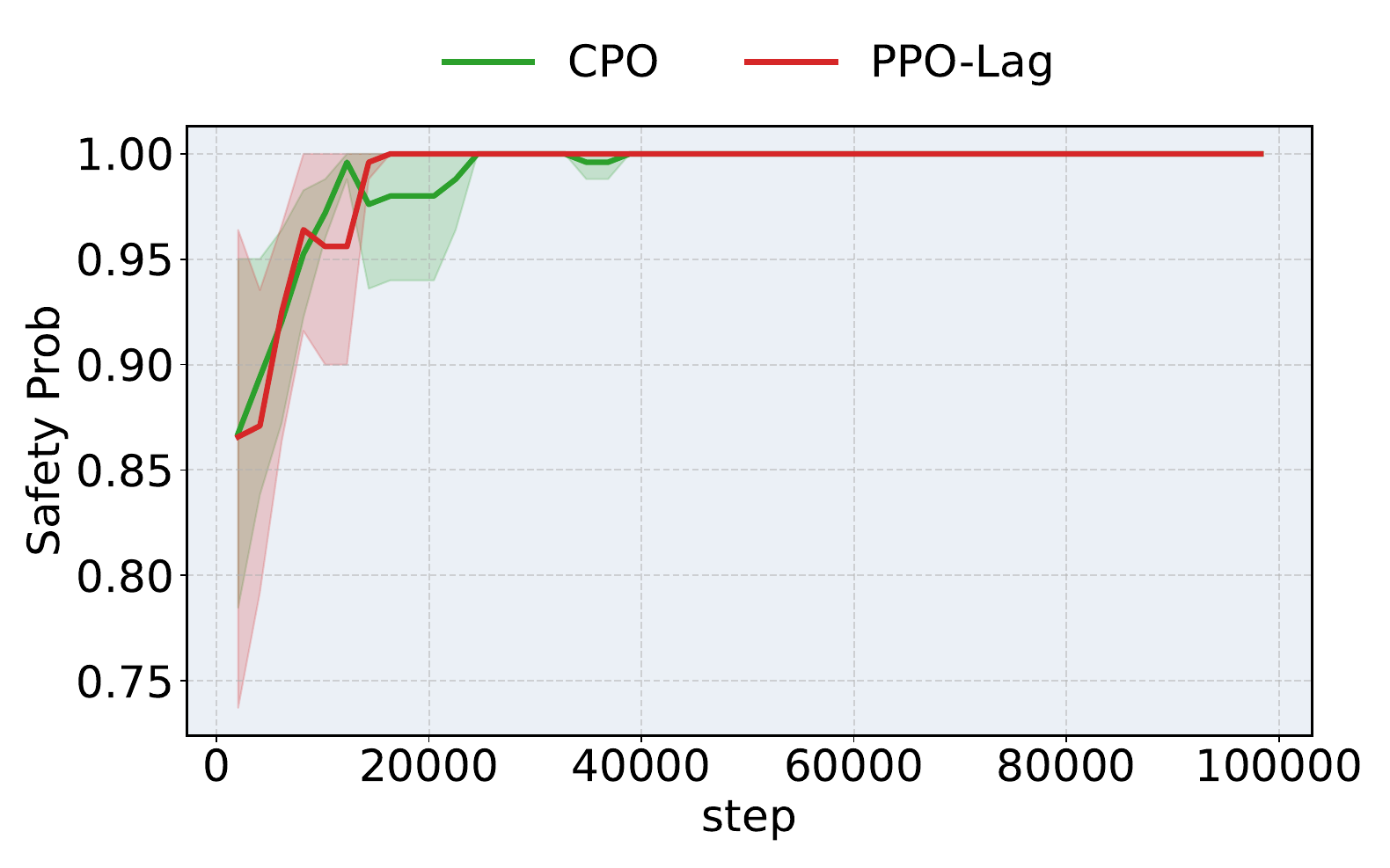}
        \caption{\texttt{obstacle}}
        
    \end{subfigure}
    \hfill
    \begin{subfigure}[t]{0.49\linewidth}
        \centering
        \includegraphics[width=0.49\textwidth]{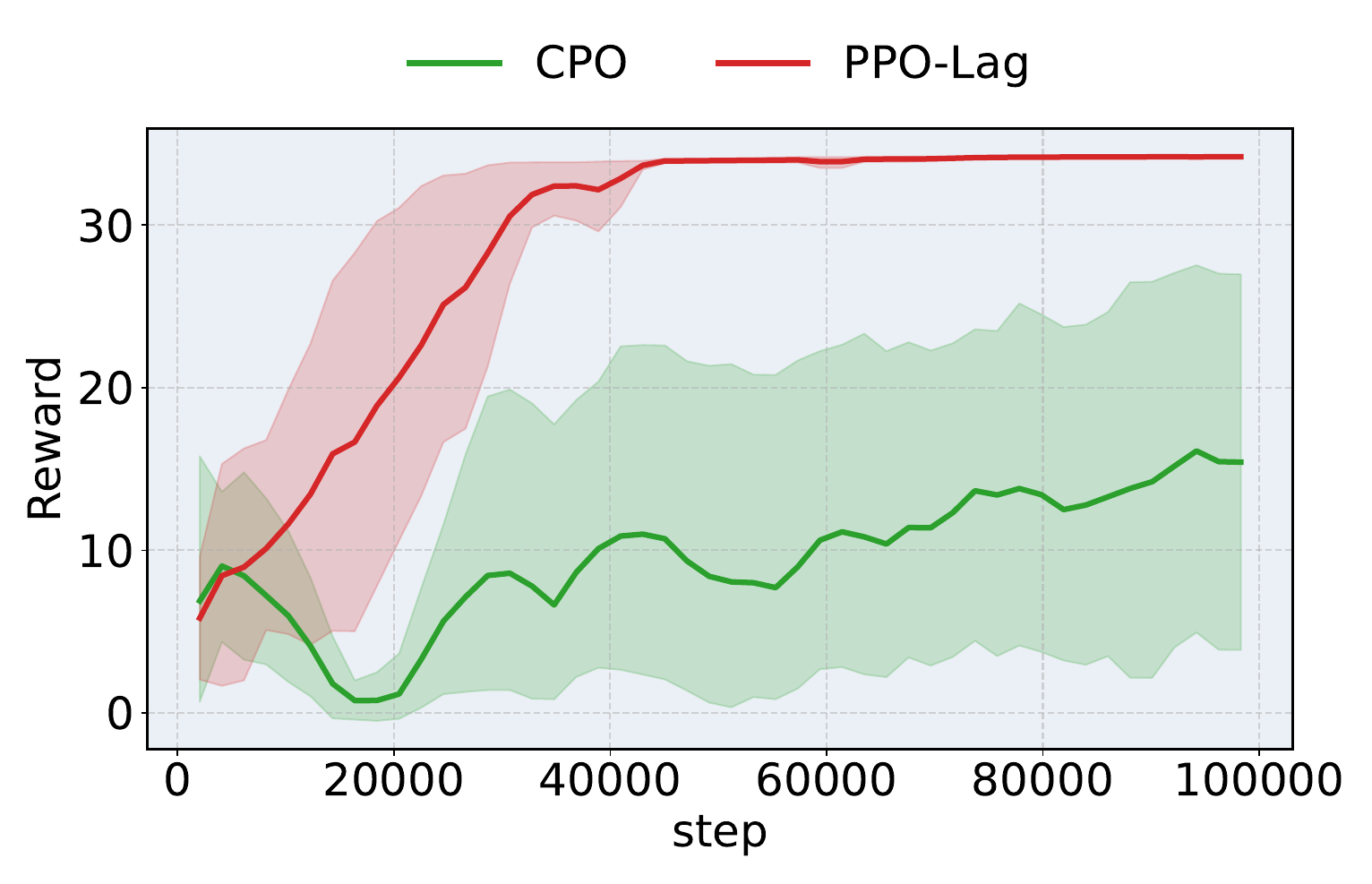}
        \includegraphics[width=0.49\textwidth]{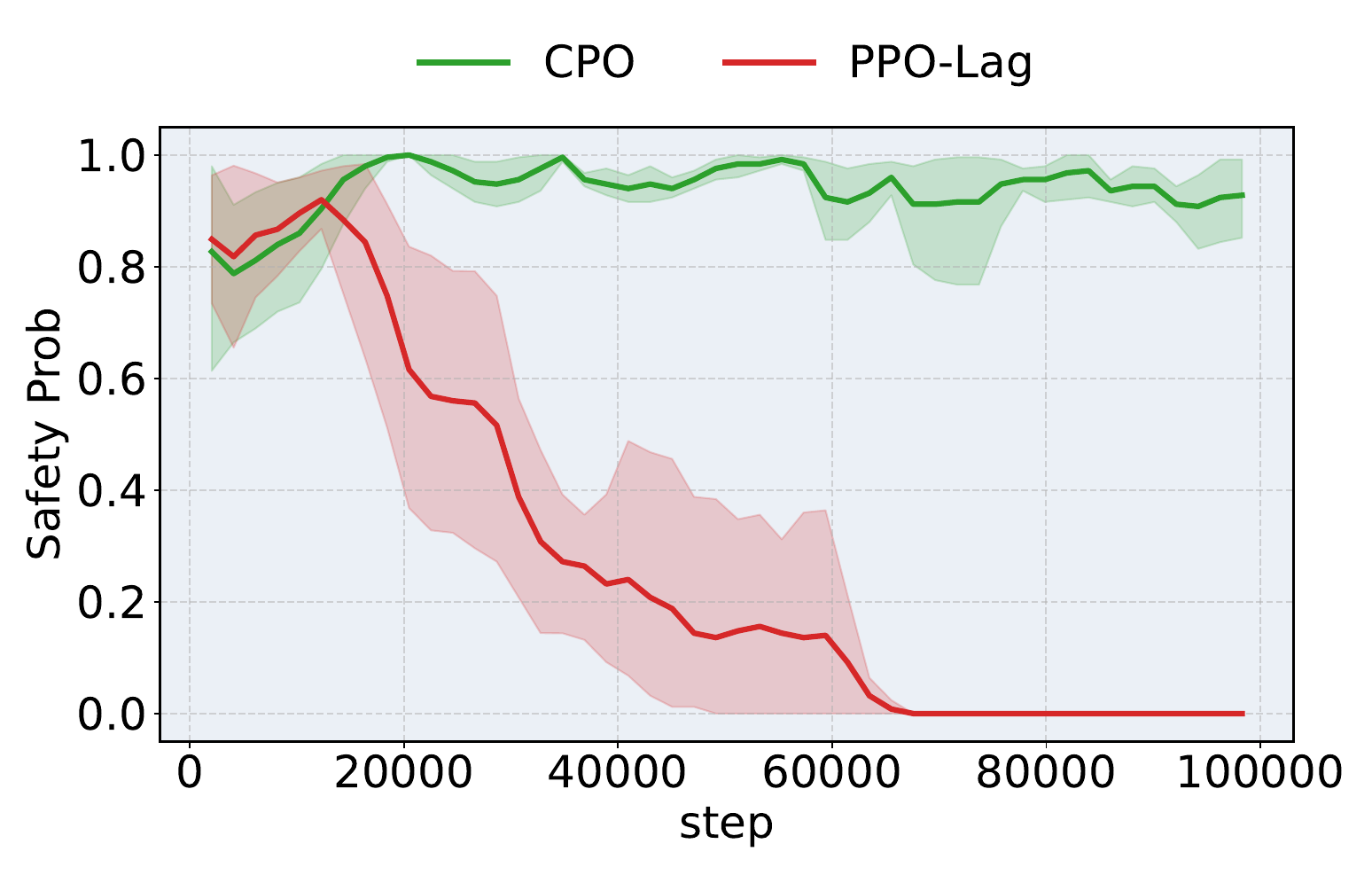}
        \caption{\texttt{obstacle2}}
    
    \end{subfigure}
    \hfill
    \begin{subfigure}[t]{0.49\linewidth}
        \centering
        \includegraphics[width=0.49\textwidth]{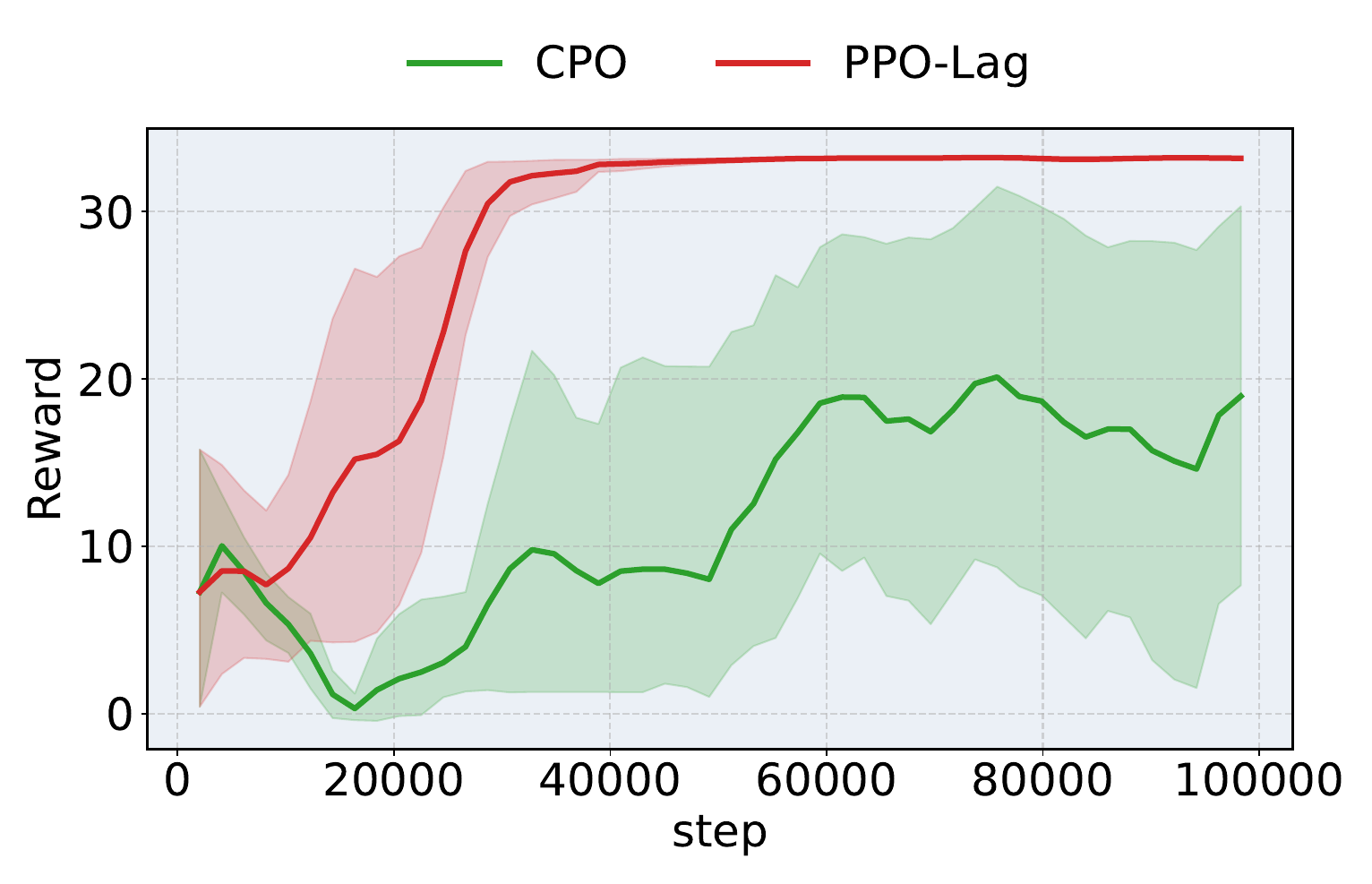}
        \includegraphics[width=0.49\textwidth]{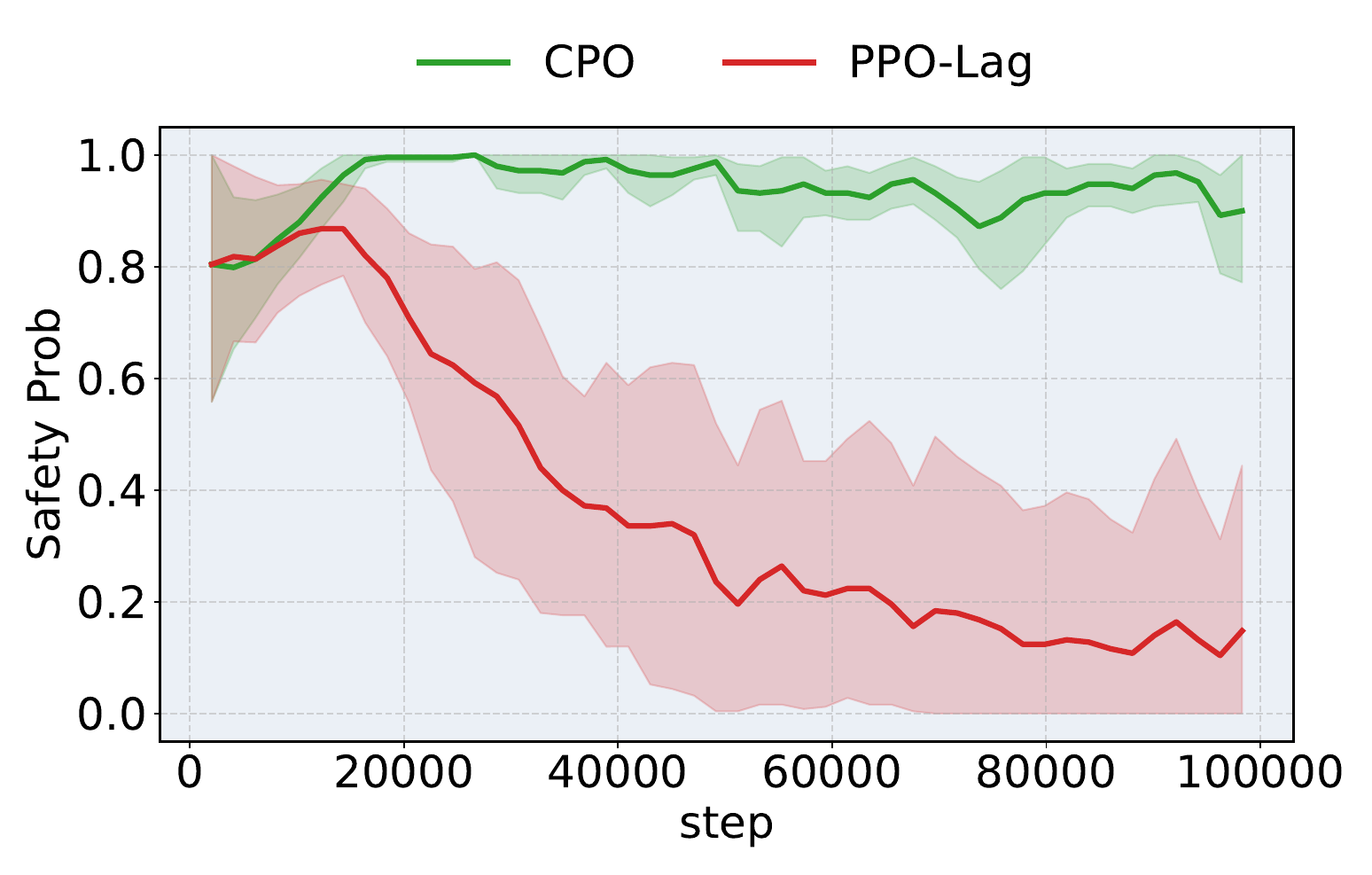}
        \caption{\texttt{obstacle3}}
    
    \end{subfigure}
    \hfill
    \begin{subfigure}[t]{0.49\linewidth}
        \centering
        \includegraphics[width=0.49\textwidth]{images/obstacle3__cpo_ppo_lag__ep_rew.pdf}
        \includegraphics[width=0.49\textwidth]{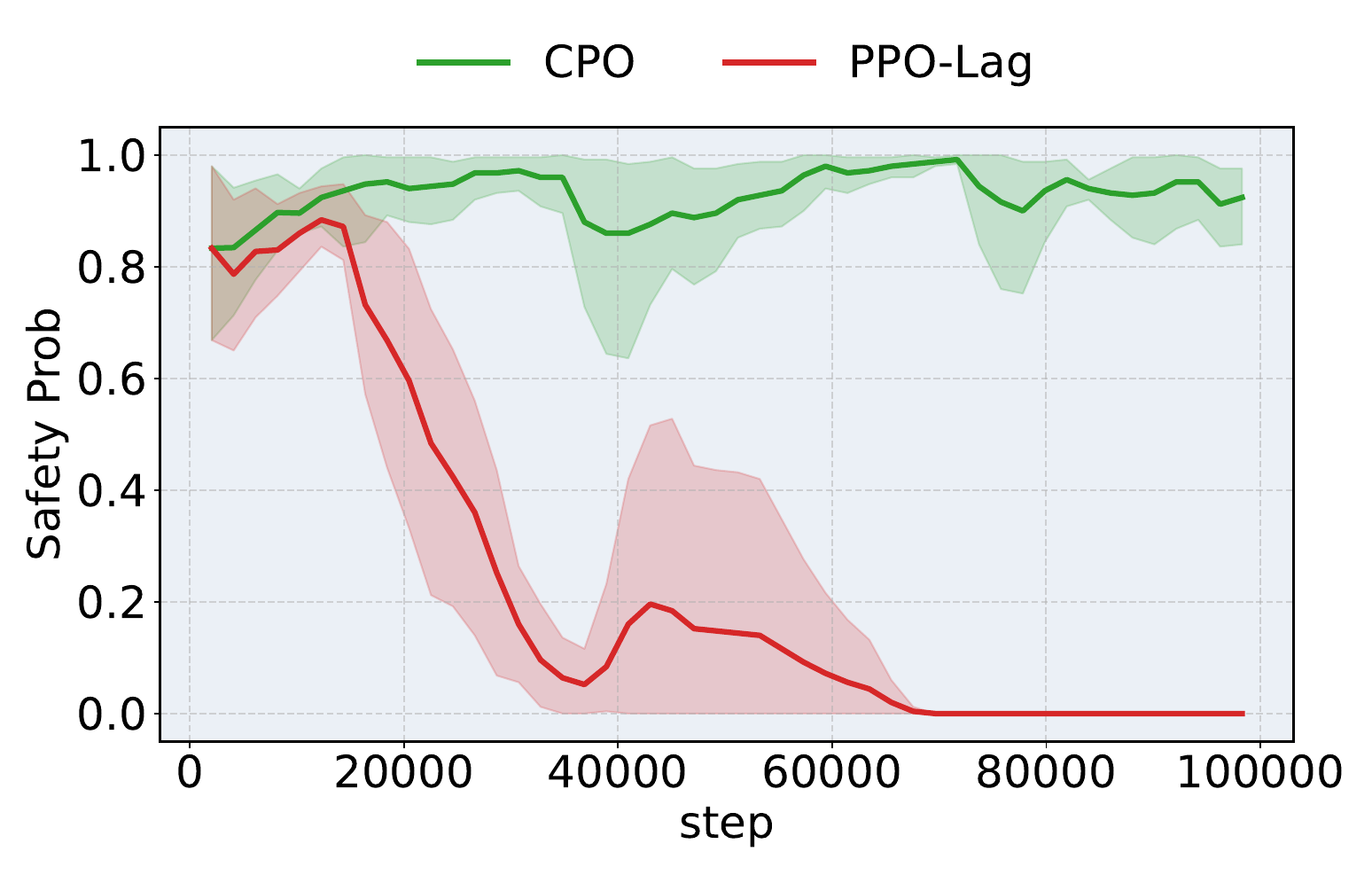}
        \caption{\texttt{obstacle4}}
    
    \end{subfigure}
    \hfill
    \begin{subfigure}[t]{0.49\linewidth}
        \centering
        \includegraphics[width=0.49\textwidth]{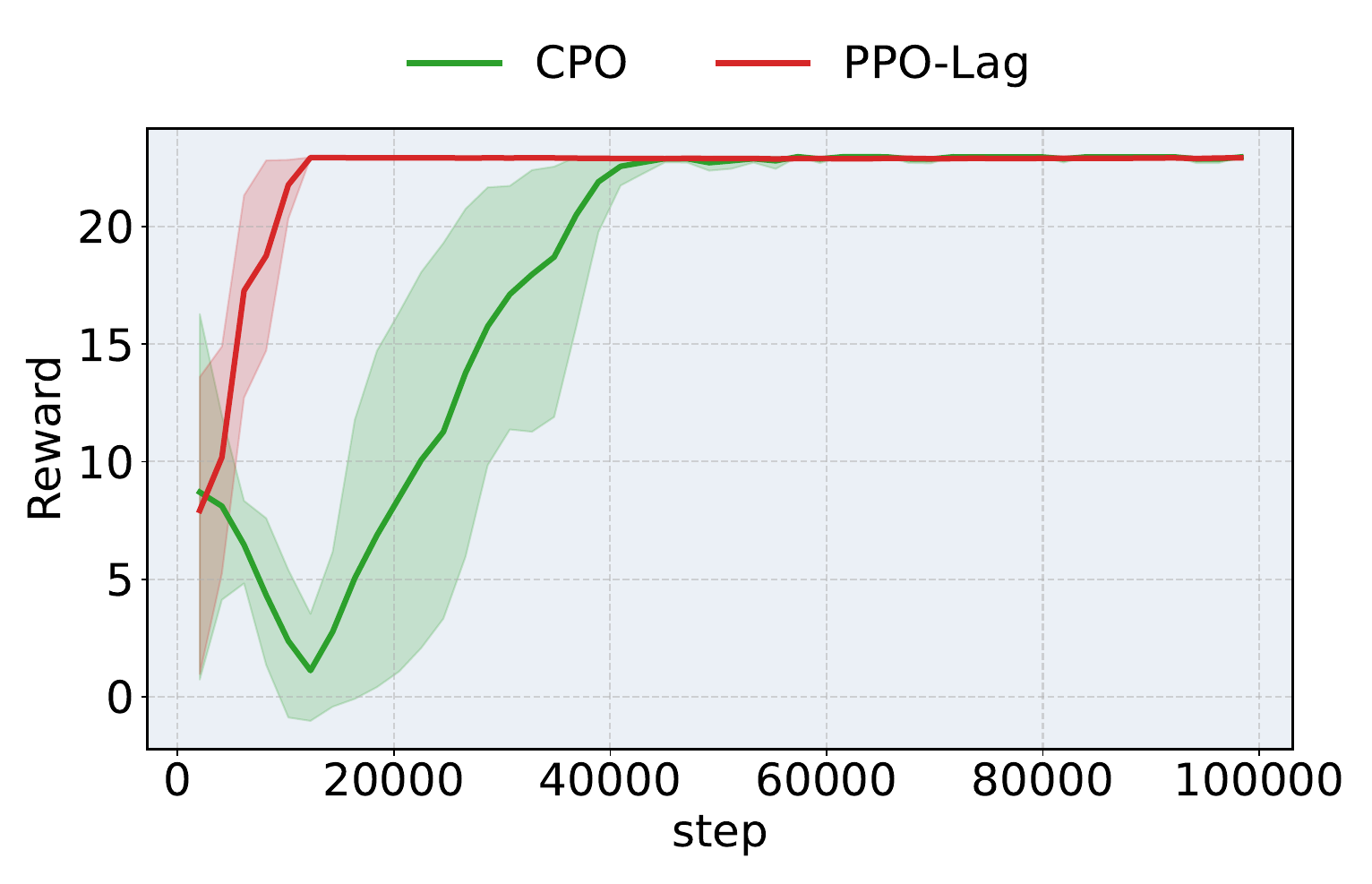}
        \includegraphics[width=0.49\textwidth]{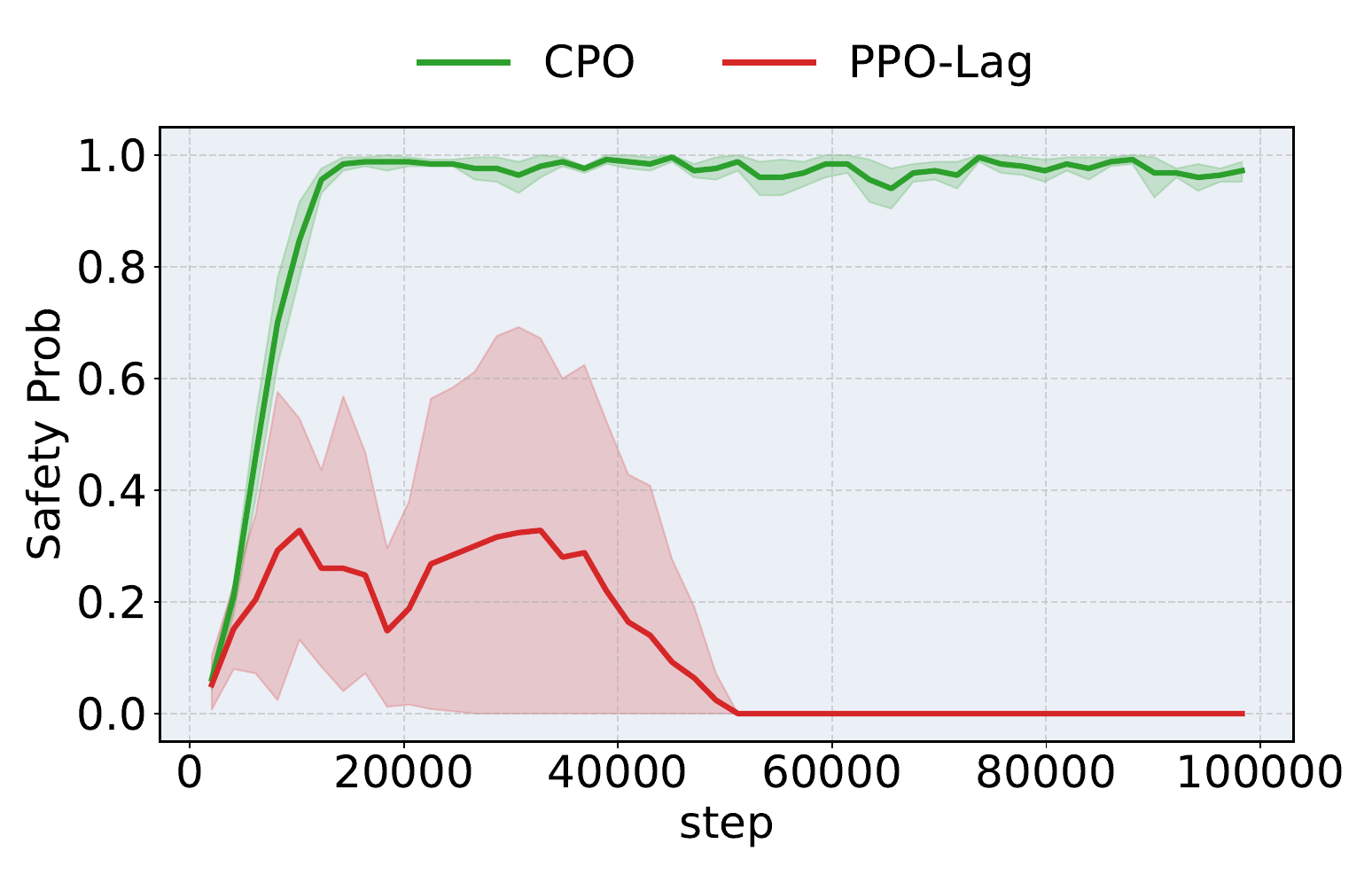}
        \caption{\texttt{road}}
        
    \end{subfigure}
    \hfill
    \begin{subfigure}[t]{0.49\linewidth}
        \centering
        \includegraphics[width=0.49\textwidth]{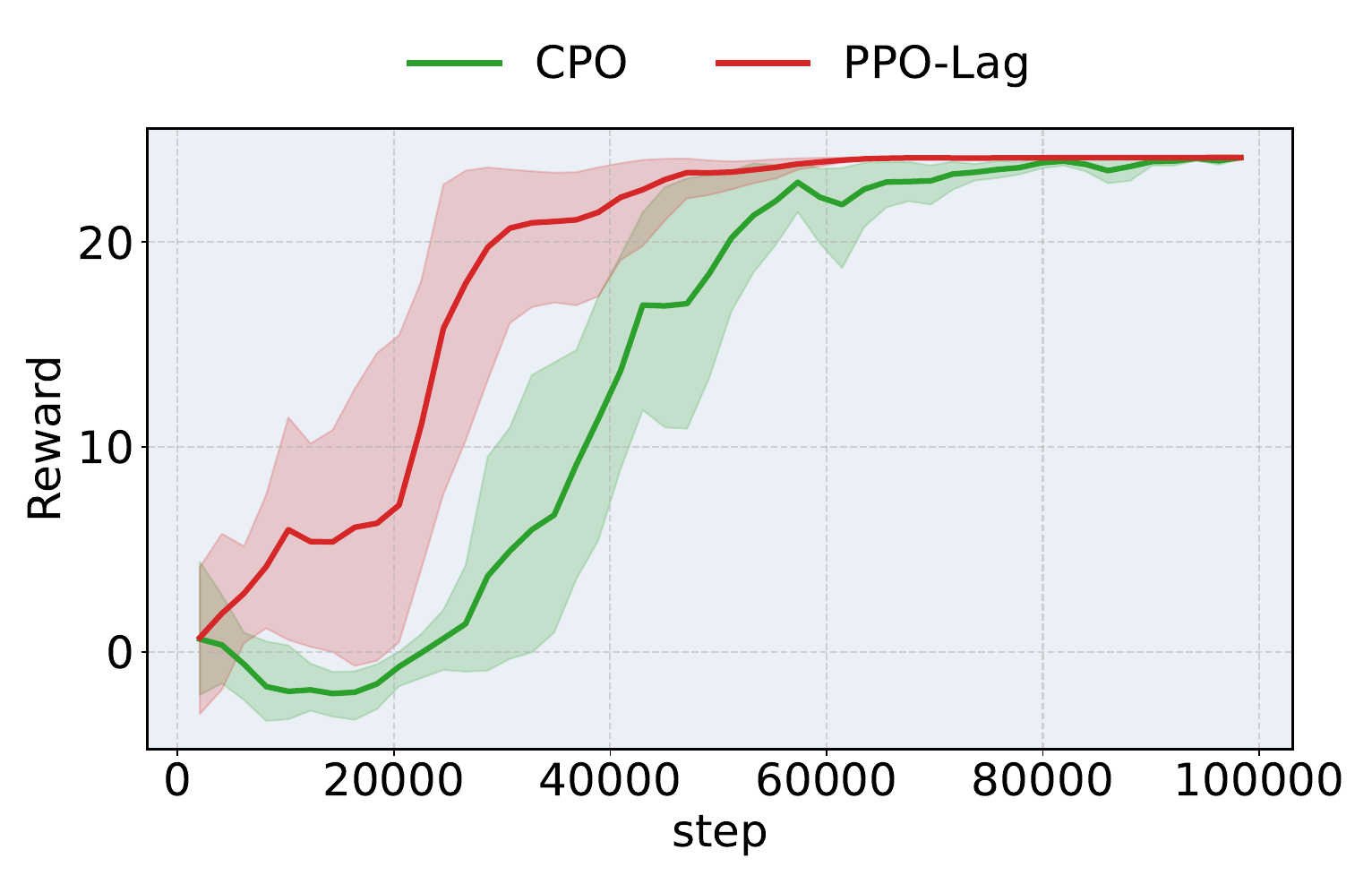}
        \includegraphics[width=0.49\textwidth]{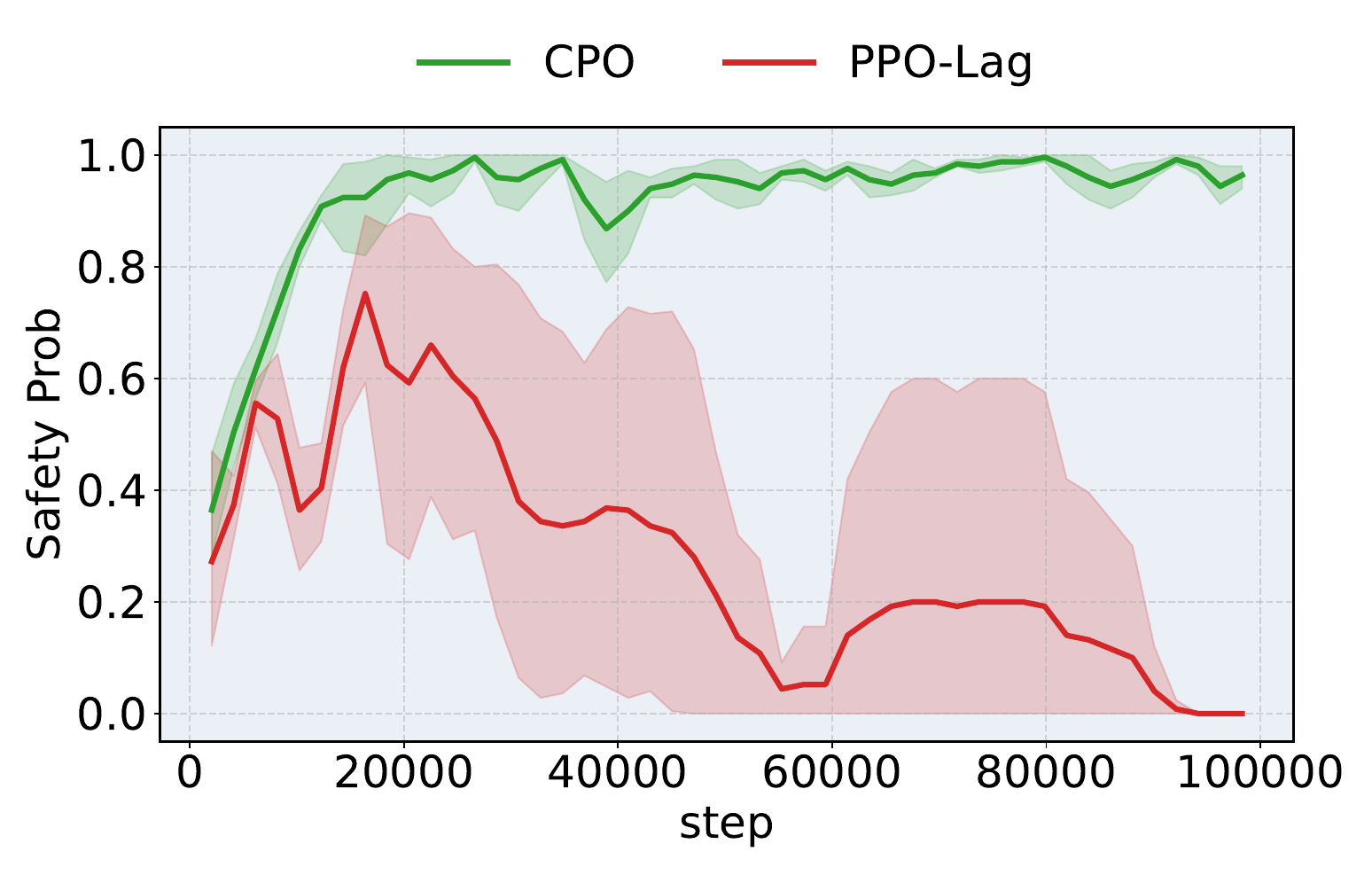}
        \caption{\texttt{road\_2d}} 
    \end{subfigure}
    \hfill
    \begin{subfigure}[t]{0.49\linewidth}
        \centering
        \includegraphics[width=0.49\textwidth]{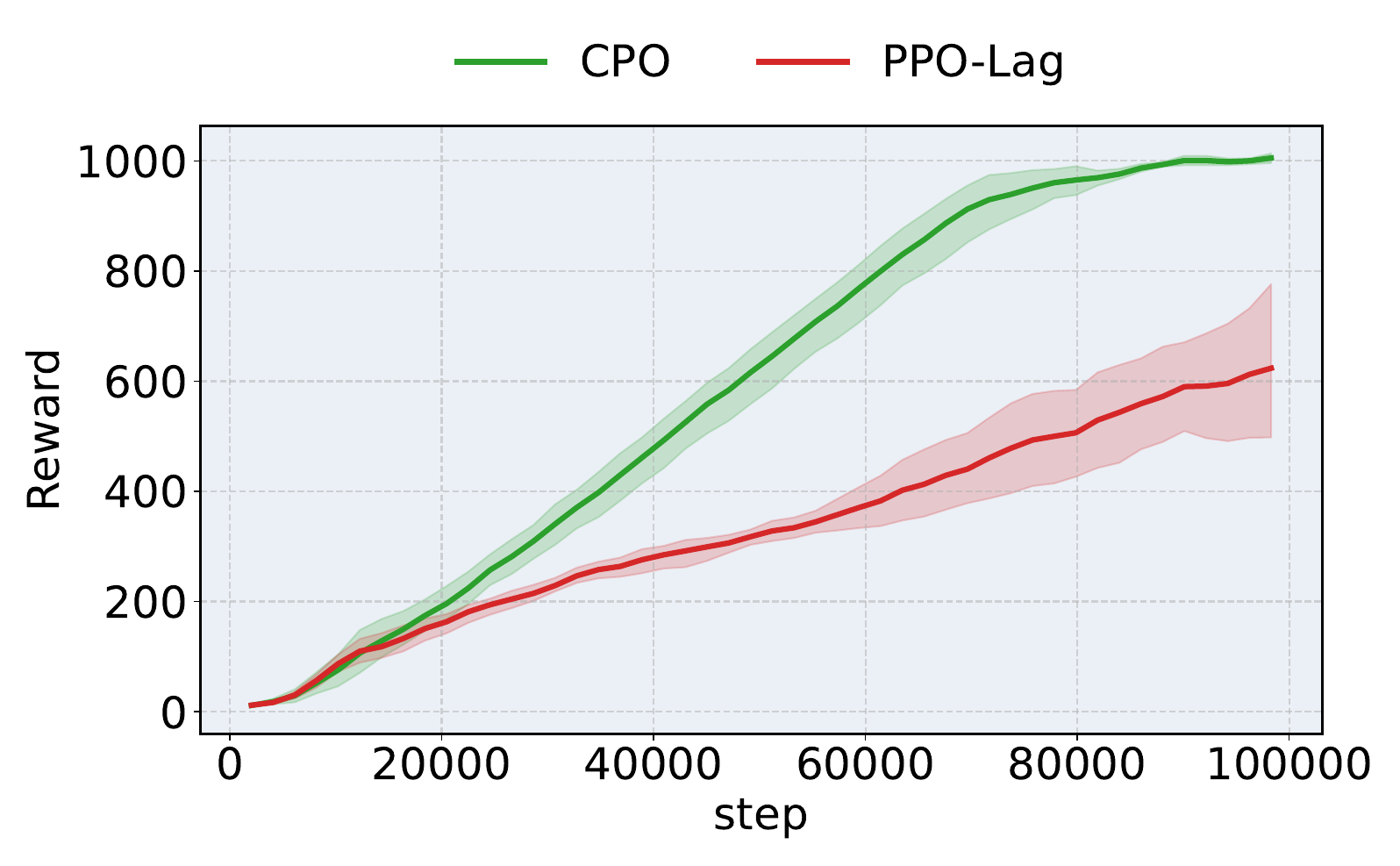}
        \includegraphics[width=0.49\textwidth]{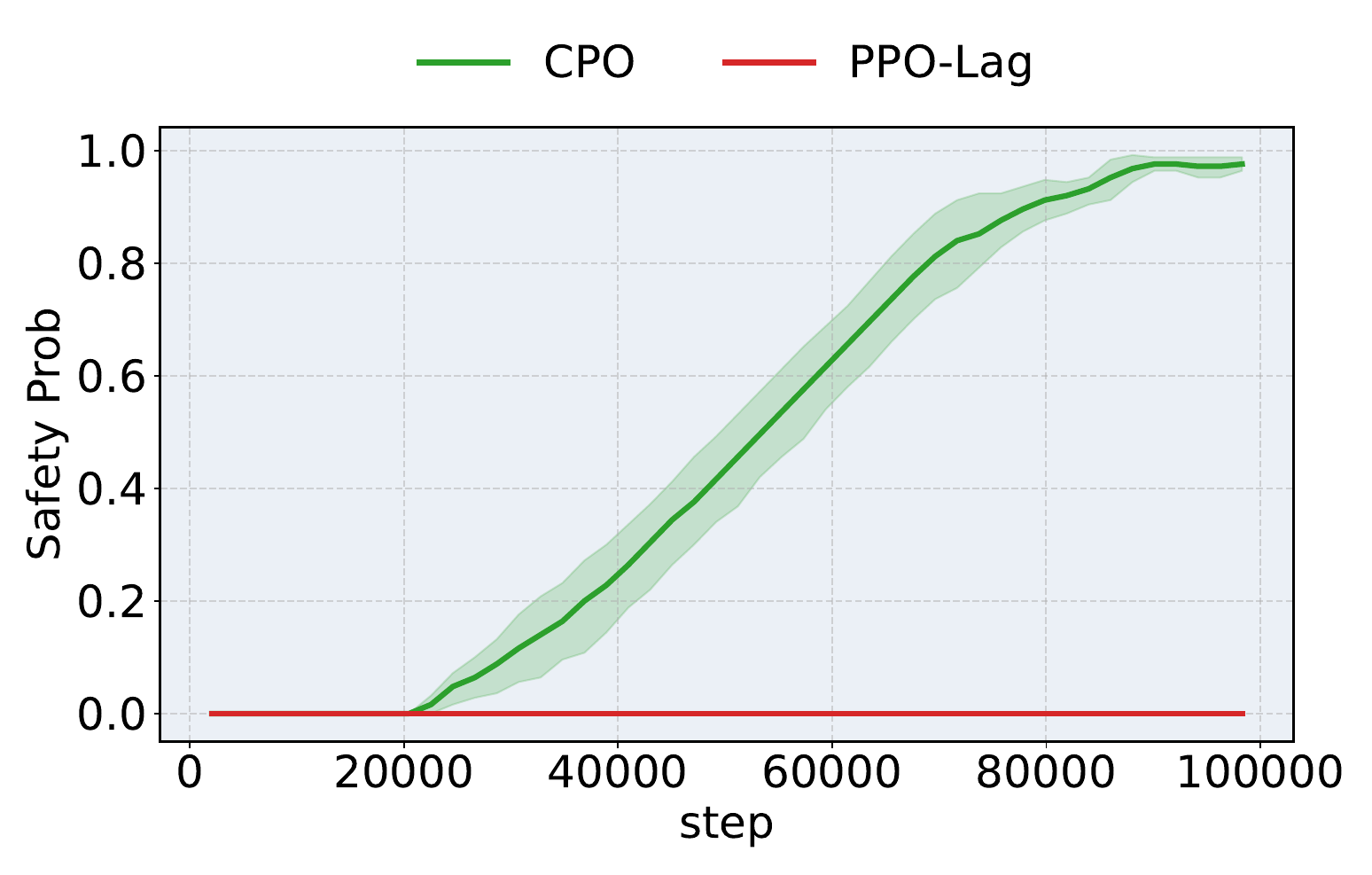}
        \caption{\texttt{Hopper-v5}} 
    \end{subfigure}
    \hfill
    \caption{Learning curves (reward and safety probability) for CPO and PPO-Lag.}
\end{figure*}
\clearpage
\newpage

\section{Additional analyses}
\label{sec:analyses}

\subsection{Comparison to Sampling}
\label{sec:samplinganalysis}

\paragraph{Effective number of samples.} We first conduct a quick analyses to identify the effective number of samples required to reach a step-wise safety guarantee of $1-\epsilon_t=0.9999$ (with high probability $1-\delta_t$), where we conservatively set $\delta_t=0.01$. \citet{bastanib2021safe} show that the number of samples required is,
\begin{equation*}
    K(\epsilon_t, \delta) \geq \frac{\log(1/\delta_t)}{\log(1/(1-\epsilon_t))} + 1
\end{equation*}
Plugging in $\epsilon_t=0.00001$ and $\delta_t=0.01$ gives $K \approx 46000$. We could not practically check this number of samples in a reasonable time, rather we tried $K \in \{5, 10, 50, 100, 500, 1000, 5000\} $, noting that $k = 5000$ gives an effective guarantee of $1-\epsilon_t\approx0.999$ (with high probability $1-\delta_t$).

\paragraph{Runtime comparison.} We now conduct a runtime comparison for action selection when using our analytic Gaussian approximation detailed in Sec.~\ref{sec:gaussianprocess}, compared to (approximately) sampling from the posterior distribution. Rather than draw exact samples from the posterior distribution (which scales cubically in the number of inputs $D$ and is far too slow) we approximate the model's kernel with a Fourier function decomposition $\hat f(x(\cdot))= \sum^{m}_{i=1} \phi_i(x(\cdot)) \cdot \theta_i$, where $\phi_i$ is the $i^{\text{th}}$ Fourier feature sampled from the model's kernel and $\theta_i$ is a sample from the standard Gaussian \cite{Pinder2022}, by default $m=500$.

Functionally, this form of the posterior distribution uses the same Fourier feature samples to evaluate all queries, meaning evaluating the posterior has a constant cost regardless of the number of posterior samples that are drawn. While the computational complexity of drawing samples is significantly reduced, the practical containment checks still scale $O(NK)$ where $N$ is the recovery horizon and $K$ is the number of samples, at some point this term dominates the complexity of drawing samples, revealing an inherent limitation with sampling-based approaches, see Tab.~\ref{tab:complexity} and Tab.~\ref{tab:complexity2}

\begin{table}[ht!]
  \caption{Runtime comparison (seconds) for \texttt{gaussian\_approx} and \texttt{posterior\_sampling\_approx} on \texttt{cartpole} ($n=4$), with GP posterior with $D=1000$ inputs.}
  \label{tab:complexity}
  \centering
  \small
  \begin{tabular}{lcccccccc}
    \toprule
    %\multicolumn{2}{c}{Part}                   \\
    %\cmidrule(r){1-2}

      \multirow{ 2}{*}{Recovery Horizon}& \texttt{gaussian\_approx} & \multicolumn{7}{c}{\texttt{posterior\_sampling\_approx}} \\
       & ($\epsilon_t = 0.9999$) & $K=5$ & $K=10$ & $K=50$ & $K=100$ & $K=500$ & $K=1000$ & $K=5000$\\
    \midrule
     $N=5$ & 0.749 & 0.537 & 0.533 & 0.648 & 0.619 & 0.683 & 0.910 & 6.890 \\
$N=10$ & 1.090 & 0.985 & 1.015 & 1.151 & 1.149 & 1.241 & 1.659 & 12.614 \\
$N=20$ & 1.652 & 1.857 & 1.911 & 2.221 & 2.159 & 2.358 & 3.155 & 35.563 \\
$N=40$ & 2.788 & 3.634 & 3.804 & 4.247 & 4.278 & 4.649 & 6.162 & 47.005 \\
$N=100$ & 6.495 & 8.857 & 8.764 & 10.566 & 10.482 & 11.294 & 15.187 & 115.771 \\
    \bottomrule
  \end{tabular}
\end{table}

\begin{table}[ht!]
  \caption{Runtime comparison (seconds) for \texttt{gaussian\_approx} and \texttt{posterior\_sampling\_approx} on \texttt{cartpole} ($n=4$), with SGPR posterior \cite{titsias2009variational} with $D=5000$ inputs and $Z=500$ inducing points.}
  \label{tab:complexity2}
  \centering
  \small
  \begin{tabular}{lcccccccc}
    \toprule
    %\multicolumn{2}{c}{Part}                   \\
    %\cmidrule(r){1-2}

      \multirow{ 2}{*}{Recovery Horizon}& \texttt{gaussian\_approx} & \multicolumn{7}{c}{\texttt{posterior\_sampling\_approx}} \\
       & ($\epsilon_t = 0.9999$) & $K=5$ & $K=10$ & $K=50$ & $K=100$ & $K=500$ & $K=1000$ & $K=5000$\\
    \midrule
     $N=5$ & 0.703 & 15.752 & 15.794 & 16.135 & 16.284 & 17.591 & 19.841 & 47.535 \\
$N=10$ & 1.230 & 28.854 & 28.943 & 29.562 & 29.845 & 32.243 & 36.369 & 87.142 \\
$N=20$ & 2.384 & 55.087 & 55.253 & 56.441 & 56.992 & 61.547 & 69.431 & 166.362 \\
$N=40$ & 4.353 & 107.614 & 107.881 & 110.180 & 111.251 & 120.161 & 135.552 & 324.794 \\
$N=100$ & 10.4313 & 264.880 & 265.655 & 271.443 & 274.051 & 296.042 & 333.899 & 800.059 \\
    \bottomrule
  \end{tabular}
\end{table}

\subsection{Skew and kurtosis}
\label{sec:skewkurtanalysis}

To verify assumption (iii) in Theorem \ref{thm:ellipsoids}, we conduct a skewness and kurtosis analysis for each of our environments used in our experimental evaluation (Sec.~\ref{sec:experiments}). In particular, we first train a GP posterior $f$ using data collected during the training dynamics of a given run. We then sample a random initial point and run the \texttt{posteror\_sampling\_approx} subroutine to generate an empirical distribution of $K$ points for each $t = 0, 1, \ldots N$. We denote $\tilde X(t) = [ \tilde x_1(t), \ldots, \tilde x_K(t)] $ as the empirical set of points for each $t = 0, 1, \ldots N$. We then compute Mardia's skewness and kurtosis, i.e., for some $t = 0, 1, \ldots N$ we have,
\begin{equation*}
    \text{skew}(t)  = \frac{1}{K^2}\sum^K_{i=1}\sum^K_{j=1} m_{ij}^3(t), \quad \text{kurt}(t)  = \frac{1}{K} \sum^K_{i=1}m_{ii}^2(t)
\end{equation*}
where,
\begin{align*}
    m_{ij}(t) &= (\tilde X_i(t) - \bar X(t)) S(t)^{-1}(\tilde X_i(t) - \bar X(t)) \quad \text{(Mahalanobis distance)}\\
    S(t) & = \frac{1}{K} \sum^K_{i=0} (\tilde x_i(t) - \bar X(t)) (\tilde x_i(t) - \bar X(t))^T\\
    \bar X(t) &= \frac{1}{K}\sum^{K}_{i=0} \tilde x_i(t)
\end{align*}
For truly multivariate distributions we should see,
\begin{align*}
    \text{skew}(t) \to 0, \quad \text{kurt}(t) \to n(n+2)
\end{align*}
as $K \to \infty$. Since we can only (approximately) sample a finite number of points $K$ we setup the usual test statistics. For Mardia's skewness if the empirical sample comes from a multivariate normal distribution then,
\begin{equation*}
    \frac{K}{6}\text{skew}(t) \sim \chi^2\left(\frac{n(n+1(n+2))}{6}\right)
\end{equation*}
the skewness is considered negligible here if the $p$-value $> 0.05$. For Mardia's kurtosis if the empirical sample comes from a multivariate normal distribution then,
\begin{equation*}
    \vert\text{kurt}(t) - n(n+2)\vert \sqrt{\frac{K}{8n(n+2)}} \sim N(0,1)
\end{equation*}
Thus if $\text{kurt}(t)$ is within 1-2 standard errors of $n(n+2)$ this is acceptable, we can setup the usual test statistic for the standard normal distribution and check if the $p$-value $> 0.05$. We present the following skewness and kurtosis analysis for each environment below in tabular format. For the most part the skewness and kurtosis is negligible across all of our environments, passing the statistical tests (or marginally failing some). We note that any discrepancies could also be down to the bias introduced in the approximate sampling procedure used (see Sec.~\ref{sec:samplinganalysis}).  

\paragraph{Tabular results.} We present all the results in tabular form, in each instance $K=5000$ samples were taken for highly robust estimates of the skewness and kurtosis. 

% TODO - fill the following tables with numbers

\begin{table}[ht!]
  \caption{Skewness and kurtosis analysis on \texttt{cartpole} ($n=4$) with GP posterior with $D=1000$ inputs.}
  \label{tab:skewkurt1}
  \centering
  \small
  \begin{tabular}{lccccccc}
    \toprule
    %\multicolumn{2}{c}{Part}                   \\
    %\cmidrule(r){1-2}

      & overall pass& $\text{skew}(t)$ & skew $p$-value & skew pass & $\text{kurt}(t)$ & kurt $p$-value & kurt pass \\
    \midrule
$t=0$ & Pass & 0.0228 & 0.5245 & Pass & 24.0808 & 0.6803 & Pass \\
$t=1$ & Pass & 0.0244 & 0.4374 & Pass & 24.0433 & 0.8252 & Pass \\
$t=2$ & Pass & 0.0313 & 0.1629 & Pass & 24.0392 & 0.8416 & Pass \\
$t=3$ & Pass & 0.0339 & 0.1038 & Pass & 24.1253 & 0.5225 & Pass \\
$t=4$ & Pass & 0.0280 & 0.2745 & Pass & 24.0372 & 0.8495 & Pass \\
$t=5$ & Pass & 0.0287 & 0.2456 & Pass & 24.2263 & 0.2482 & Pass \\
$t=6$ & Pass & 0.0367 & 0.0607 & Pass & 24.1273 & 0.5160 & Pass \\
$t=7$ & Pass & 0.0354 & 0.0786 & Pass & 24.0859 & 0.6613 & Pass \\
$t=8$ & Fail & 0.0410 & 0.0250 & Fail & 24.2315 & 0.2374 & Pass \\
$t=9$ & Fail & 0.0299 & 0.2054 & Pass & 24.4235 & 0.0307 & Fail \\
$t=10$ & Fail & 0.0232 & 0.5016 & Pass & 24.5831 & 0.0029 & Fail \\
\multicolumn{8}{c}{...}\\
$t=90$ & Pass & 0.0202 & 0.6647 & Pass & 23.8474 & 0.4361 & Pass \\
$t=91$ & Pass & 0.0244 & 0.4350 & Pass & 24.0161 & 0.9345 & Pass \\
$t=92$ & Pass & 0.0286 & 0.2499 & Pass & 24.1215 & 0.5352 & Pass \\
$t=93$ & Pass & 0.0326 & 0.1303 & Pass & 24.2817 & 0.1506 & Pass \\
$t=94$ & Pass & 0.0287 & 0.2476 & Pass & 24.2922 & 0.1359 & Pass \\
$t=95$ & Pass & 0.0247 & 0.4220 & Pass & 24.0463 & 0.8131 & Pass \\
$t=96$ & Pass & 0.0301 & 0.1972 & Pass & 23.9203 & 0.6844 & Pass \\
$t=97$ & Pass & 0.0306 & 0.1839 & Pass & 23.9490 & 0.7947 & Pass \\
$t=98$ & Pass & 0.0295 & 0.2174 & Pass & 23.9738 & 0.8934 & Pass \\
$t=99$ & Pass & 0.0260 & 0.3601 & Pass & 23.8628 & 0.4838 & Pass \\
$t=100$ & Pass & 0.0175 & 0.8001 & Pass & 23.9352 & 0.7407 & Pass \\
    \bottomrule
  \end{tabular}
\end{table}

\newpage
\begin{table}[ht!]
  \caption{Skewness and kurtosis analysis on \texttt{mountain\_car} ($n=2$) with GP posterior with $D=1000$ inputs.}
  \label{tab:skewkurt2}
  \centering
  \small
  \begin{tabular}{lccccccc}
    \toprule
    %\multicolumn{2}{c}{Part}                   \\
    %\cmidrule(r){1-2}

      & overall pass& $\text{skew}(t)$ & skew $p$-value & skew pass & $\text{kurt}(t)$ & kurt $p$-value & kurt pass \\
    \midrule
     $t=0$ & Pass & 0.0049 & 0.3999 & Pass & 8.1132 & 0.3169 & Pass \\
$t=1$ & Fail & 0.0132 & 0.0269 & Fail & 8.0025 & 0.9824 & Pass \\
$t=2$ & Fail & 0.0141 & 0.0191 & Fail & 7.8566 & 0.2050 & Pass \\
$t=3$ & Pass & 0.0095 & 0.0935 & Pass & 7.9025 & 0.3889 & Pass \\
$t=4$ & Pass & 0.0109 & 0.0595 & Pass & 7.9244 & 0.5043 & Pass \\
$t=5$ & Pass & 0.0099 & 0.0834 & Pass & 8.0009 & 0.9938 & Pass \\
$t=6$ & Pass & 0.0017 & 0.8399 & Pass & 7.9185 & 0.4714 & Pass \\
$t=7$ & Pass & 0.0013 & 0.8930 & Pass & 7.9599 & 0.7228 & Pass \\
$t=8$ & Pass & 0.0009 & 0.9413 & Pass & 7.9219 & 0.4899 & Pass \\
$t=9$ & Pass & 0.0068 & 0.2285 & Pass & 7.8764 & 0.2745 & Pass \\
$t=10$ & Pass & 0.0072 & 0.2002 & Pass & 8.0056 & 0.9608 & Pass \\
\multicolumn{8}{c}{...}\\
$t=70$ & Pass & 0.0014 & 0.8863 & Pass & 8.0283 & 0.8025 & Pass \\
$t=71$ & Pass & 0.0006 & 0.9720 & Pass & 7.9748 & 0.8235 & Pass \\
$t=72$ & Pass & 0.0007 & 0.9658 & Pass & 8.0059 & 0.9584 & Pass \\
$t=73$ & Pass & 0.0002 & 0.9978 & Pass & 7.9946 & 0.9621 & Pass \\
$t=74$ & Pass & 0.0011 & 0.9268 & Pass & 8.0131 & 0.9081 & Pass \\
$t=75$ & Pass & 0.0019 & 0.8148 & Pass & 8.0321 & 0.7764 & Pass \\
$t=76$ & Pass & 0.0029 & 0.6525 & Pass & 8.0448 & 0.6922 & Pass \\
$t=77$ & Pass & 0.0025 & 0.7221 & Pass & 8.0364 & 0.7479 & Pass \\
$t=78$ & Pass & 0.0019 & 0.8191 & Pass & 7.8971 & 0.3631 & Pass \\
$t=79$ & Pass & 0.0018 & 0.8222 & Pass & 7.9648 & 0.7560 & Pass \\
$t=80$ & Pass & 0.0026 & 0.7002 & Pass & 7.9582 & 0.7117 & Pass \\
    \bottomrule
  \end{tabular}
\end{table}

\begin{table}[ht!]
  \caption{Skewness and kurtosis analysis on \texttt{obstacle} ($n=4$) with GP posterior with $D=1000$ inputs.}
  \label{tab:skewkurt3}
  \centering
  \small
  \begin{tabular}{lccccccc}
    \toprule
    %\multicolumn{2}{c}{Part}                   \\
    %\cmidrule(r){1-2}

      & overall pass& $\text{skew}(t)$ & skew $p$-value & skew pass & $\text{kurt}(t)$ & kurt $p$-value & kurt pass \\
    \midrule
$t=0$ & Pass & 0.0222 & 0.5542 & Pass & 24.2452 & 0.2108 & Pass \\
$t=1$ & Pass & 0.0297 & 0.2123 & Pass & 23.8672 & 0.4980 & Pass \\
$t=2$ & Fail & 0.0487 & 0.0042 & Fail & 23.7801 & 0.2618 & Pass \\
$t=3$ & Pass & 0.0368 & 0.0595 & Pass & 23.8696 & 0.5057 & Pass \\
$t=4$ & Pass & 0.0361 & 0.0684 & Pass & 23.7830 & 0.2681 & Pass \\
$t=5$ & Pass & 0.0262 & 0.3484 & Pass & 23.7300 & 0.1682 & Pass \\
$t=6$ & Pass & 0.0351 & 0.0826 & Pass & 23.7324 & 0.1721 & Pass \\
$t=7$ & Fail & 0.0465 & 0.0072 & Fail & 24.1845 & 0.3465 & Pass \\
$t=8$ & Fail & 0.0591 & 0.0003 & Fail & 24.1589 & 0.4174 & Pass \\
$t=9$ & Fail & 0.0702 & 0.0000 & Fail & 23.7868 & 0.2765 & Pass \\
$t=10$ & Pass & 0.0284 & 0.2575 & Pass & 23.6276 & 0.0574 & Pass \\
$t=11$ & Pass & 0.0298 & 0.2084 & Pass & 23.8470 & 0.4349 & Pass \\
$t=12$ & Fail & 0.0497 & 0.0033 & Fail & 23.9000 & 0.6099 & Pass \\
$t=13$ & Fail & 0.0967 & 0.0000 & Fail & 24.1875 & 0.3385 & Pass \\
$t=14$ & Fail & 0.0750 & 0.0000 & Fail & 24.1454 & 0.4580 & Pass \\
$t=15$ & Fail & 0.0797 & 0.0000 & Fail & 24.0992 & 0.6128 & Pass \\
$t=16$ & Fail & 0.0880 & 0.0000 & Fail & 24.0947 & 0.6289 & Pass \\
$t=17$ & Fail & 0.1005 & 0.0000 & Fail & 24.2734 & 0.1630 & Pass \\
$t=18$ & Fail & 0.0894 & 0.0000 & Fail & 24.1385 & 0.4798 & Pass \\
$t=19$ & Fail & 0.1050 & 0.0000 & Fail & 24.5370 & 0.0061 & Fail \\
$t=20$ & Fail & 0.1182 & 0.0000 & Fail & 24.5583 & 0.0044 & Fail \\
    \bottomrule
  \end{tabular}
\end{table}

\begin{table}[ht!]
  \caption{Skewness and kurtosis analysis on \texttt{obstacle2/3/4} ($n=4$) with GP posterior with $D=1000$ inputs.}
  \label{tab:skewkurt4}
  \centering
  \small
  \begin{tabular}{lccccccc}
    \toprule
    %\multicolumn{2}{c}{Part}                   \\
    %\cmidrule(r){1-2}

      & overall pass& $\text{skew}(t)$ & skew $p$-value & skew pass & $\text{kurt}(t)$ & kurt $p$-value & kurt pass \\
    \midrule
     $t=1$ & Pass & 0.0139 & 0.9291 & Pass & 24.0364 & 0.8525 & Pass \\
$t=2$ & Pass & 0.0313 & 0.1639 & Pass & 24.0531 & 0.7862 & Pass \\
$t=3$ & Pass & 0.0352 & 0.0810 & Pass & 24.0604 & 0.7580 & Pass \\
$t=4$ & Pass & 0.0303 & 0.1912 & Pass & 24.0422 & 0.8295 & Pass \\
$t=5$ & Fail & 0.0463 & 0.0075 & Fail & 24.4101 & 0.0364 & Fail \\
$t=6$ & Fail & 0.0676 & 0.0000 & Fail & 24.2153 & 0.2719 & Pass \\
$t=7$ & Fail & 0.0485 & 0.0045 & Fail & 24.2753 & 0.1601 & Pass \\
$t=8$ & Pass & 0.0371 & 0.0567 & Pass & 24.2335 & 0.2335 & Pass \\
$t=9$ & Pass & 0.0359 & 0.0713 & Pass & 24.0308 & 0.8750 & Pass \\
$t=10$ & Pass & 0.0351 & 0.0822 & Pass & 23.8916 & 0.5803 & Pass \\
$t=11$ & Fail & 0.0517 & 0.0020 & Fail & 24.1671 & 0.3939 & Pass \\
$t=12$ & Fail & 0.0566 & 0.0006 & Fail & 24.1577 & 0.4211 & Pass \\
$t=13$ & Fail & 0.0594 & 0.0003 & Fail & 24.1114 & 0.5698 & Pass \\
$t=14$ & Fail & 0.0688 & 0.0000 & Fail & 24.2100 & 0.2838 & Pass \\
$t=15$ & Fail & 0.0789 & 0.0000 & Fail & 24.0984 & 0.6156 & Pass \\
$t=16$ & Fail & 0.0956 & 0.0000 & Fail & 24.2683 & 0.1710 & Pass \\
$t=17$ & Fail & 0.0851 & 0.0000 & Fail & 24.2325 & 0.2354 & Pass \\
$t=18$ & Fail & 0.0848 & 0.0000 & Fail & 24.2235 & 0.2540 & Pass \\
$t=19$ & Fail & 0.1003 & 0.0000 & Fail & 24.3562 & 0.0691 & Pass \\
$t=20$ & Fail & 0.1063 & 0.0000 & Fail & 24.3329 & 0.0894 & Pass \\
    \bottomrule
  \end{tabular}
\end{table}
\newpage

\begin{table}[ht!]
  \caption{Skewness and kurtosis analysis on \texttt{road} ($n=2$) with GP posterior with $D=1000$ inputs.}
  \label{tab:skewkurt5}
  \centering
  \small
  \begin{tabular}{lccccccc}
    \toprule
    %\multicolumn{2}{c}{Part}                   \\
    %\cmidrule(r){1-2}

      & overall pass& $\text{skew}(t)$ & skew $p$-value & skew pass & $\text{kurt}(t)$ & kurt $p$-value & kurt pass \\
    \midrule
     $t=0$ & Pass & 0.0047 & 0.4179 & Pass & 8.1394 & 0.2179 & Pass \\
$t=1$ & Pass & 0.0050 & 0.3840 & Pass & 8.0203 & 0.8577 & Pass \\
$t=2$ & Pass & 0.0075 & 0.1803 & Pass & 7.8849 & 0.3090 & Pass \\
$t=3$ & Pass & 0.0067 & 0.2320 & Pass & 7.8799 & 0.2885 & Pass \\
$t=4$ & Pass & 0.0105 & 0.0673 & Pass & 8.0387 & 0.7323 & Pass \\
$t=5$ & Fail & 0.0181 & 0.0045 & Fail & 8.0728 & 0.5198 & Pass \\
    \bottomrule
  \end{tabular}
\end{table}

\begin{table}[ht!]
  \caption{Skewness and kurtosis analysis on \texttt{road\_2d} ($n=4$) with GP posterior with $D=1000$ inputs.}
  \label{tab:skewkurt6}
  \centering
  \small
  \begin{tabular}{lccccccc}
    \toprule
    %\multicolumn{2}{c}{Part}                   \\
    %\cmidrule(r){1-2}

      & overall pass& $\text{skew}(t)$ & skew $p$-value & skew pass & $\text{kurt}(t)$ & kurt $p$-value & kurt pass \\
    \midrule
$t=0$ & Pass & 0.0213 & 0.6065 & Pass & 24.1454 & 0.4582 & Pass \\
$t=1$ & Pass & 0.0232 & 0.4977 & Pass & 23.6600 & 0.0827 & Pass \\
$t=2$ & Pass & 0.0202 & 0.6635 & Pass & 23.8459 & 0.4317 & Pass \\
$t=3$ & Fail & 0.0233 & 0.4936 & Pass & 24.4087 & 0.0370 & Fail \\
$t=4$ & Pass & 0.0217 & 0.5827 & Pass & 23.9153 & 0.6657 & Pass \\
$t=5$ & Pass & 0.0111 & 0.9799 & Pass & 24.0857 & 0.6618 & Pass \\
    \bottomrule
  \end{tabular}
\end{table}

% TODO - fix the table positioning etc.

\paragraph{Qualitative results.} We present a direct comparison between \texttt{gaussian\_approx} and \texttt{posterior\_sampling\_approx} in Fig.~\ref{fig:pointsandellipses}. We start from an $\epsilon$-recoverable state $[0.0, 0.0, -0.15, 0.3 ]^T \in \mathcal{X}$ in \texttt{cartpole}, computing the uncertainty sets $\mathcal{E}(0), \ldots \mathcal{E}(N)$ with \texttt{gaussian\_approx} and plotting the corresponding $1-\epsilon_t=99.99\%$ confidence intervals, in addition we (approximately) sample points with the subroutine \texttt{posterior\_sampling\_approx}. We see for all $t=0,1,\ldots, N$ the uncertainty sets $\mathcal{E}(0), \ldots \mathcal{E}(N)$ capture the empirical samples and to the naked eye the empirical samples broadly follow multivariate Gaussian distributions with correct mean and variance. 
\newpage

\begin{figure}[t]
  \centering
  \includegraphics[width=0.55\linewidth]{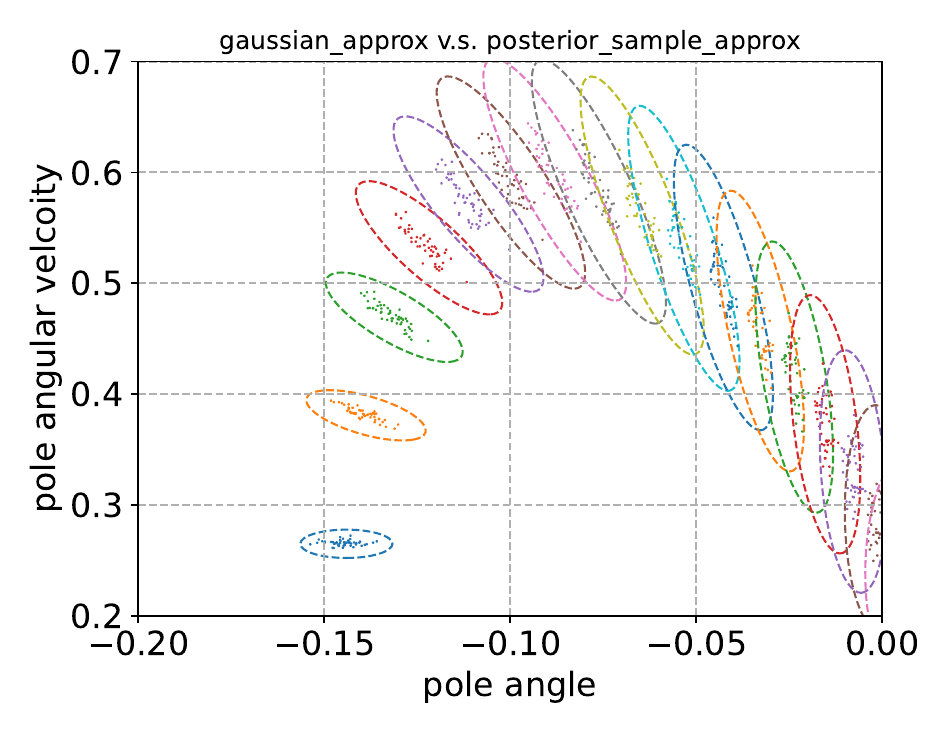}
        \caption{$K=50$ empirical samples and 99.99\%-CI uncertainty sets $\mathcal{E}(0), \ldots \mathcal{E}(N)$ for \texttt{cartpole (i)/(ii)}, with GP posterior with $D=400$ inputs.}
        \label{fig:pointsandellipses}
\end{figure}

\section{Hyperparameters and access to code}
\label{sec:hyperparameters}

\paragraph{Access to code.} Code for the paper is currently available at \url{https://github.com/sacktock/A2C-GP-Shield} (MIT License).
%To maintain a high standard of anonymity our code we provide code for the main experiments and baselines: A2C-Shield, A2C-Eval, CPO, PPO-Lag. 

\paragraph{Training details.}  For collecting all sets of experiments we has access to 2 NVIDIA Tesla A40 (48GB RAM) GPU and a 24-core/48 thread Intel Xeon CPU each with 32GB of additional CPU RAM. For computation of the linear backup policies and control invariant sets we use AROC \cite{kochdumper2021aroc} which is released under the GPLv3 license. Each run can typically take several hours to a day, depending on the environment and hyperparameters used. Our own implementation of A2C is written in Python with JAX \cite{jax2018github} and inspired by Stable Baselines JAX (SBX) \cite{stable-baselines3}. For GP dynamics learning (as mentioned in the paper) we use GPJAX  \cite{Pinder2022}. For key dependencies we refer the reader to \url{https://docs.jaxgaussianprocesses.com/} (MIT License). For CPO \cite{achiam2017constrained} and PPO-Lag \cite{ray2019benchmarking}, we use the implementations provided by Omnisafe \cite{omnisafe}, for setup instructions please refer to the Omnisafe code base \url{https://github.com/PKU-Alignment/omnisafe} (Apache-2.0 license).

\paragraph{Statistical significance.} We report standard error bars for each of our experiments (c.f., Tab.~\ref{tab:results}). In particular, we used 5 random initializations (seeds) for each experiment. For the learning curves (plots) the error bars are non-parametric (bootstrap) 95\% confidence intervals, provided by \texttt{seaborn.lineplot} with default parameters: \texttt{errorbar=(`ci', 95), n\_boot=1000}. The error bars capture the randomness in the initialization of the task policy parameters, the randomness of the environment (observation noise and disturbances) and any randomness in the batch sampling. 

\paragraph{Hyperparameters.} Tab.~\ref{tab:hyperparameters} details the default hyperparameters used for our approach. The hyperparameters vary across environments, in particular the recovery horizon $N \in [0, 100]$ varies depending on how quickly we expect the backup policy to recover the system state. For environment specific hyperparameters please refer to the \texttt{configs.yaml} provided in as part of our implementation.
%For environment specific hyperparameters we refer the reader to the \texttt{configs.yaml} file provided in our code which is provided in the supplementary material. 
We also provide the hyperparameters used as input to Omnisafe \cite{omnisafe} for collecting the results for CPO \cite{achiam2017constrained} and PPO-Lag \cite{ray2019benchmarking} in Tab.~\ref{tab:hyperparams2}, anything not specified is assumed to be the default for Omnisafe. For hyperparameter for MPS and DMPS we refer the reader to \cite{banerjee2024dynamic}.
\clearpage
\newpage

\begin{table}[H]
    \caption{Default hyperparameters for A2C-Shield}
    \label{tab:hyperparameters}
    \centering
    \small
    \begin{tabular}{lcc}
        \toprule
        Name  & Symbol & Value \\
        \midrule
        \multicolumn{3}{c}{Shield}\\
        \midrule
        Recovery horizon     & $N$          & 20        \\
        Sigma coef.           & $z$ ($\epsilon_t$)          & 3.82 (0.9999)        \\
        \midrule
        \multicolumn{3}{c}{Gaussian process}\\
        \midrule
        Learning rate  & -          & 0.01        \\
        Optim.~iters. & -          & 100       \\
        Batch size          & -         & 64        \\
        Buffer size & $\vert E\vert$ or $D$ & 1000\\
        Sparse & - & False\\
        Collapsed & - & False\\
        Num.~inducing inp. & - & 500\\
        \midrule
        \multicolumn{3}{c}{Actor-critic}\\
        \midrule
        num.~layers & - & 2\\
        num.~units & - & 64\\
        actor opt. & - & adam\\
        actor lr & - & $3 \times 10^{-4}$\\
        actor clip norm & - & 0.5\\
        actor ent.~coef. & - & 0.05\\
        critic opt. & - & adam\\
        critic lr & - & $1 \times 10^{-3}$\\
        critic clip norm & - & 0.5\\
        critic reg. coef & - & 1.0\\
        discount & $\gamma$ & 0.99\\
        gae lambda & $\lambda_{\text{gae}}$ & 0.95\\
        polyak tau & $\tau_{\text{polyak}}$ & 0.02\\
        rollout horizon & $H$ & 16\\
        num.~posterior samples & - & 5\\
        num.~starting states & $\vert B \vert$ & 32\\
        (effective) batch size & - & 2560\\
        Optim.~iters & - & 10\\
        \bottomrule
    \end{tabular}
\end{table}

\begin{table}[H]
  \caption{Hyperparameter details for CPO and PPO-Lag.}
  \label{tab:hyperparams2}
  \centering
  \small
  \begin{tabular}{lcc}
    \toprule
    %\multicolumn{2}{c}{Part}                   \\
    %\cmidrule(r){1-2}
       Name & Symbol & Value  \\
    \midrule
    \multicolumn{3}{c}{Shared}\\
    \midrule
    Actor learning rate & $\eta$ & $3\times 10^{-4}$\\
    Discount factor & $\gamma$ & $\{0.99, 0.997\}$\\
    Cost coefficient & $c$ & $1.0$\\
    Cost threshold & - & $0.01$\\
    Cost gamma & $\gamma_c$ & $\{0.99, 0.997\}$ \\
    TD-lambda & $\lambda$ & $0.95$ \\
    Cost TD-lambda & $\lambda_c$ & $0.95$ \\
    Max grad norm & - & $0.5$ \\
    Entropy coefficient & - & $0.0$\\
    Steps per epoch & $n$ & $2048$\\
    Update iterations (per epoch) & $k$ & $10$ \\
    Batch size & $B$ & $64$\\
    \midrule
    \multicolumn{3}{c}{CPO}\\
    \midrule
     \multicolumn{3}{c}{...}  \\
    \midrule
    \multicolumn{3}{c}{PPO-Lag}\\
    \midrule
    Initial Lagrangian multiplier & $\lambda_{\textit{init}}$& $10.0$\\
    Epsilon clip & $\epsilon_{\textit{clip}}$ & $0.2$\\
    \bottomrule
  \end{tabular}
\end{table}

\section{Environment descriptions}
\label{sec:environmentdescriptions}

\subsection{Cartpole}

The cartpole problem (\texttt{cartpole}) \cite{6313077} is a classic benchmark for non-linear dynamical systems, where the task is to balance the pole upright towards its stable equilibrium point. Fig.~\ref{fig:cartpole} visualizes the problem.

\begin{figure}[ht!]
    \centering
    \includegraphics[width=0.45\linewidth]{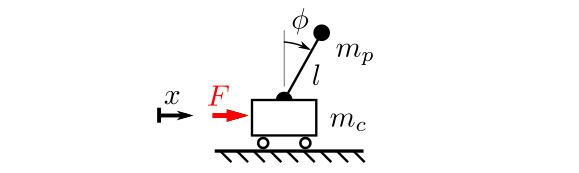}
    \caption{Visualization of the cartpole dynamical system}
    \label{fig:cartpole}
\end{figure}

\textbf{Dynamics.} The system dynamics are governed by the equations:
\begin{align}
    &x(t+1)_1 = x(t)_3 \\
    &x(t+1)_2 = x(t)_4 \\
    \begin{split}
    &x(t+1)_3 = \frac{m_p l x(t)_4^2\sin(x(t)_2)\cos(x(t)_2)^2}{(4/3)(m_c + m_p)^2l - m_p(m_c + m_p)l\cos(x(t)_2)^2} \\
    & - \frac{(m_c + m_p) g \cos(x(t)_2)\sin(x(t)_2) + \cos(x(t)_2)^2Fu(t)}{(4/3)(m_c + m_p)^2l - m_p(m_c + m_p)l\cos(x(t)_2)^2}\\
    & +\frac{Fu(t)+m_plx(t)_4^2\sin(x(t)_2)}{m_c + m_p} + w(t)_1
    \end{split} \\
    \begin{split}
        & x(t+1)_4 = \frac{(m_c +m_p)g\sin(x(t)-2)}{(4/3)(m_c + m_p)l - m_pl\cos(x(t)_2)^2}\\
        & -\frac{m_p l x(t)^2_4\sin(x(t)_2)\cos(x(t)_2) + \cos(x(t)_2) Fu(t)}{(4/3)(m_c + m_p)l - m_pl\cos(x(t)_2)^2} + w(t)_2
    \end{split}
\end{align}
where $x(t)_1$ is the position of the cart, $x(t)_3$ is the velocity of the cart, $x(t)_2$ is the angle of the pole and $x(t)_4$ is the angular velocity of the pole. In addition, $u(t)$ is the force applied, with $F=10$ being the force magnifier, $m_c=1$ is the mass of the car, $m_p=0.1$ is the mass of the pole, $l=0.5$ is the length of the pole and $g=9.8$ is gravity.

\textbf{Constraints.} The state and action constraints are $\mathcal{X} = [-4.8, 4.8] \times [-0.4190, 0.4190] \times [-\infty, \infty] \times [-\infty, \infty] \subset \mathbb{R}^4$ and $\mathcal{U} = \{[-1.0, 1.0]\} \subset \mathbb{R}$. The safety constraints are given by the box constraints, $\mathcal{X}_{\text{safe}} = [-2.4, 2.4] \times [-0.2095, 0.2095] \times [-\infty, \infty] \times [-\infty, \infty] \subset \mathbb{R}^4$. The disturbances are drawn uniformly from the set $\mathcal{W} =  [-0.001, 0.001] \times [-0.001, 0.001] \subset \mathbb{R}^2$, the observed noise variances are $\sigma^2_i=1\times 10^{-6}$ for all $i = 1, \cdots, n$.

\textbf{Initial state.} The initial state is sampled uniformly from $\mathcal{X}_0 = [-0.5, 0.5]^4$

\textbf{Reward.} We consider 3 different reward functions in this environment: (i) $r(t) = +1 \text{ if } x(t+1) \in \mathcal{X}_{\text{safe}} \text{ otherwise } 0$, (ii) $r(t) = +1 \text{ if } x(t+1)_1 \geq 0.1 \text{ otherwise } 0$, (iii) $r(t) = +1 \text{ if } x(t+1)_2 \geq 0.1 \text{ otherwise } 0$

\textbf{Termination condition.} The episode terminates either after 200 timesteps or when the current system state leaves $\mathcal{X}_{\text{safe}}$.

\textbf{Backup policy.} The backup policy $\pi_\text{backup}$ is given by the feedback matrix, $K=[-0.7488, -1.2280, -7.2758,-1.7787]^T$, computed via LQR using a linearization (matrices $A, b$) of the environment dynamics around the stable equilibrium point $x_{eq}$. The equilibrium state control is also $u_{eq} = \mathbf{0}$. 

\textbf{Control invariant set.} Computed using AROC \cite{kochdumper2021aroc}, with feedback matrix $K$ (see above) around the equilibrium point $x_{eq} = \mathbf 0$ taking into account the disturbances $\mathcal{W}$ and observation noise $\mathcal{V} = \text{diag}([\sigma_1^2, \ldots, \sigma_n^2])$ (with each $\sigma^2_i=1\times 10^{-6}$).

\subsection{Mountain car}

Mountain car (\texttt{mountain\_car}) \cite{Moore90efficientmemory-based} is another classical control problem, the goal is to escape a sinusoidal valley to the top of the `mountain', where the only possible actions that can be taken are to accelerate the car left or right. The car is under-actuated, meaning it cannot escape the valley by simply going right, it must plan a long trajectory using the valley to build up momentum before progressing to the top of the mountain. Fig.~\ref{fig:mountaincar} visualizes the problem.

\begin{figure}[ht!]
    \centering
    \includegraphics[width=0.25\linewidth]{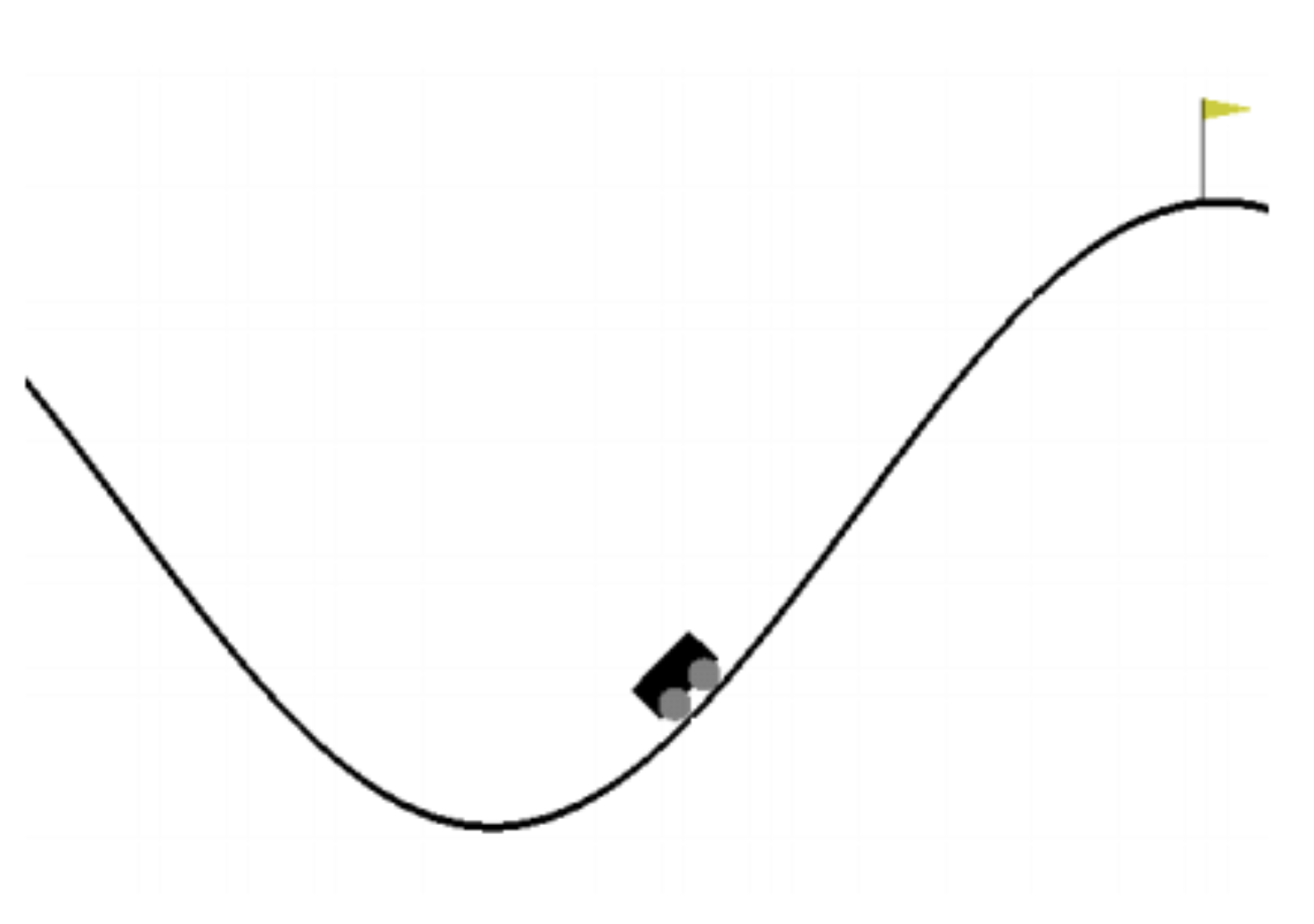}
    \caption{Visualization of the mountain car problem}
    \label{fig:mountaincar}
\end{figure}

\textbf{Dynamics.} The system dynamics are governed by the equations:
\begin{align}
    &x(t+1)_1 = x(t)_1 + x(t+1)_2 \\
    &x(t+1)_2 = x(t)_2 + F u(t) - g \cos(3x(t)_1) + w(t)
\end{align}
where, $x(t)_1$ is the position of the car and $x(t)_2$ is the velocity of the car, $u(t)$ is the force applied, with $F=0.00015$ being the force dampener, and $g=0.0025$ representing `gravity'.

\textbf{Constraints.} The state and action constraints are $\mathcal{X} = [-1.2, 0.6] \times [-0.07, 0.07] \subset \mathbb{R}^2$ and $\mathcal{U} = \{[-1.0, 1.0]\} \subset \mathbb{R}$. The safety constraints are given by the box constraints, $\mathcal{X}_{\text{safe}} = [-1.2, \infty] \times [-\infty, \infty] \subset \mathbb{R}^2$, this corresponds to not `hitting' or `falling off' the left hand side of the valley. The disturbances are drawn uniformly from the set $\mathcal{W} = [-0.001, 0.001] \subset \mathbb{R}^2$, the observed noise variances are $\sigma^2_i=1\times 10^{-6}$ for all $i = 1, \cdots, n$.

\textbf{Initial state.} The initial velocity is always $0.0$, but the initial car position state is sampled uniformly from $[-0.6, -0.4]$, i.e., $\mathcal{X}_0 = [0.0, 0.0] \times [-0.6, -0.4]$

\textbf{Reward.} A reward of $+100$ is achieved by reaching the goal position $0.45$ at the top of the mountain, otherwise the agent is penalized for the magnitude of its actions. Formally, $r(t) = 100 \text{ if } x(t+1)_1 \geq 0.45 \text{ otherwise } - 0.1 \cdot \lVert u(t)\rVert$.

\textbf{Termination condition.} The episode terminates either after 1000 timesteps or after reaching the target position $0.45$.

\textbf{Backup policy.} The backup policy $\pi_\text{backup}$ is given by the feedback matrix, $K = [0.7625, 34.5971]^T$, computed via LQR using a linearization (matrices $A, b$) of the environment dynamics around the stable equilibrium point $x_{eq} = [0.0, -\pi/6]^T$. The equilibrium state control here is $u_{eq} = \mathbf{0}$. 

\textbf{Control invariant set.} Computed using AROC \cite{kochdumper2021aroc}, with feedback matrix $K$ (see above) around the equilibrium point $x_{eq}$ taking into account the disturbances $\mathcal{W}$ and observation noise $\mathcal{V} = \text{diag}([\sigma_1^2, \ldots, \sigma_n^2])$ (with each $\sigma^2_i=1\times 10^{-6}$).

\subsection{Obstacle (2/3/4)}

We consider two simple obstacle navigation scenarios (\texttt{obstacle} and \texttt{obstacle2/3/4}). The agent is point that must navigate a 2D plane to reach a goal position, while avoiding an unsafe zone which corresponds on an obstacle. The agent can apply force to accelerate (or decelerate) along each of the two axes. 

\textbf{Dynamics.} The system dynamics for \texttt{obstacle} are governed by the equations:
\begin{align}
    &x(t+1)_1 = x(t)_1 + 2x(t+1)_3 \\
    &x(t+1)_2 = x(t)_2 + 2x(t+1)_4 \\
    &x(t+1)_3 = x(t)_3 + 5Fu(t)_1 + w(t)_1\\
    &x(t+1)_4 = x(t)_4 + 5Fu(t)_2 + w(t)_2\\
\end{align}
where, $x(t)_1$ and $x(t)_2$ are the x/y position of the agent, $x(t)_3$ and $x(t)_4$ are the x/y velocity of the agent, $u(t)_1$ and $u(t)_2$ are the x/y force applied to the agent, $F=0.001$ is the force dampener. 

The system dynamics for \texttt{obstacle2/3/4} are governed by the equations:
\begin{align}
    &x(t+1)_1 = x(t)_1 + x(t+1)_3 \\
    &x(t+1)_2 = x(t)_2 + x(t+1)_4 \\
    &x(t+1)_3 = x(t)_3 + 2Fu(t)_1 + w(t)_1 \\
    &x(t+1)_4 = x(t)_4 + 2Fu(t)_2 + w(t)_2\\
\end{align}
where, $x(t)_1$ and $x(t)_2$ are the x/y position of the agent, $x(t)_3$ and $x(t)_4$ are the x/y velocity of the agent, $u(t)_1$ and $u(t)_2$ are the x/y force applied to the agent, $F=0.001$ is the force dampener. 

\textbf{Constraints.} For both \texttt{obstacle} and \texttt{obstacle2} the state and action constraints are $\mathcal{X} = [-0.5, 3.5]^2 \times [-0.05, 0.05]^2 \subset \mathbb{R}^4$ and $\mathcal{U} = [-2.0, 2.0]^2 \subset \mathbb{R}^2$. For \texttt{obstacle} the obstacle (unsafe zone) is in the position $[0.0, 1.0]\times[2.0, 3.0]$, for \texttt{obstacle2} the obstacle (unsafe zone) is in the position $[1.0, 2.0]\times[1.0, 2.0]$. The disturbances are drawn uniformly from the set $\mathcal{W} = [-0.001, 0.001]^2 \subset \mathbb{R}^2$, the observed noise variances are $\sigma^2_i=1\times 10^6$ for all $i = 1, \cdots, n$. \texttt{obstacle3} includes a non-convex combination of one elongated obstacle (unsafe zone) in position $[1.5, 2.0]\times[0.5, 2.0]$ and one boundary condition $x_4(t) < 2.5$. \texttt{obstacle4} is more complex still, it includes a non-convex combination of two obstacles at positions $[1.5, 2.0]\times[0.5, 2.0]$

\textbf{Initial state.} The initial is fixed at $\mathbf{0} \in \mathbb{R}^4$

\textbf{Reward.} Each environment has different goal regions:
\begin{itemize}
    \item For \texttt{obstacle} the goal is to reach $x(t)_1\geq 3.0$ and $x(t)_2 \geq 0.0$, where the agent receives a reward $+30$.
    \item For \texttt{obstacle2} the goal is to reach $x(t)_1\geq 3.0$ and $x(t)_2 \geq 3.0$, where the agent receives a reward $+30$. 
    \item For \texttt{obstacle3} the goal is to reach $x(t)_1\geq 3.0$ and $x(t)_2 \geq 1.5$, where the agent receives a reward $+30$. 
    \item For \texttt{obstacle4} the goal is to reach $x(t)_1\geq 3.0$ and $x(t)_2 \geq 2.0$, where the agent receives a reward $+30$. 
\end{itemize}
In addition, the agent is provided with a dense reward e.g., $\lVert x(t+1)_1 - 3.0 \rVert - \lVert x(t)_1 - 3.0 \rVert$  (\texttt{obstacle}) to incentivise the agent to seek out the goal position. 

\textbf{Termination condition.} The episode terminates after 200 timesteps.

\textbf{Backup policy.} The backup policy $\pi_\text{backup}$ is given by the feedback matrix,
\[
K = 
\begin{bmatrix}
0.0 & 0.0 \\
0.0 & 0.0 \\
30.6386 & 0.0 \\
0.0 & 30.6386
\end{bmatrix}
\]
computed via LQR using a linearization (matrices $A, b$) of the environment dynamics around the stable equilibrium point $x_{eq} = \mathbf{0}$. The equilibrium state control here is $u_{eq} = \mathbf{0}$. 

\textbf{Control invariant set.} Is the trivially the full state space. If the point mass can slow to a stationary position without intersecting with an unsafe zone (obstacle) or boundary condition within $N$ timesteps, then this is sufficient for safety checks. $N$ just needs to be large enough so that the agent halts in this time to stationary even from its top speed, for \texttt{obstacle} $N=20$ is sufficient for \texttt{obstacle2/3/4} $N=40$ is sufficient.
\subsection{Road}

We consider a simple road environment, in both along both the 1D line (\texttt{road}) and in the 2D plane (\texttt{road\_2d}). The agent must navigate to a goal position in both of these environments, while adhering to a given speed limit. The agent can accelerate (or decelerate) along each of the axes. 

\textbf{Dynamics.} The system dynamics for \texttt{road} are governed by the equations:
\begin{align}
    &x(t+1)_1 = x(t)_1 + 10x(t+1)_2 \\
    &x(t+1)_2 = x(t)_2 + 10Fu(t) + w(t)\\
\end{align}
where, $x(t)_1$ is the position of the agent and $x(t)_2$ is the velocity of the agent, $u(t)$ is the force applied to the agent and  $F=0.0001$ is the force dampener. 

The system dynamics for \texttt{road\_2d} are governed by the equations:
\begin{align}
    &x(t+1)_1 = x(t)_1 + 10x(t+1)_3 \\
    &x(t+1)_2 = x(t)_2 + 10x(t+1)_4 \\
    &x(t+1)_3 = x(t)_3 + 5Fu(t)_1 + w(t)_1 \\
    &x(t+1)_4 = x(t)_4 + 5Fu(t)_2 + w(t)_2\\
\end{align}
where, $x(t)_1$ and $x(t)_2$ are the x/y position of the agent, $x(t)_3$ and $x(t)_4$ are the x/y velocity of the agent, $u(t)_1$ and $u(t)_2$ are the x/y force applied to the agent, $F=0.0001$ is the force dampener. 

\textbf{Constraints.} For \texttt{road} the state and action constraints are $\mathcal{X} = [-4.0, 4.0] \times [-0.1, 0.1] \subset \mathbb{R}$ and $\mathcal{U} = [-2.0, 2.0] \subset \mathbb{R}$. Similarly, for \texttt{road\_2d} the state and action constraints are $\mathcal{X} = [-4.0, 4.0]^2 \times [-0.1, 0.1]^2 \subset \mathbb{R}^2$ and $\mathcal{U} = [-2.0, 2.0]^2 \subset \mathbb{R}^2$. The safety constraints are to ensure the absolute velocity in either direction is bounded by $0.01$, in particular, for \texttt{road} we have $\mathcal{X}_{\text{safe}}= [-\infty, \infty] \times [-0.01, 0.01] \subset \mathbb{R}^2$ and for \texttt{road\_2d} we have $\mathcal{X}_{\text{safe}}= [-\infty, \infty]^2 \times [-0.01, 0.01]^2 \subset \mathbb{R}^2$. The disturbances are drawn uniformly from the set $\mathcal{W} = [-0.001, 0.001] \subset \mathbb{R}$ or $\mathcal{W} = [-0.001, 0.001]^2 \subset \mathbb{R}^2$, the observed noise variances are $\sigma^2_i=1\times 10^{-6}$ for all $i = 1, \cdots, n$.

\textbf{Initial state.} The initial is fixed at $\mathbf{0} \in \mathbb{R}^2$ or $\mathbf{0} \in \mathbb{R}^4$

\textbf{Reward.} For \texttt{road} the goal is to reach $x(t)_1\geq 3.0$, where the agent receives a reward $+20$. For \texttt{road\_2d} the goal is to reach $x(t)_1\geq 3.0$ and $x(t)_2 \geq 3.0$, where the agent receives a reward $+20$. In addition, for \texttt{road} the agent is provided with a dense reward $\lVert x(t+1)_1 - 3.0 \rVert - \lVert x(t)_1 - 3.0 \rVert$ to incentivise the agent to seek out the goal position. Similarly, for \texttt{road\_2d} the agent is provided with the dense reward $\lVert x(t+1)_{1:2} - [3.0, 3.0] \rVert - \lVert x(t)_{1:2} - [3.0, 3.0] \rVert$.

\textbf{Termination condition.} The episode terminates either after 200 timesteps or after reaching the goal position.

\textbf{Backup policy.} The backup policy $\pi_\text{backup}$ for \texttt{road} is given by the feedback matrix,
\[
K =
[
0.0,14.0425]^\top
\]
and for \texttt{road2d} analogously,
\[
K =
\begin{bmatrix}
0.0 & 0.0 \\
0.0 & 0.0 \\
14.0425 & 0.0 \\
0.0 & 14.0425
\end{bmatrix}
\]
computed via LQR using a linearization (matrices $A, b$) of the environment dynamics around the stable equilibrium point $x_{eq} = \mathbf{0}$. The equilibrium state control here is $u_{eq} = \mathbf{0}$. 

\textbf{Control invariant set.} Here it is the full safe set $\mathcal{X}_{\text{safe}}$. This can be easily seen as for any velocity within the safe set $\mathcal{X}_{\text{safe}}$ there exists some control $u^*$ that can decelerate the point mass.

\subsection{Hopper-v5: Environment setup and invariant set computation}
\label{sec:hopper}

The \texttt{Hopper-v5} environment simulates a 2D single-leg robot in MuJoCo with an $11$-dimensional observation vector (joint angles and velocities excluding the horizontal root position) and a $3$-dimensional torque control input (c.f., Fig.~\ref{fig:hopper}). 
\begin{figure}[H]
    \centering
    \includegraphics[width=0.25\linewidth]{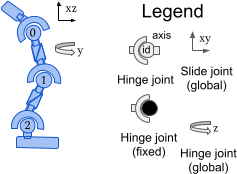}
    \caption{Visualization of \texttt{Hopper-v5}}
    \label{fig:hopper}
\end{figure}

\textbf{Dynamics.} See \cite{towers2024gymnasium}.

\textbf{Constraints.} The state and action constraints are $\mathcal{X} = [-\infty, \infty]^{11} $ and $\mathcal{U}= [-1.0, 1.0]^3$ respectively. The safety constraints $\mathcal{X}_{\text{safe}}$ are used to ensure the hopper stays ``healthy'' and upright and no extreme values. In particular, all values of the state vector must remain in the range $[-100, 100]$, the z-coordinate of the torse must lie in the range $[0.7, \infty]$ and the angle of the torse must lie in the range $[-0.2, 0.2]$. Disturbances are not considered here as the problem is complex, although the usual observation noise with assumed noise variances $\sigma^2_i=1\times 10^{-6}$ for all $i = 1, \cdots, n$ is present.
 
\textbf{Initial state.} The initial state is sampled uniformly from $\mathcal{X}_0 = [-0.005, 0.005]^{11}$ plus the $1.25$ for the torso z-coordinate.

\textbf{Reward.} The agent receives a reward of $+1$ for remaining healthy (safe), in addition they receive a bonus for moving forward, see \cite{towers2024gymnasium} for precise details.

\textbf{Terminal condition.} The episode terminates after either the agent is unhealthy or after 1000 timesteps.

\textbf{Backup policy.} To construct the required \emph{control invariant set} $\mathcal{X}_{\text{inv}}$ and backup controller $\pi_{\text{backup}}$, we first linearized the MuJoCo dynamics about the nominal equilibrium $x_{eq}$, obtaining $A, b$ matrices, the backup policy $\pi_\text{backup}$ is then obtained via LQR and is defined by the feedback matrix $K$ and further improved with Newton refinement to obtain $u_{eq}$ and maintain the equilibrium state ($f(x_{eq}, u_{eq}) \approx x_{eq}$). Thus,  the linear feedback gain $K$ defines the local stabilizing law $u = u_{eq} -K(x - x_{eq})$.

\textbf{Control invariant set.}
We then estimated $\mathcal{X}_{\text{inv}}$ as the largest verified ellipsoidal level set
\[
\mathcal{E}(P, c) = \{x : (x-x_{eq})^\top P (x-x_{eq}) \le c \}
\]
such that trajectories under $u=u_{eq}-K(x-x_{eq})$ remain within $\mathcal{E}(P,c)$. The set radius parameter $c$ was expanded geometrically by a factor of two each iteration, up to $\text{MAX\_GROW\_STEPS}=8$, using $256$ Monte Carlo rollouts per candidate. Expansion was halted after three successful iterations, beyond which verification failed, indicating that the controller could no longer guarantee invariance. The final verified ellipsoid was converted to an inscribed axis-aligned box and convex hull to form the invariant polytope $\mathcal{X}_{\text{inv}}$ used by the shield.

Although the computed $\mathcal{X}_{\text{inv}}$ sufficed for local stabilization, it was extremely small relative to the Hopper’s operational range. Consequently, the shield classified nearly all encountered states as unrecoverable, leading to frequent interventions that prevented the learned policy from acting. This illustrates a key dependency of our framework: while only $\pi_{\text{backup}}$ and $\mathcal{X}_{\text{inv}}$ are needed for safety certification, the practical performance and permissiveness of the shield depend critically on the quality of these components. 

%In highly non-linear systems like Hopper—featuring contact dynamics, actuator limits, and non-smooth transitions—linear controllers are often insufficient to yield a large enough invariant region, motivating the use of non-linear or learned backup controllers in future work.

\end{document}